\documentclass{amsart}
\usepackage{graphicx,fullpage,xcolor,comment,bm}
\usepackage{amsmath,amsfonts}
\usepackage[english]{babel}
\usepackage{amsthm}

\usepackage{algorithm, setspace}
\usepackage{textcomp}
\usepackage{algpseudocode}
\usepackage{bbm}
\usepackage[shortlabels]{enumitem}
\newtheorem{corollary}{Corollary}

\usepackage[hidelinks]{hyperref}

\def\brow{\beta_{\text{row}}}
\def\ncor{n_{\text{corr}}}
\def\tbrow{\tilde{\beta}_{\text{row}}}
\def\tbeta{\tilde {\beta}}

\newtheorem{theorem}{Theorem}
\newtheorem{lemma}{Lemma}
\newtheorem{remark}{Remark}
\newtheorem{definition}{Definition}

\newcommand{\E}{\mathbb{E}}
\newcommand{\C}{\mathbb{C}}

\newcommand{\tens}[1]{\bm{\mathcal{#1}}}
\newcommand{\mat}[1]{\bm{#1}}
\def\tA{{\tens{A}}}  
\def\tB{{\tens{B}}}  

\def\tE{{\tens{E}}}

\def\tH{{\tens{H}}}

\def\tX{{\tens{X}}}  
\def\tY{{\tens{Y}}}

\def\vb{{\bm{b}}}

\def\vv{{\bm{v}}}

\def\vx{{\bm{x}}}

\def\R{{\mathbb{R}}} 
\def\E{{\mathbb{E}}} 

\def\bcirc{{\mathrm{bcirc}}}
\def\unfold{{\mathrm{unfold}}}
\def\fold{{\mathrm{fold}}}

\def\bcirc{{\mathrm{bcirc}}}

\title{Quantile-Based Randomized Kaczmarz for Corrupted Tensor Linear Systems}
\author{Alejandra Castillo, Jamie Haddock, Iryna Hartsock, Paulina Hoyos, Lara Kassab, Alona Kryshchenko, Kamila Larripa, Deanna Needell, Shambhavi Suryanarayanan, Karamatou Yacoubou Djima}
\date{}

\begin{document}


\begin{abstract}
The reconstruction of tensor-valued signals from corrupted measurements, known as tensor regression, has become essential in many multi-modal applications such as hyperspectral image reconstruction and medical imaging. In this work, we address the tensor linear system problem 
$\tA \tX=\tB$, where $\tA$ is a measurement operator, $\tX$ is the unknown tensor-valued signal, and $\tB$ contains the measurements, possibly corrupted by arbitrary errors. Such corruption is common in large-scale tensor data, where transmission, sensory, or storage errors are rare per instance but likely over the entire dataset and may be arbitrarily large in magnitude. We extend the Kaczmarz method, a popular iterative algorithm for solving large linear systems, to develop a Quantile Tensor Randomized Kaczmarz (QTRK) method robust to large, sparse corruptions in the observations 
$\tB$. This approach combines the tensor Kaczmarz framework with quantile-based statistics, allowing it to mitigate adversarial corruptions and improve convergence reliability. We also propose and discuss the Masked Quantile Randomized Kaczmarz (mQTRK) variant, which selectively applies partial updates to handle corruptions further. We present convergence guarantees, discuss the advantages and disadvantages of our approaches, and demonstrate the effectiveness of our methods through experiments, including an application for video deblurring.

\end{abstract}
\maketitle

\section{Introduction}
The problem of reconstructing a tensor-valued signal from corrupted measurements, also known as \textit{tensor regression}, has become an important problem in applications such as hyperspectral image reconstruction~\cite{fan2017hyperspectral} and medical imaging~\cite{zhou2013tensor}. 
\emph{Multi-modal} data can be represented as a multidimensional array. 
Mathematically this is known as a \textit{tensor}, a higher-order generalization of a matrix.
A tensor $\tA \in \mathbb R^{n_1 \times \cdots \times n_d}$ is a tensor of \textit{order} $d$ which refers to the number of \textit{modes} in the tensor.
For example, a matrix has 2 modes and therefore is a tensor of order 2.
The modes in \emph{multi-modal} data can represent measurements along different data dimensions, e.g., spatial or temporal.

We consider the tensor linear system 
\begin{equation}
\label{eq:tensor-system}
    \tA \tX = \tB,
\end{equation}
where $\tA \in \R^{m \times l \times n}$ is the measurement operator, $\tX \in \R^{l \times p \times n}$ is the (unknown) signal of interest, and $\tB \in \R^{m \times p \times n}$ represents observed measurements. 
The notation $\tA \tX$ refers to the t-product between $\tA$ and $\tX$ (defined in Section~\ref{subsec:t-product}). 
The t-product, which was introduced in~\cite{kilmer2011factorization}, has gained significant traction and shown to be useful in dictionary learning~\cite{soltani2016tensor, zhang2015denoising, newman2020nonnegative}, low-rank tensor completion~\cite{semerci2014tensor, zhang2014novel, zhang2016exact, zhou2017tensor}, facial recognition~\cite{hao2013facial, zhang2018randomized}, and neural networks~\cite{newman2018stable, wang2020tensor}.

We consider the setting in which the tensor linear system has arbitrary or possibly adversarial corruptions in $\tB$. 
This setting is relevant in most modern applications where measurements must be collected, stored, and repeatedly accessed.  
These steps often introduce transmission or transcription corruptions in the data; on any one instance this corruption is rare, but across a large-scale tensor, corruption is likely and could be of arbitrary magnitude.
We distinguish between \textit{corruption}, in which there are few but relatively large errors in the measurement tensor, and \textit{noise}, in which there are many but relatively small errors in the measurement tensor.

In our work, we build upon the Kaczmarz method~\cite{Kaczmarz1937Angen}, an iterative algorithm for solving linear systems that has gained recent attention due to its low memory and computational requirements.
Simple iterative methods like the Kaczmarz method are prime candidates for corruption-robust methods.  
The information calculated in-iteration (e.g., residual entries) can often additionally provide information about the problem's geometry and consistency, the data's trustworthiness, and how nearby the solution is to the current iterate. 

It has become common to aggregate information across multiple iterations to attempt to mitigate the effect of benign noise~\cite{bach2014adaptivity, moorman2020randomized}.
However, variants using this approach in the case of adversarial corruptions
are newer and less well-understood~\cite{quantHNRS20, steinerberger2021quantile, pmlr-v108-shah20a, diakonikolas2019sever}.  
Since these problems arise in medical imaging, sensor networks, error correction, and data science, the effects of adversarial corruptions in the data could be catastrophic downstream. 
An ill-timed update using corrupted data can reverse or abolish the valuable information learned in prior iterations on uncorrupted data, making iterative methods, in these settings, challenging to employ on large-scale data.

\subsection{Contributions}
We propose \textit{Quantile Tensor Randomized Kaczmarz} (QTRK), a method for solving tensor linear systems under the t-product $\tA \tX = \tB$ with large corruptions in $\tB$. 
The method extends the tensor RK algorithm proposed in~\cite{ma2022randomized} to handle corruptions using quantile-based statistics for the randomized Kaczmarz method, as proposed in~\cite{quantHNRS20}.  
We prove that this method converges at least linearly in expectation to the solution of the unperturbed tensor system under mild assumptions on the quantile parameter, the measurement tensor $\tA$, and the rate and distribution of corruptions. 

Higher-order tensors offer a wider array of possible corruption patterns than the matrix case. 
As evidenced by this work, both the algorithm and the theory for the third-order tensor case is drastically different than that of the matrix case. 
Most notably, QTRK's performance suffers when the corruptions appear in many of the horizontal slices in $\tB$.
To address this shortcoming,
we additionally propose a ``relaxed" method that applies coordinate-wise masking to make use of partial projection updates; we call this method \textit{Masked Quantile Randomized Kaczmarz} (mQTRK).
We provide some intuition as to when mQTRK can and cannot be guaranteed to converge to a solution. 
Throughout, we note major differences in the behavior and theoretical limitations of the QTRK methods in contrast to their simpler matrix counterparts.  
Finally, we illustrate the potential of our proposed methods with numerical experiments on synthetic data and an application to video deblurring. 

\subsection{Organization}

In Section~\ref{subsec:notation}, we present the notation used throughout this paper.  
Next, we present some relevant literature focusing on the Kaczmarz method in Section~\ref{subsec:Kaczmarz}, Kaczmarz variants for systems with noise or corruption, such as that of~\cite{quantHNRS20}, in Section~\ref{subsec:QRK}, the tensor t-product algebra in Section~\ref{subsec:t-product}, and the randomized Kaczmarz method for tensor t-product linear systems in Section~\ref{subsec:TRK}.  
In Section~\ref{sec:QTRK}, we introduce QTRK and provide a convergence analysis for this method.  
In Section~\ref{sec:masking}, we discuss using coordinate-wise masking to make partial projection updates and propose an initial such approach that we call masked QTRK and provide some initial convergence analysis for this method.  
In Section~\ref{sec:experiments}, we present numerical experiments illustrating our results and probing the behavior of the QTRK and mQTRK methods, including an application for video deblurring.

\section{Background and Related Work}

\subsection{Notation}\label{subsec:notation}
We use boldfaced lower-case Latin letters (e.g., $\vx$) to denote vectors, bolded upper-case Latin letters (e.g., $\mat A$) to denote matrices, and bolded upper-case calligraphic Latin letters (e.g., $\tA$) to denote higher-order tensors.  We use unbolded lower-case Latin and Roman letters (e.g., $q$ and $\beta$) to denote scalars.  We let $[m]$ denote the set $\{1, 2, \cdots, m\}$.   

We denote by $\mat A_{i :}$ the $i$th row of matrix $\mat A$.
For a third-order tensor $\tA \in \C^{m \times l \times n}$, we denote by $\tA_{i::} \in \C^{1 \times l \times n}$ the $i$th horizontal slice, $\tA_{: j :} \in \C^{m \times 1 \times n}$ the $j$th lateral slice, $\tA_{::k} \in \C^{m \times l \times 1}$ the $k$th frontal slice. 
We denote the conjugate transpose of the tensor $\tA \in \C^{m \times l \times n}$ by $\tA^* \in \C^{l \times m \times n}$ which is obtained by taking the conjugate transpose of each of the frontal slices $\tA_{::i}$ for $i = 1, \cdots, n$ then reversing the order of the transposed frontal slices 2 through $n$.
We will use the terms \textit{row slice} and \textit{column slice} (instead of horizontal slice and lateral slice, respectively) to ease analogies with matrix linear system problems.

The notation $\|\vv\|$ denotes the Euclidean norm of a vector $\vv$ and $\|\cdot\|_F$ the Frobenius norm of a tensor. 
We denote by $\sigma_{\min}{(\mat A)}$ (and  $\sigma_{\max}{(\mat A)}$) the smallest (and largest, respectively) singular value of the matrix $\mat A$.
For a matrix $\mat A \in \C^{l \times n}$, we use $\mat A^*$ to denote the conjugate transpose and $\mat A^\dagger \in \C^{n \times l}$ to denote the pseudoinverse of the matrix $\mat A$.

While we focus our attention on real-valued tensors in this work, the definitions provided in this section are defined more generally on $\C$.

\subsection{The Kaczmarz Method}\label{subsec:Kaczmarz}
The Kaczmarz method~\cite{Kaczmarz1937Angen} (later rediscovered for use in computerized tomography as the Algebraic Reconstruction Technique~\cite{Herman1993Algebr}) is an iterative algorithm used for solving overdetermined consistent linear systems.
Suppose $\mat A \vx= \vb$ where $\mat A \in \R ^{m\times n}$ and $\vb \in \R^n$ ($m > n$) is a system of equations with unique solution $\vx_\star$.
Starting with initial point $\vx^{(0)}$, the Kaczmarz method aims to find a solution $\vx_\star$ satisfying the equations by selecting one equation at a time and projecting the current solution onto the hyperplane defined by that equation.
The update at the $k$th iteration is defined by:
\begin{equation*}
    \vx^{(k)} = \vx^{(k-1)} + \frac{b_{i} - \mat{A}_{i :} \vx^{(k-1)}}{\|\mat{A}_{i :}\|^2}\mat{A}_{i :}^\top
\end{equation*}
where $i = k \text{ mod $m$}$ and $\mat{A}_{i :}^\top$ denotes the transpose of $\mat{A}_{i :}$.
In the landmark paper~\cite{strohmer2009randomized}, the authors introduced a randomized variant of the Kaczmarz method where at iteration $k$ row $i$ is selected with probability $\|\mat{A}_{i :} \|^2/\|\mat A\|_F^2$.
The method is called Randomized Kaczmarz (RK) and it can be viewed as a special case of Stochastic Gradient Descent (SGD) with a specific step size. 
The authors proved the following expected linear convergence guarantee:
\begin{equation*}
    \mathbb{E} \|\vx^{(k)} - \vx^*\|^2\le \left(1 - \frac{\sigma_{\min}^2(\mat{A})}{\|\mat{A}\|_F^2}\right)^k \|\vx^{(0)} - \vx^*\|^2. 
\end{equation*} 
Many variants and extensions of the Kaczmarz method followed, including convergence analyses for inconsistent and random linear systems~\cite{Nee10:Randomized-Kaczmarz, CP12:Almost-Sure-Convergence}.

\subsection{Kaczmarz Method for Corrupted Systems} \label{subsec:QRK}
Several results have shown convergence of the Kaczmarz method (or SGD more generally) in the case of additive noise to the measurement vector $\vb$. 
For example, in~\cite{needell2010randomized,schopfer2022extended}, the iterates converge linearly to the least squares solution up to some radius that depends on the noise magnitude. 
In particular, there is a body of work on the problem of solving large-scale systems of linear equations $\mat A \vx = \vb$ that are inconsistent due to noise and corruptions in the measurement vector $\vb$~\cite{quantHNRS20, steinerberger2021quantile, pmlr-v108-shah20a,  diakonikolas2019sever, zouzias2013randomized}. 

Indeed, when some entries in $\vb$ have corruptions (arbitrarily large levels of additive noise), RK periodically projects onto the associated corrupted hyperplanes and therefore does not converge. 
Given that the corruptions can be large, iterative steps that encounter such corruption should be statistically different than non-affected iterative steps. 

The proposed remedy in~\cite{quantHNRS20} called \textit{Quantile-Randomized Kaczmarz} (QRK) utilizes the residual error $\mat A \vx^{(k)}- \vb$ to detect corruptions in each iteration.
If the magnitude of a residual entry is above a certain quantile of all residual entries, that entry is deemed unreliable and the corresponding hyperplane, if selected in the $k$th iteration, will not be used. 
Otherwise, QRK projects in the same manner as RK.
The threshold value assigned to be the $q$-quantile of the residuals is a parameter of the method.
The QRK method shows reliable convergence to the true solution of the uncorrupted system under mild assumptions on the number of corruptions both empirically and theoretically. 

\subsection{Tensor t-product Algebra}
\label{subsec:t-product}

In~\cite{ma2022randomized}, RK is extended to solve multi-linear systems under the t-product, as defined in \eqref{eq:tensor-system}, with convergence guarantees analogous to those of RK for matrix linear systems.
We first provide background on the t-product before discussing the method proposed in~\cite{ma2022randomized}.

The t-product was defined in the foundational work~\cite{kilmer2011factorization} as the product between two tensors of order three. 
The authors subsequently derived formulations of the associated tensor identity, inverse, pseudoinverse, and transpose and extended orthogonal matrix factorizations such as the SVD and QR factorizations to tensors.

\begin{definition}
\label{def:t-product}
The \emph{tensor-tensor t-product} between $\tA \in \C^{m \times l \times n}$ and $\tB \in \C^{l \times p \times n}$ is defined as
\begin{equation*}
\tA \tB = \fold (\bcirc(\tA) \unfold (\tB)) \in \C^{m \times p \times n}
\end{equation*}
where $\bcirc (\tA)$ denotes the block-circulant matrix
\[\bcirc(\tA) = \begin{pmatrix}
\tA_{::1} & \tA_{::n} & \tA_{::n-1} & \dots & \tA_{::2} \\
\tA_{::2} & \tA_{::1} & \tA_{::n} & \dots & \tA_{::3} \\
\vdots & \vdots & \vdots & \dots & \vdots \\
\tA_{::n} & \tA_{::n-1} & \tA_{::n-2} & \dots & \tA_{::1} \\
\end{pmatrix} \in \C^{mn \times ln}\]
and $\unfold(\tB)$ denotes the unfolding operation defined as
\[\unfold(\tB) = \begin{pmatrix}
 \tB_{::1}\\
 \tB_{::2}\\
 \vdots \\
 \tB_{::n}
\end{pmatrix} \in \C^{ln \times p}.\]
and $\fold(\unfold (\tB)) = \tB$.
\end{definition}

A useful property for our analysis is given in~\cite[Lemma 2]{castillo2024randomized} and stated below.

\begin{lemma}[\cite{castillo2024randomized}] 
\label{lem:boundsforFrobeniusnorm}
For any two tensors $\tA \in \C^{m \times l \times n}$ and $\tB \in \C^{l \times p \times n}$,
\begin{align*}
        \sigma_{\min}(\bcirc(\tA)) \|\tB\|_F 
        &\le \|\tA \tB\|_F \le \sigma_{\max}(\bcirc(\tA)) \|\tB\|_F. 
        \
\end{align*}
\end{lemma}

\subsection{Tensor Randomized Kaczmarz}
\label{subsec:TRK}
Recently, Kaczmarz-type iterative methods have been proposed for a variety of tensor linear systems and regression problems.  
Several variants of Kaczmarz methods have been generalized to tensor regression problems~\cite{ma2022randomized,chen2021regularized,tang2023sketch}. 
In~\cite{ma2022randomized},
the authors propose Tensor Randomized Kaczmarz (TRK), a generalization of the randomized Kaczmarz method for matrix linear systems to multi-linear systems $\tA \tX = \tB$, as defined in \eqref{eq:tensor-system}.

The TRK algorithm iteratively samples a row slice of the system to obtain $\tA_{i_k::}$ and $\tB_{i_k::}$ and projects $\tX^{(k-1)}$ onto the solution space of this sampled subsystem. 
The pseudocode is provided in Algorithm~\ref{alg:TRK}. 
The authors prove that TRK converges at least linearly in expectation to the system's unique solution $\tX_\star$.

\begin{algorithm}
\caption{Tensor Randomized Kaczmarz (TRK)~\cite{ma2022randomized} }\label{alg:TRK}
	\begin{algorithmic}[1]
		\Procedure{TRK}{$\tA,\tB,\tX^{(0)}, K$} 
		\For{$k = 1, \ldots, K$}
		      \State{Sample $i_k \in [m]$.}
                \State{$\tX^{(k)} = \tX^{(k-1)} - \tA_{i_k : :}^*(\tA_{i_k : :} \tA_{i_k : :}^*)^{-1}(\tA_{i_k : :}\tX^{(k-1)} - \tB_{i_k : :})$}
		\EndFor{} \\
		\Return{$\tX^{(K)}$}
		\EndProcedure
	\end{algorithmic}
\end{algorithm}

Building on TRK~\cite{ma2022randomized} and QRK~\cite{quantHNRS20} (in the matrix case), we propose a quantile-based randomized Kaczmarz method for solving tensor linear systems under the t-product $\tA \tX = \tB$ with large corruptions in $\tB$.

\section{Quantile Tensor Randomized Kaczmarz}
\label{sec:QTRK}
We consider a tensor linear system $\tA \tX = \tB$ defined by $\tA \in \mathbb{R}^{m \times l \times n}$ and $\tB \in \mathbb{R}^{m \times p \times n}$. 
The observed measurement tensor $\tB$ decomposes as \[\tB = \tB_\star + \tB_{\text{corr}}\] where $\tB_\star \in \mathbb{R}^{m \times p \times n}$ is the ideal, unobserved measurement tensor for the consistent system $\tA \tX = \tB_\star$ and $\tB_{\text{corr}} \in \mathbb{R}^{m \times p \times n}$ is the tensor of corruptions which we assume to be sparse.
Let $\tX_\star \in \mathbb{R}^{l \times p \times n}$ denote the solution of the true underlying, consistent system
$\tA \tX_\star = \tB_\star$.  Our problem is to compute $\tX_\star$ given only knowledge of $\tA$ and $\tB$.
We refer to a tensor linear system as overdetermined if $m \geq l$.

We first remark the following fact that describes how a single corruption in $\tB$ would affect solving for $\tX_\star$. 

\begin{remark}\label{lem:corr_propag}
If a single entry $i,j,h$ of $\tB_\star$ is perturbed to create $\tB$, the system defined by $\tA$ and $\tB$ no longer need be consistent.  However, if we let $\tilde{\tX}$ be the solution to the least-squares problem $\min_{\tX} \|\tA \tX - \tB\|_F^2$ then $\tilde{\tX}$ and $\tX_\star$ are guaranteed to agree 
in all column slices besides the $j$th one. Only the entries in the $j$th column slice of $\tilde{\tX}$ may differ from those of $\tX_\star$.
This follows from the column decoupling of the t-product linear system.
\end{remark}

We now define the following parameters that will be useful for our analysis.

\begin{definition}
\label{def:ustar}
We define $U_\star$ as the set of indices corresponding to the rows slices in $\tB$ that are entirely uncorrupted, 
\begin{equation*}
    U_\star = \{i \in [m]: (\tB_{\text{corr}})_{i::} = {\bf 0}\}. 
\end{equation*}
\end{definition}

\begin{definition}
\label{def:beta}
The set of all corrupted indices in $\tB$ is defined as
$$ C = \{(i,j,h) \in [m] \times [p] \times [n]: (\tB_{\text{corr}})_{ijh} \neq 0\}.$$
We let $\ncor = |C|$ denote the number of corruptions in $\tB$.
Further, $0 \leq \beta \leq 1$ is defined as the proportion $\beta = |C|/(mpn)$ of corrupted entries in $\tB$ and $0 \leq \brow \leq 1$ the proportion of rows in $\tB$ that are corrupted, 
\begin{equation}
\label{eq:betarow}
\brow = \frac{|U_{\star}^c|}{m} = \frac{m-|U_{\star}|}{m}.
\end{equation}
\end{definition}

\begin{definition}
\label{def:Vi}
For a given index $i \in [m]$, we define $V_i^c$ to be the set of all indices $j \in [p]$ where $\tB_{ijh}$ is corrupted for some $h \in [n]$:
$$ V_i^c = \{j\in [p] : (\tB_{\text{corr}})_{ijh} \neq 0 \text{ for some } h \in [n] \}.$$
Further, we denote by $V_i = [p] \backslash V_i^c$ the complement of $V_i^c$.   
\end{definition}

By Remark~\ref{lem:corr_propag}, the set $V_i^c$ represents the indices $j$ of column slices $\tX_{:j:}$ that could be affected by the corruptions $\tB_{ijh}$ for some $h \in [n]$.
Note that in Definitions \ref{def:ustar} -- \ref{def:Vi} the parameters are defined based on true knowledge of corruptions.

We now present our method QTRK which extends the TRK approach to corrupted tensor linear systems setting using quantile-based statistics to estimate corrupted entries.

\begin{definition}[Estimated Corruptions]
\label{def:q-quantile}
Let $Q_q(\tE)$ denote the empirical $q$-quantile of the residual tensor $\tE = \tA \tX -\tB$ over all its entries,
\begin{equation*}
Q_q(\tE) = q\text{-quantile} \left \{ |\tE_{ijh}|: i \in [m], j \in [p], h \in [n] \right \},
\end{equation*} 
where $| \cdot |$ denotes the entry-wise absolute value. Here, the $q$-quantile of a finite set $S \subset \mathbb{R}$ is defined to be the $\lfloor  q|S| \rfloor$-th smallest element of $S$.
In other words, it is the element $s \in S$ such that $|\{r \in S: r \le s\}| = \lfloor  q|S| \rfloor$. 
The entries $\tB_{ijh}$ are estimated to be corrupted if $ |\tE_{ijh}| > Q_q(\tE )$.
\end{definition}

In each iteration $k$ of QTRK, the $q$-quantile $Q_q(\tE^{(k)})$ is computed on residual tensor $\tE^{(k)}= \tA \tX^{(k)} -\tB$.
Based on this estimate, QTRK omits projecting onto row slice $\tA_{i::}$ if $|\tE^{(k)}_{ijh}| > Q_q(\tE^{(k)})$ for some $h \in [n], j \in[p]$.
The algorithm then randomly samples a row slice from the estimated set of uncorrupted rows and projects $\tX^{(k-1)}$ onto the solution space of this subsystem. 
We present QTRK in Algorithm~\ref{alg:QTRK}.  We note that an algorithm similar to QTRK was proposed previously in~\cite{tensorregression24}, but utilized a quantile-threshold over the norms of the slices of the residual tensor. Unlike in this current work, no theoretical results were provided in~\cite{tensorregression24}. 
 We additionally explore a novel masking variant of this method in Section~\ref{sec:masking}.

\begin{algorithm}[h]
\caption{Quantile Tensor Randomized Kaczmarz (QTRK)}
\label{alg:QTRK}
\begin{algorithmic}[1]
\Procedure{QTRK}{$\tens{A},\tens{B}$, $\tens{X}^{(0)}$, quantile level $q$, $K$}
\For{$k = 0, 1, \ldots, K-1$}
\State{Compute residual tensor $\tE^{(k)}= \tA \tX^{(k)} -\tB$ \text{ and $q$-quantile } $Q_q(\tE^{(k)})$} \Comment{See Definition~\ref{def:q-quantile}}
\State{Estimate corrupted row slices $U^c = \left \{ i \in [m] :  |\tE^{(k)}_{ijh}| > Q_q(\tE^{(k)}) \text{ for some } h \in [n], j \in[p] \right \}$ }
\State{Sample row slice $i\sim\text{Uniform}(U)$}
\State{$\tens{X}^{(k+1)} = \tens{X}^{(k)} - \tens{A}^*_{i ::}(\tens{A}_{i ::}\tens{A}^*_{i ::})^{-1}(\tens{A}_{i ::}\tens{X}^{(k)} -\tens{B}_{i::})$}  
\EndFor{}
\\
\Return{$\tens{X}^{(K)}$}
\EndProcedure
\end{algorithmic}
\end{algorithm}

We note that under the assumption that $\tA_{i::} \tA_{i::}^*$ is invertible, 
\begin{equation}
\label{eq:projop}
{\tens P}_{i}(\tX) := \left(\tA_{i::}^*\left (\tA_{i::}\tA_{i::}^*\right )^{-1} \tA_{i::}\right) \tX
\end{equation}
is an orthogonal projection onto the range of $\tA_{i::}$.
This is proven in~\cite{kilmer2013third}. 

\subsection{Convergence Analysis of QTRK}
\label{subsec:QTRK_convergence}
We now provide an analysis of the convergence of QTRK (Algorithm~\ref{alg:QTRK}) using a similar proof technique to that in~\cite{steinerberger2021quantile} which is a modification from~\cite{quantHNRS20}. 
We recall the dimensions of the tensor linear system,
\[\tA \tX = \tB,\]
with $\tA \in \R^{m \times l \times n}$, $\tX \in \R^{l \times p \times n}$, and $\tB \in \R^{m \times p \times n}$.

First, in Lemma~\ref{qbound}, we bound the quantile value $Q_q(\tE^{(k})$ from above at each iteration $k$. 
In Lemma~\ref{lem:qtrk_corr}, we show that if we update some entries in $\tX$ that are affected by corruptions in $\tB$, then due to the upper bound on $Q_q(\tE^{(k})$, these corruptions cannot be too detrimental. 
Then, in Lemma~\ref{lem:qtrk_cautious}, we argue that if we update only non-corrupted entries, the improvement is the same as in the TRK method in~\cite{ma2022randomized}.  

\begin{lemma}[Bounding the Quantile Value]\label{qbound}
    Let the quantile level $q$ be $0<q \leq 1-\beta$ and let $\tX^{(k)}$ be any arbitrary tensor of the same dimensions as $\tX_\star$ (which for our purposes will be an arbitrary iterate in QTRK).  
    Then,
    $$
    Q_q(\tE^{(k)}) \leq \frac{\sigma_{\max}(\bcirc(\tA))}{\sqrt{mpn(1-\beta-q)}} \|\tX^{(k)} - \tX_\star\|_F,
    $$
    where $\tE^{(k)} = \tA\tX^{(k)} - \tB$ denotes the current residual tensor.
\end{lemma}

\begin{proof}
Suppose we restrict the residual $\tE^{(k)} $ to the set of entries in $\tB$ that are uncorrupted, 
$$C^c = \{(i,j,h)\in [m] \times [p] \times [n] : (\tB_{\text{corr}})_{ijh} = 0\}.$$ 
On this set, we have
\begin{align*}
\|\tE^{(k)}_{C^c}\|_F 
&= \|(\tA(\tX^{(k)} - \tX_*))_{C^c}\|_F
\leq   \|\tA(\tX^{(k)}-\tX_\star)\|_F 
\overset{\text{Lem }\ref{lem:boundsforFrobeniusnorm}}{\leq}  \sigma_{\max}(\bcirc(\tA))\|\tX^{(k)}-\tX_\star\|_F.
\end{align*}

By definition of $Q_q(\tE^{(k)})$ and the bound $0<q \leq 1-\beta$, we have at least $(1-q)mpn$ entries in $|\tE^{(k)}|$ that are greater than or equal to $Q_q(\tE^{(k)})$ and at least $(1 - q) (mpn) - \beta (mpn)$ of those correspond to entries in $\tB$ that have not been corrupted.
Thus, we have
$$
    \sqrt{mpn(1-\beta-q)}  Q_q(\tE^{(k)}) \leq  \|\tE^{(k)}_{C^c}\|_F \leq \sigma_{\max}(\bcirc(\tA))\|\tX^{(k)}-\tX_\star\|_F,
$$
and therefore
$$
    Q_q(\tE^{(k)}) \leq \frac{\sigma_{\max}(\bcirc(\tA))}{\sqrt{mpn(1-\beta-q)}} \|\tX^{(k)} - \tX_\star\|_F.
$$
\end{proof}

In the next lemma, we bound the expected error $\|\tX^{(k+1)} - \tX_\star\|_F$ given that in iteration $k$ a row $i_k$ is sampled where $B_{i_k::}$ has corruptions, but were not detected by the quantile test with value $Q_q(\tE^{(k)})$.
In other words, entries in $\tX^{(k)}$ that are affected by corruptions were updated.

\begin{lemma}[Corrupted Row Selected] 
\label{lem:qtrk_corr}
    Let the quantile level $q$ be $0<q \leq 1-\beta$.
    Suppose in a given iteration $k$, QTRK samples a row $i_k \in [m]$ such that $i_k\notin U_\star$ and $|\tE_{i_kjh}|  < Q_q(\tE^{(k)})$ for all $j \in [p]$ and $h \in [n]$. Then, 
    $$
    \E\left(\|\tX^{(k+1)} - \tX_\star\|_F \;|\; i_k\notin U_\star \right) \leq \|\tX^{(k)}-\tX_\star\|_F \left( 1 + \frac{\sigma_{\max}(\bcirc(\tA))}{\sqrt{m(1-\beta-q)}}\eta\right),
    $$
    where $\eta = \max \limits _i \sigma_{\max}(\bcirc(\tA _{i::}^\dagger))$ and the expectation is taken over the choice of row slice $i_k$.
\end{lemma}

\begin{proof}
If we select $i_k \notin U_\star$, then by the QTRK update rule, we have:

\begin{align*}
        \|\tX^{(k+1)} - \tX_\star\|_F &= \|\tX^{(k)} -\tA_{i_k::}^*(\tA_{i_k::}\tA_{i_k::}^*)^{-1}\left(\tA_{i_k::}\tX^{(k)}-\tB_{i_k::}\right) - \tX_\star\|_F \\
        &= \|\tX^{(k)} -\tA_{i_k::}^\dagger\left(\tA_{i_k::}\tX^{(k)}-\tB_{i_k::}\right) - \tX_\star\|_F\\
        &\leq \|\tX^{(k)} - \tX_\star\|_F + \|\tA_{i_k::}^\dagger\left(\tA_{i_k::}\tX^{(k)}-\tB_{i_k::}\right)\|_F.
    \end{align*} 
Then, by Lemma~\ref{lem:boundsforFrobeniusnorm} and matrix norm properties, we have
\begin{align*}
\|\tA_{i_k::}^\dagger\left(\tA_{i_k::}\tX^{(k)}-\tB_{i_k::}\right)\|_F & \leq \sigma_{\max}(\bcirc(\tA_{i_k::}^{\dagger})) \|(\tA_{i_k::}\tX^{(k)}-\tB_{i_k::})\|_F \\
& \leq \sqrt{pn} \sigma_{\max}(\bcirc(\tA_{i_k::}^{\dagger})) \|(\tA_{i_k::}\tX^{(k)}-\tB_{i_k::})\|_\infty.
\end{align*}
Given $i_k\notin U_\star$ and $|\tE_{i_kjh}|  < Q_q(\tE^{(k)})$ for all $j \in [p]$ and $h \in [n]$ and from Lemma~\ref{qbound}, we obtain
\begin{align*}
\|(\tA_{i_k::}\tX^{(k)}-\tB_{i_k::})\|_\infty & \leq  Q_q(\tE^{(k)}) \leq \frac{\sigma_{\max}(\bcirc(\tA))}{\sqrt{mpn(1-\beta-q)}} \|\tX^{(k)} - \tX_\star\|_F.\\
\end{align*}
Employing the definition $\eta = \max \limits _i \sigma_{\max}(\bcirc(\tA _{i::}^\dagger))$, we obtain
\begin{equation*}
\|\tX^{(k+1)} - \tX_\star\|_F \leq \|\tX^{(k)}-\tX_\star\|_F \left( 1 + \frac{\sigma_{\max}(\bcirc(\tA))}{\sqrt{m(1-\beta-q)}}\eta\right).
\end{equation*}
Then, taking the expected value conditioned on choosing $i_k$, where $i_k\notin U_\star$, yields the result.
\end{proof}

\begin{lemma}[Uncorrupted Row Selected,~\cite{ma2022randomized}]
\label{lem:qtrk_cautious}
Suppose in a given iteration $k$, QTRK samples a row $i_k \in U_\star$.  Then,
  $$
    \E\left(\|\tX^{(k+1)} - \tX_\star\|_F \;|\; i_k\in U_\star \right) \leq \left( 1 - \sigma_{\min}(\E_{U_{\star}}\left[\bcirc(\tens{P}_i)\right]\right))\|\tX^{(k)}-\tX_\star\|_F ,
    $$
    where $\E_{U_{\star}}$ denotes the expectation over the row selection in $U_\star$ and ${\tens P}_{i}$ the orthogonal projection operator defined in Equation \eqref{eq:projop}.
\end{lemma}
\begin{proof}
This follows immediately from the convergence guarantee of TRK~\cite[Theorem 6]{ma2022randomized} applied to the subsystem of row slices indexed by $U_{\star}$. 
Further, it is shown that $1 - \sigma_{\min}(\E_{U_{\star}}\left[\bcirc(\tens{P}_i)\right]) < 1$.
\end{proof}

Putting Lemmas~\ref{lem:qtrk_corr} and~\ref{lem:qtrk_cautious} together gives a convergence guarantee for QTRK.

\begin{theorem}[QTRK Convergence Guarantee]
\label{thm:qtrk}
Let the quantile level $q$ be $0<q \leq 1-\beta$.
The expected error at
the $k$th iteration of QTRK satisfies
$$
 \E\|\tX^{(k)} - \tX_\star\|_F \leq R^k\|\tX^{(0)} - \tX_\star\|_F,
$$
where 
$$
R = \left(1 - \frac{1-\brow}{q}\right)\left( 1 + \frac{\sigma_{\max}(\bcirc(\tA))}{\sqrt{m(1-\beta-q)}}\eta\right) +
\left(1 - \frac{\brow}{q}\right)\left( 1 - \sigma_{\min}(\E_{U_{\star}}\left[\bcirc(\tens{P}_i)\right]\right),
$$
where $\brow$, $\tens P_i$, and $U_{\star}$ is defined previously in Section~\ref{sec:QTRK} and $\eta = \max \limits _i \sigma_{\max}(\bcirc(\tA _{i::}^\dagger))$ as in Lemma~\ref{lem:qtrk_corr}.

\end{theorem}
\begin{proof}
We first note that, by definition of $\brow$, we have $|U_*^c| = \brow m$. Then, noting that the largest acceptable row set $U$ would be given when all $(1-q)mpn$ entries in $|\tE^{(k)}|$ that are at least $Q_q(\tE^{(k)})$ lie in $(1-q)m$ rows, we have $|U| \le m - (1-q)m = qm$. Thus, we have that $$\mathbb{P}(i_k \in U_*) = \frac{|U \cap U_*|}{|U|} \le 1 - \frac{|U_*^c|}{|U|} \le 1 - \frac{\brow}{q}.$$  A similar calculation yields $\mathbb{P}(i_k \not\in U_*) \le 1 - \frac{1-\brow}{q}.$

Then, using the law of total expectation, we have:
\begin{align*}
 \E\|\tX^{(k+1)} &- \tX_\star\|_F   = \underbrace{\E\left(\|\tX^{(k+1)} - \tX_\star\|_F | {i_k} \notin U_{\star}\right)}_{\text{Lemma~\ref{lem:qtrk_corr}}} \mathbb P({i_k} \notin U_{\star}) +
\underbrace{\E\left(\|\tX^{(k+1)} - \tX_\star\|_F | {i_k} \in U_{\star} \right)}_{\text{Lemma~\ref{lem:qtrk_cautious}}} \mathbb P({i_k} \in U_{\star})   \\
&\leq \left\{ \left(1 - \frac{1-\brow}{q}\right)\left( 1 + \frac{\sigma_{\max}(\bcirc(\tA))}{\sqrt{m(1-\beta-q)}}\eta\right) +
\left(1 - \frac{\brow}{q}\right)\left( 1 - \sigma_{\min}(\E_{U_{\star}}\left[\bcirc(\tens{P}_i)\right]\right)\right\}\|\tX^{(k)}-\tX_\star\|_F,
\end{align*}
which by rearranging and applying recursion completes the claim. 
\end{proof}

\begin{remark}
    We comment here that in the large-scale regime of high number of rows $m$, a small enough proportion of corruptions $\brow$ will guarantee that $R<1$ and thus that the algorithm converges to the solution. The quantity below which this parameter needs to be is, of course, dependent on the other parameters in the theorem and the conditioning of the system, but this is typically the regime where these methods are applied (large-scale data and bounded proportion of corruptions, with no need to bound the \textit{magnitude} of the corruptions).  
\end{remark}

We also immediately observe the following corollary from Theorem~\ref{thm:qtrk}.

\begin{corollary}[Case Without Corruptions]
\label{cor:qtrk_no_corr} 
In the case of no corruptions; that is $\brow=0$, $q= 0$, and $U_{\star} = [m]$, QTRK recovers the same convergence bound as TRK~\cite{ma2022randomized},
    $$
 \E\|\tX^{(k)} - \tX_\star\|_F \leq (1-\sigma_{\min}(\E\left[\bcirc(\tens{P}_i)\right]))^k\|\tX^{(0)} - \tX_\star\|_F.
$$
In~\cite[Theorem 11]{ma2022randomized} the authors show that this also provides a bound specified through the Fourier domain,
    $$
 \E\|\tX^{(k)} - \tX_\star\|_F \leq \left(1-\min_{h\in[n-1]}\frac{\sigma_{\min}(\tens{\widehat{A}}_{::h})}{\sqrt{m}\|\tens{\widehat{A}}_{::h}\|_{\infty,2}}\right)^k\|\tX^{(0)} - \tX_\star\|_F,
$$
where $\tens{\widehat{A}}$ denotes the tensor resulting from applying the discrete Fourier transform (DFT) matrix to each of the tube fibers of $\tA$, $\tens{\widehat{A}}_{::h}$ denotes the $h$th frontal slice of $\tens{\widehat{A}}$, and $\|\cdot\|_{\infty,2}$ denotes the operator norm from $L_{\infty}$ to $L_2$.
\end{corollary}

Note that, as expected from its formulation, QTRK penalizes an entire row slice if it corresponds to a single (entry) corruption in $\tB$. In that sense, the convergence guarantee for QTRK is ``density blind" to the corruptions -- for the right choice of quantile $q$, the bound is the same whether a row slice has a single corruption or is completely corrupted in every entry of the slice. In the former case, one may suppose that the row slice still contains enough information for QTRK to use in some way. Additionally, one might hope that utilizing such information could yield a convergence guarantee that depends only on $\beta$ and not $\brow$. We consider this query in the next section.

\section{A Conversation on Masking}\label{sec:masking}

In the matrix setting, we seek to solve a linear system $\mat A \vx = \vb$ with corruptions in the vector $\vb$.
In QRK~\cite{quantHNRS20}, corruptions are estimated using a quantile test that dictates whether or not to project onto the selected row space. 
QTRK, our proposed method in Section~\ref{sec:QTRK},  is defined similarly but on row slices of the tensor $\tA$. 
However, this severely limits the number of corruptions that can be tolerated.
If a single entry $\tB_{ijh}$ is estimated to be corrupted, QTRK disregards all information in the row slice $\tA_{i::}$.
In the tensor setting, we therefore have more flexibility when choosing what information in row slice $\tA_{i::}$ we wish to disregard and may instead do a partial projection.  

Suppose in some iteration $k$ we select a row slice that contains an entry $\tB_{ijh}$ that is labeled corrupted by QTRK.
The row slice $\tB_{i::}$ still contains useful information outside of the single corruption $\tB_{ijh}$. 
Remark~\ref{lem:corr_propag} notes that a perturbation in $\tB_{ijh}$ makes the consistent system 
$\tA \tX = \tB$ inconsistent only in the $j$th column slice: $\tA \tX_{:j:} \not= \tB_{:j:}$.
Then for the $k$th iteration, we may want to perform a projection where we disregard or \textit{mask} the column slices $\tX_{:j:}$ that are affected by the corruptions $\tB_{ijh}$.  This intuition is the crux of our proposed method mQTRK, presented in Algorithm~\ref{alg:mQTRK}.

The method uses the same quantile test as QTRK to estimate the locations of the corrupted entries in $\tB$. 
Then, for the selected row slice at iteration $k$, it only updates the column slices in $\tX$ that are not affected by the entries in $\tB$ that are estimated to be corrupted.
The set of these column indices is defined as in Definition \ref{def:Vi} and denoted by $\hat{V}_i$ in Algorithm~\ref{alg:mQTRK}.
We will observe that mQTRK works well in practice, but it suffers from unlikely but possible catastrophic situations that impede convergence.
We propose the masked Quantile Tensor Randomized Kaczmarz (mQTRK) method that employs masking in QTRK.
The pseudocode is given in Algorithm~\ref{alg:mQTRK}. 

\begin{algorithm}[h!]
\caption{Masked Quantile Tensor Randomized Kaczmarz (mQTRK) }\label{alg:mQTRK}
\begin{algorithmic}[1]
\Procedure{mQTRK}{$\tens{A},\tens{B}$, quantile level $q$, $\tens{X}^{(0)}$, $K$}
\For{$k = 0, 1, \ldots, K-1$}

\State{Compute residual tensor $\tE^{(k)}= \tA \tX^{(k)} -\tB$ \text{ and $q$-quantile } $Q_q(\tE^{(k)})$} \Comment{See Definition~\ref{def:q-quantile}}

\State{Sample row slice $i\sim\text{Uniform}(1, \ldots, m)$}
         
\State{Estimated corrupted column slices $\hat{V}_i^c = \left \{ j \in [p] :  |\tE^{(k)}_{ijh}| > Q_q(\tE^{(k)}) \text{ for some } h \in [n] \right \}$ }

\State{Update the approximation $\tens{X}^{(k+1)}_{:\hat{V}_i:} = \tens{X}^{(k)}_{:\hat{V}_i:} - \tA_{i::}^* \left (\tA_{i::}\tA_{i::}^* \right )^{-1} \left (\tA_{i::} \tX^{(k)}_{:\hat V_{i}:} - \tB_{i \hat V_{i} :} \right )$}
	
\EndFor{} \\
\Return{$\tX^{(K)}$}
	\EndProcedure
\end{algorithmic}
\end{algorithm}

We show that we cannot guarantee convergence of mQTRK, unless we guarantee our iterates remain in a ``nice" subspace of $\R^{l \times p \times n}$
that cannot be controlled either algorithmically or theoretically (see the discussion below).  
Experimentally we will see that mQTRK exhibits good convergence, but through some theoretical observations, we will explain that this cannot be guaranteed.

Fix an iteration $k$.
Suppose we are fortunate and select $i_k$ such that $\hat V_{i_k}\subseteq V_{i_k}$. 
Thus, mQTRK will not update any column slices in $\tX^{(k)}$ that are affected by corruptions.
However, it may be the case that too few column slices in $\tX^{(k)}$ are updated.
Note that since $\hat V_{i_k} \subseteq  V_{i_k}$ we have $\tA_{i_k::} \left (\tX_\star \right )_{:\hat V_{i_k}:} = \tB_{i_k,\hat V_{i_k},:}$. 
We obtain the following,

\begin{align*}
\left\|\tX^{(k+1)}-\tX_\star\right\|_F^2 & 
=\left\|\left(\tX^{(k+1)}-\tX_\star\right)_{:\hat V_{i_k}^c:} \right\|_F^2+\left\|\left(\tX^{(k+1)}-\tX_\star\right)_{:\hat V_{i_k}:}\right\|_F^2 \\  
& =\left\|\left(\tX^{(k)}-\tX_\star\right)_{:\hat V_{i_k}^c:}\right\|_F^2 + \left \|\tX^{(k)}_{:\hat V_{i_k}:} - \tA_{i_k::}^*\left (\tA_{i_k::}\tA_{i_k::}^* \right )^{-1} \left (\tA_{i_k::} \tX^{(k)}_{:\hat V_{i_k}:} - \tB_{i_k,\hat V_{i_k},:} \right ) - \left (\tX_\star \right )_{:\hat V_{i_k}:} \right \|_F^2 \\
& = \left\|\left(\tX^{(k)}-\tX_\star\right)_{:\hat V_{i_k}^c:}\right\|_F^2 
+ \left \|\left (\tX^{(k)} - \tX_\star \right )_{:\hat V_{i_k}:} - \tA_{i_k::}^*\left (\tA_{i_k::}\tA_{i_k::}^* \right )^{-1} \tA_{i_k::} \left (\tX^{(k)} - \tX_\star \right )_{:\hat V_{i_k}:}   \right \|_F^2 \\
& = \left\|\left(\tX^{(k)}-\tX_\star\right)_{:\hat V_{i_k}^c:}\right\|_F^2 
+ \left \| \left( I- \tA_{i_k::}^*\left (\tA_{i_k::}\tA_{i_k::}^* \right )^{-1} \tA_{i_k::} \right)\left (\tX^{(k)} - \tX_\star \right )_{:\hat V_{i_k}:} \right\|_F^2  \\
& = \left\|\tX^{(k)}-\tX_\star\right\|_F^2 - \left\| {\tens P}_{i_k}\left ( \left(\tX^{(k)} - \tX_\star\right)_{:\hat V_{i_k}:} \right )\right\|_F^2,
\end{align*}
by definition \eqref{eq:projop} of the projection.
The last equality holds from the properties of orthogonal projection operators and 2-norm inner products. 
Then, 
we have

\begin{align*}
    \mathbb{E}_{i_k} \left [ \left\|\tX^{(k+1)}-\tX_\star\right\|_F^2 \right ] & = \left\|\tX^{(k)}-\tX_\star\right\|_F^2 - \mathbb{E}_{i_k} \left\| {\tens P}_{i_k}\left ( \left(\tX^{(k)} - \tX_\star\right)_{:\hat V_{i_k}:} \right )\right\|_F^2.
\end{align*}
Now, at this point, following similar analysis strategies in e.g.,~\cite{quantHNRS20}, we might hope that
$$\mathbb{E}_{i_k} \left\| {\tens P}_{i_k}\left ( \left(\tX^{(k)} - \tX_\star\right)_{:\hat V_{i_k}:} \right )\right\|_F^2 \ge \psi \left\|\tX^{(k)}-\tX_\star\right\|_F^2$$ for some $\psi > 0$.   However, this is not likely to be the case in general.  
Additionally, this is hard to guarantee in practice, but shows that convergence is likely possible unless your iterates move into a ``bad" space.
This will also highlight an important fact, that for a fixed $\tX^{(k)}-\tX_\star$ and any row slice choice $i_k$, there is always some possibility that the column set $\hat{V}_{i_k}$ is such that mQTRK makes no progress this iteration. 
This happens because it is always possible that, even when all of our selected indices in $\hat{V}_{i_k}$ are uncorrupted, we mask out column indices in $\hat{V}_{i_k}^c$ that were vital for making progress. In other words, ${\tens P}_{i_k}\left ( \left(\tX^{(k)} - \tX_\star \right)_{:\hat V_{i_k}:} \right )$ may be zero in this unlucky case. In fact, mQTRK is \textit{by design} going to choose column indices $\hat V_{i_k}$ that make the norms of these projections small (because large ones appear like corruptions).
In fact, one can see that, given knowledge about the system $\tA$ and $\tB$, one can select an initialization such that mQTRK never makes progress at all, see e.g., Remark~\ref{rem:mQTRK-bad-behavior} below.  Therefore, in an adversarial setting, mQTRK may catastrophically fail to converge. With random initialization, of course, it is unlikely this will happen, which is why we see good empirical performance, shown in the next section.

\begin{remark}
\label{rem:mQTRK-bad-behavior}
    Consider the following example in which mQTRK will fail catastrophically. 
    Let $\tA$ and $\tB$ be given, with unique solution $\tX_\star$, so that $\tA\tX_\star = \tB.$  
    Corrupt $\tB$ such that every row slice $i$ contains at least one corruption in the tube fiber $\tB_{i1:} \in \R^n$ obtained by fixing row index $i$ and column index 1.
    Note this can be done with $m$ corruptions, meaning $\beta=1/(pn)$. Create $\tX^{(0)}$ to be equal to $\tX_\star$ everywhere except a single entry, say $\tX^{(0)}_{111} \ne (\tX_\star)_{111}$.  Assume we use a quantile $q$ arbitrarily close to $1-\beta$.  It is clear we need to update the $\tX^{(0)}_{111}$ entry or we will not make progress toward the solution.  However, if the corruptions are each large enough, each $\hat{V}_i^c$ will be sure to contain the first column, and thus mQTRK will not make any progress.    
\end{remark}

\begin{remark}
\label{rem:mQTRK-mod-behavior}
Let us remark on an opposing situation. 
    Suppose $\brow$ as defined in \eqref{eq:betarow} is small. Suppose $q$ is chosen small enough that we do not admit too many corruptions, and large enough that we do not mask too many columns, then mQTRK should do no worse than QTRK since at worst there is still a selection of rows for which even mQTRK will perform a full projection. This is not guaranteed to happen deterministically though -- it could be that each algorithm identifies the same locations for estimated corruptions, but QTRK ignoring an entire row is actually better than mQTRK using that row and projecting using missed corruptions.  However, more often than not, we expect mQTRK to utilize rows and make progress when QTRK would ignore those rows entirely. We explore this comparison empirically more in the next section and use it as motivation for the following theoretical guarantee.
\end{remark}

\begin{theorem}\label{thm:mqtrk}
Let $1-1/(pn) < q < 1 - \beta$.  The expected error at the $k$th iteration of mQTRK exhibits the following convergence guarantee:
$$
\E\|\tX^{k} - \tX\|_F \leq \left\{q(1-R)+ \left(\min \{1-\beta, (1-q)pn\} + \beta pn\right)\left( 1 + \frac{\sigma_{\max}(\bcirc(\tA))}{\sqrt{m(1-\beta-q)}}\eta\right)\right\}^k\|\tX^{(0)} - \tX\|_F.
$$
\end{theorem}
\begin{proof}
    We use the law of total expectation and condition on whether $V_{i_k}^c$ is empty or not, and in the first case whether $\hat{V}_{i_k}^c$ is empty or not, that is, whether or not mQTRK indicates there may be a corruption.  We have
  \begin{align*}
 \E\|\tX^{(k+1)} &- \tX_\star\|_F   = \underbrace{\E\left(\|\tX^{(k+1)} - \tX_\star\|_F | V_{i_k}^c =\emptyset \text{ and } \hat{V}_{i_k}^c =\emptyset\right)}_{\text{(i) Setting of TRK}} \mathbb P(V_{i_k}^c =\emptyset \text{ and } \hat{V}_{i_k}^c =\emptyset) \\
 &+\underbrace{\E\left(\|\tX^{(k+1)} - \tX_\star\|_F |V_{i_k}^c =\emptyset \text{ and } \hat{V}_{i_k}^c \not=\emptyset\right)}_{\text{(ii) uncorrupted projection (overly cautious)}} \mathbb P(V_{i_k}^c =\emptyset \text{ and } \hat{V}_{i_k}^c \not=\emptyset)  \\
 &+ \underbrace{\E\left(\|\tX^{(k+1)} - \tX_\star\|_F |V_{i_k}^c \not=\emptyset\right)}_{\text{(iii) Lemma~\ref{lem:qtrk_corr}}} \mathbb P(V_{i_k}^c \not=\emptyset).  
\end{align*}
First, we begin by bounding the probabilities $\mathbb P(\hat{V}_{i_k}^c \not=\emptyset)$ and $\mathbb P(V_{i_k}^c \not=\emptyset)$. We begin by counting the number of rows with $\hat{V}_i^c \not= \emptyset$, that is with entries flagged for corruption.  This number is at least $(1-q)m$ and no more than $(1-q)mpn$.  Thus, the probability of randomly sampling such a row from the $m$ rows satisfies $$1-q \le \mathbb P(\hat{V}_{i_k}^c \not=\emptyset) \le (1-q)pn.$$  Similarly, the number of rows with $V_i^c \not= \emptyset$, that is affected by corruption, is at least $\beta m$ and no more than $\beta mpn$.  Thus, the probability of sampling such a row from the $m$ total rows satisfies $$\beta \le \mathbb P(V_{i_k}^c \not=\emptyset) \le \beta pn.$$
Now, the probability in case (i), $\mathbb P(V_{i_k}^c =\emptyset \text{ and } \hat{V}_{i_k}^c =\emptyset)$, satisfies $$\mathbb P(V_{i_k}^c =\emptyset \text{ and } \hat{V}_{i_k}^c =\emptyset) \le P(\hat{V}_{i_k}^c =\emptyset) \le q.$$ Next, the probability in case (ii), $\mathbb P(V_{i_k}^c =\emptyset \text{ and } \hat{V}_{i_k}^c \not=\emptyset)$ satisfies $$\mathbb P(V_{i_k}^c =\emptyset \text{ and } \hat{V}_{i_k}^c \not=\emptyset) \le \min \{P(V_{i_k}^c =\emptyset), P(\hat{V}_{i_k}^c \not=\emptyset)\} = \min \{1-\beta, (1-q)pn\}.$$  Finally, the probability in case (iii), $\mathbb P(V_{i_k}^c \not=\emptyset)$ satisfies $$\mathbb P(V_{i_k}^c \not=\emptyset) \le \beta pn.$$

Now we apply bounds that are most likely loose.  We apply Lemma~\ref{lem:qtrk_corr} to the term marked (iii) by noting that $V_{i_k}^c \not= \emptyset$ if and only if $i_k \not\in U_*$.  To bound the term marked (ii) we observe it is at most $\|\tX^{(k)}-\tX_\star\|_F$ (since performing an uncorrupted partial projection can do no harm) and thus can also be loosely upper bounded again by the upper bound of Lemma~\ref{lem:qtrk_corr}.  For the term marked (i), we observe that, conditioned on selecting a row for which no corruption has been flagged and no corruption is present, we perform precisely a standard TRK update, and may apply our main theorem, Theorem~\ref{thm:qtrk} there.  Putting these together and using the notation as in Theorem~\ref{thm:qtrk} yields 

  \begin{align*}
 \E\|\tX^{(k+1)} &- \tX_\star\|_F   \leq q(1-R)\|\tX^{(k)} - \tX_\star\|_F + \mathbb P(V_{i_k}^c =\emptyset \text{ and } \hat{V}_{i_k}^c \not=\emptyset)\|\tX^{(k)} - \tX_\star\|_F \\
 &\quad+ \mathbb P(V_{i_k}^c \not=\emptyset)\left( 1 + \frac{\sigma_{\max}(\bcirc(\tA))}{\sqrt{m(1-\beta-q)}}\eta\right)\|\tX^{(k)} - \tX_\star\|_F\\
 &\leq q(1-R)\|\tX^{(k)} - \tX_\star\|_F \\&\quad+ \left(\mathbb P(V_{i_k}^c =\emptyset \text{ and } \hat{V}_{i_k}^c \not=\emptyset)
 + \mathbb P(V_{i_k}^c \ne\emptyset)\right)\left( 1 + \frac{\sigma_{\max}(\bcirc(\tA))}{\sqrt{m(1-\beta-q)}}\eta\right)\|\tX^{(k)} - \tX_\star\|_F\\
&\leq q(1-R)\|\tX^{(k)} - \tX_\star\|_F + \left(\min \{1-\beta, (1-q)pn\} + \beta pn\right) \left( 1 + \frac{\sigma_{\max}(\bcirc(\tA))}{\sqrt{m(1-\beta-q)}}\eta\right)\|\tX^{(k)} - \tX_\star\|_F\\
 &= \left\{q(1-R)+ \left(\min \{1-\beta, (1-q)pn\} + \beta pn\right)\left( 1 + \frac{\sigma_{\max}(\bcirc(\tA))}{\sqrt{m(1-\beta-q)}}\eta\right)\right\}\|\tX^{(k)} - \tX_\star\|_F.
\end{align*}

Iterating recursively completes the claim.
\end{proof}

Note that this result recovers the same result as Theorem~\ref{thm:qtrk} when $\beta\rightarrow 0$ and thus $q\rightarrow 1$. Note also that the bound in Theorem~\ref{thm:mqtrk} is very pessimistic. Indeed, as seen by the proof, there are two sources of pessimism. First, the bound implicitly assumes the worst case alignment of corruptions, i.e., a pattern that affects the largest possible number of rows and columns. Second, the term (ii) corresponding to the masked but uncorrupted projection is trivially bounded as if that projection made no progress at all. Because of the masking, the concern is that in the uncorrupted case, the masking ignores columns for precisely the columns that require an update to make progress. This is not likely to happen in all iterations, and a more detailed analysis, perhaps with stochastic assumptions on the corruption locations, could improve this guarantee. 

\section{Numerical Experiments}\label{sec:experiments}

In this section, we illustrate the empirical performance of QTRK (Algorithm~\ref{alg:QTRK}) and mQTRK (Algorithm~\ref{alg:mQTRK}) for solving tensor linear systems $\tA \tX = \tB$ under various corruption settings.
Our code is publicly available\footnote{\url{https://github.com/lara-kassab/qtrk-code/}}. In our implementations of QTRK and mQTRK, we use the publicly available Tensor-Tensor Product Toolbox \cite{lu2018tproduct}.

\subsection{Experimental Design}
\label{sec:exp design}
We use the same definitions and notations as presented in Section \ref{sec:QTRK}.
In the construction of synthetic data, for all of the experiments, the entries of $\tA$ and $\tX_\star$ are sampled i.i.d.\ from a standard Gaussian distribution.
The uncorrupted observed measurements $\tB_\star$ are computed as $\tA\tX_\star$ and the observed right-hand side of the corrupted system is defined as $\tB = \tB_\star + \tB_{\text{corr}}$. 
In the different experiments, the tensor of corruptions $\tB_{\text{corr}}$ is generated differently.
In particular, we vary the number of corruptions, their magnitudes, and their distribution across the row slices of $\tB_{\text{corr}}$.

As in Definition~\ref{def:beta}, $0 \leq \beta \leq 1$ is the proportion of corrupted entries in $\tB$ and $0 \leq \brow \leq 1$ is the proportion of row slices in $\tB$ that contain corruptions.
We define $0 \leq \tbeta , \tbrow \leq 1$. 
In construction of the $\tB_{\text{corr}}$, we sample uniformly at random without replacement~$(m \tbrow) \in \mathbb N$ row slices to which the corrupted entries are restricted to. 
From these row slices, we then sample uniformly at random with replacement a total of~$(\tbeta mpn) \in \mathbb N$ entries to be corrupted.  
With this sampling scheme, not every of the $(m \tbrow) \in \mathbb N$ rows is guaranteed to contain a corruption; thus,~$\tbrow \geq \brow$.
It is also not guaranteed that there are~$(\beta mpn) \in \mathbb N$ unique entries that are corrupted; thus,~$\tilde \beta \geq \beta$.  

In all sets of experiments, we apply QTRK and/or mQTRK to solve a corrupted tensor linear system defined by $\tA \in \R^{25 \times 5 \times 10}$, $\tB \in \R^{25 \times 4 \times 10}$, and $\tX \in \R^{5 \times 4 \times 10}$. 
We denote the dimensions by $m = 25$, $p = 4$, $n = 10$, and $l = 5$. 
We initialize QTRK and mQTRK with $\tX^{(0)}$ where the entries are sampled i.i.d.\ from a standard Gaussian distribution and we run each algorithm for 2000 iterations.
We define the relative error as $\| \tX_\star - \tX^{(k)}\|_F / \| \tX_\star\|_F$ where $\tX^{(k)}$ is the output at the $k$-th iteration of either algorithm.

In Figures~\ref{fig:qtrk-large-corr} and~\ref{fig:mqtrk-large-corr}, the magnitudes of the corruptions are sampled from $\mathcal{N}(100,20)$ and in Figures~\ref{fig:qtrk-small-corr} and~\ref{fig:mqtrk-small-corr} the magnitudes are sampled from $\mathcal{N}(10,5)$.
In all figures, we report the median relative residual error in each iteration over 150 trials. 
In each trial, we define a new system $\tA \tX = \tB$ with a new initialization $\tX^{(0)}$. 
In Figures~\ref{fig:qtrk-large-corr} through \ref{fig:comp_plots}, the y-axes of all plots are set to $\log$ scale and in each figure all subplots have the same legend.
Further, for quantile values $q = 1 - \tbeta$ the corresponding graphs are thicker for convenience.

\begin{figure}[h!]
    \centering  
    \includegraphics[width=0.3\textwidth]{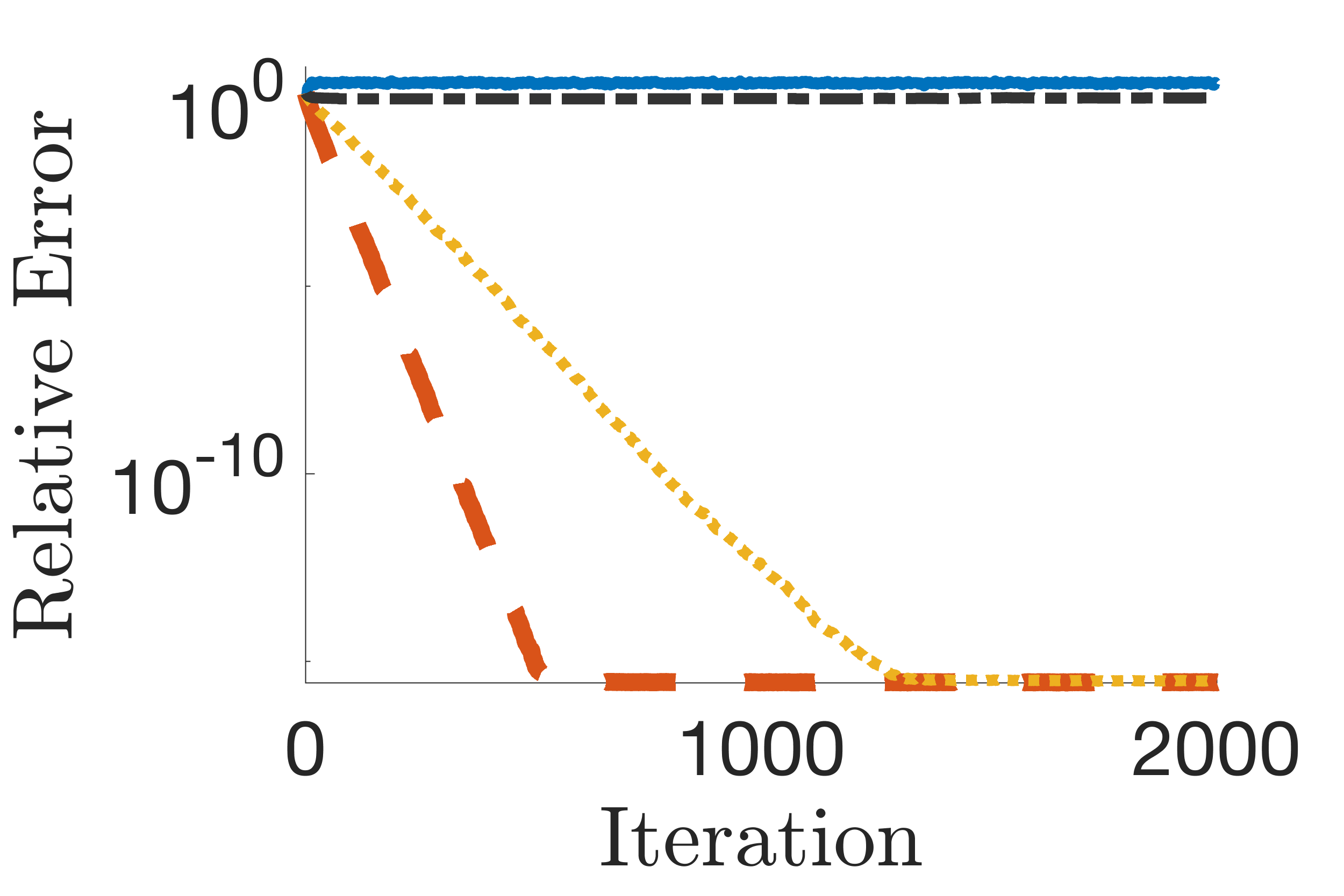}%
    \includegraphics[width=0.3\textwidth]{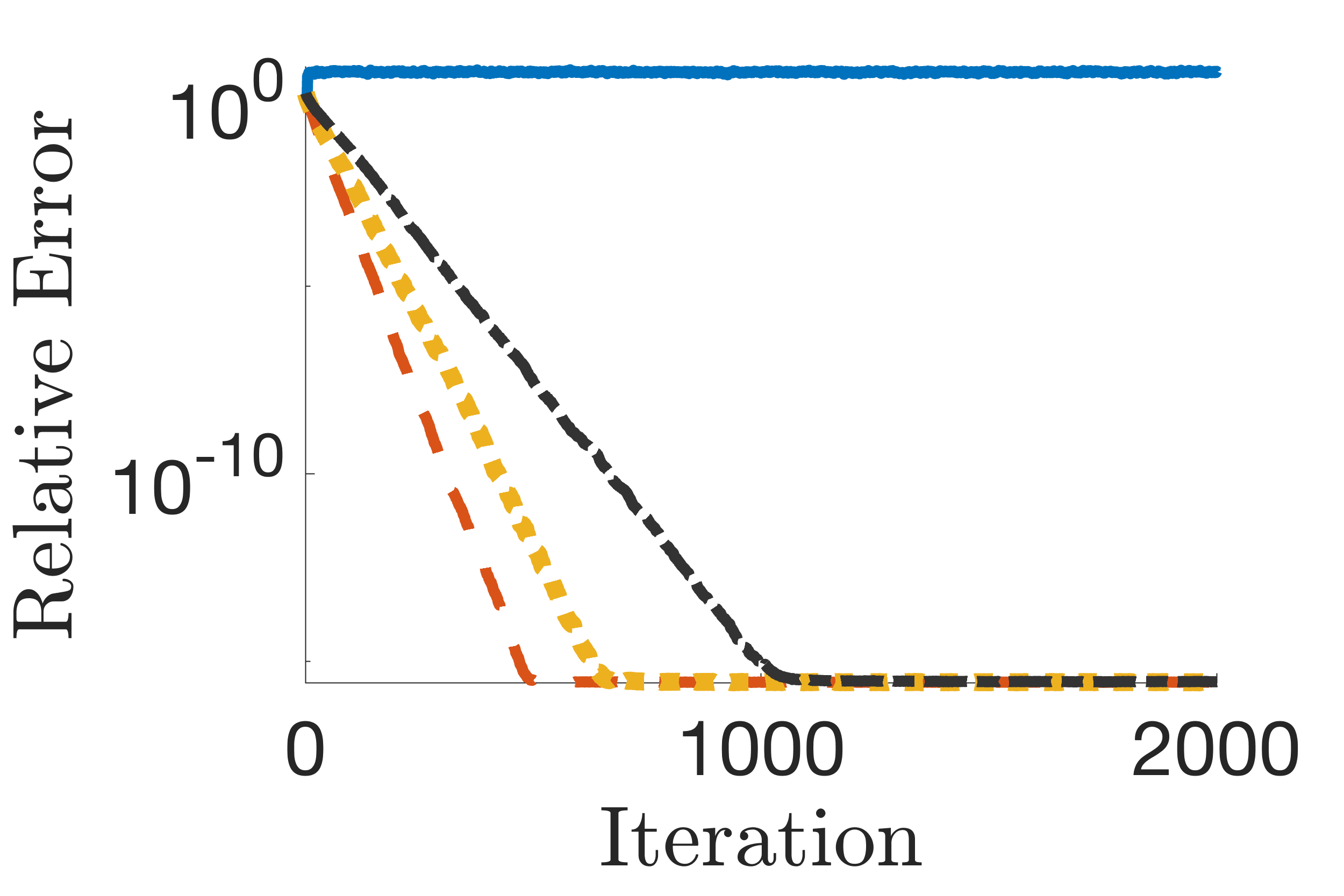}%
    \includegraphics[width=0.3\textwidth]{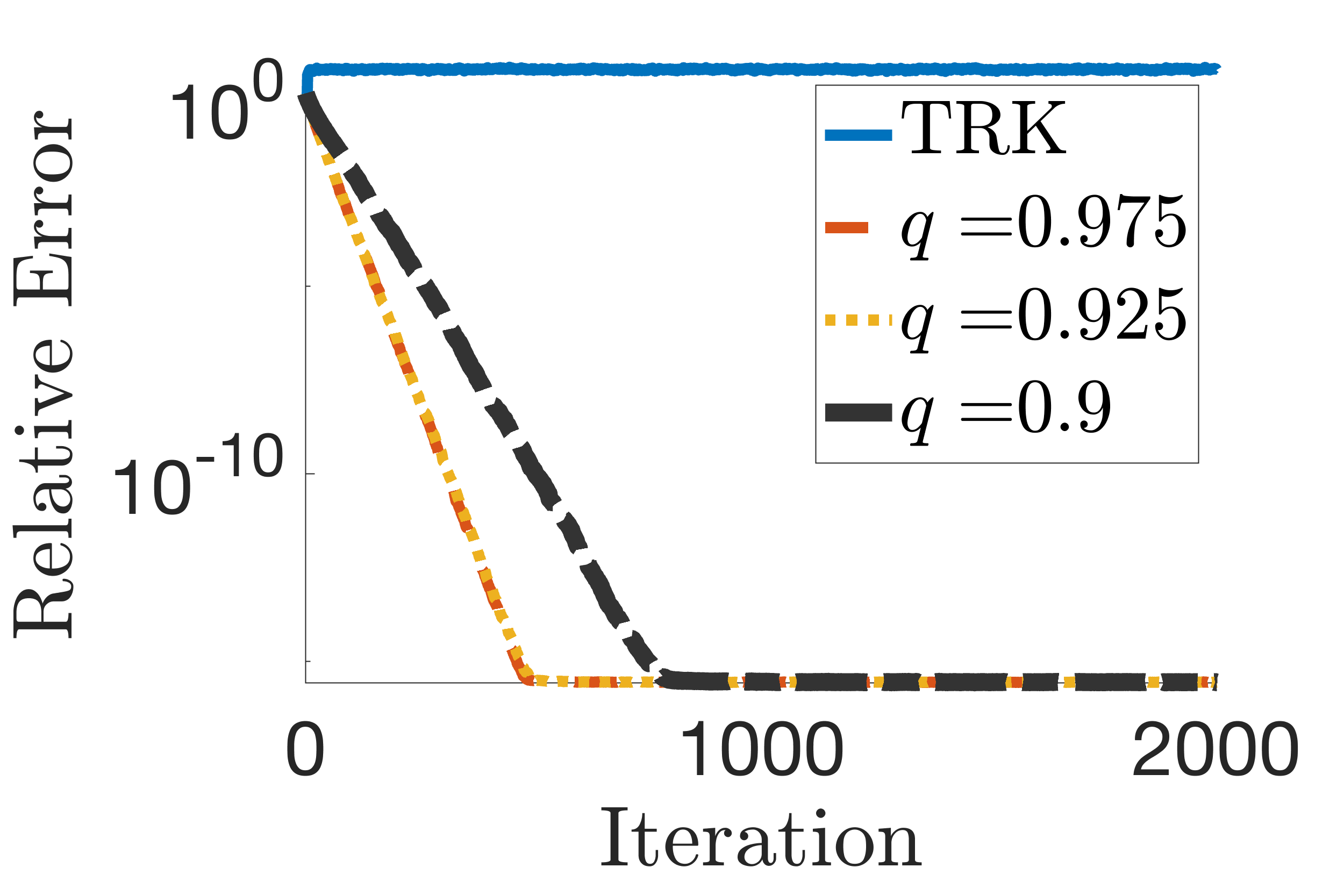}
    \\
    \includegraphics[width=0.3\textwidth]{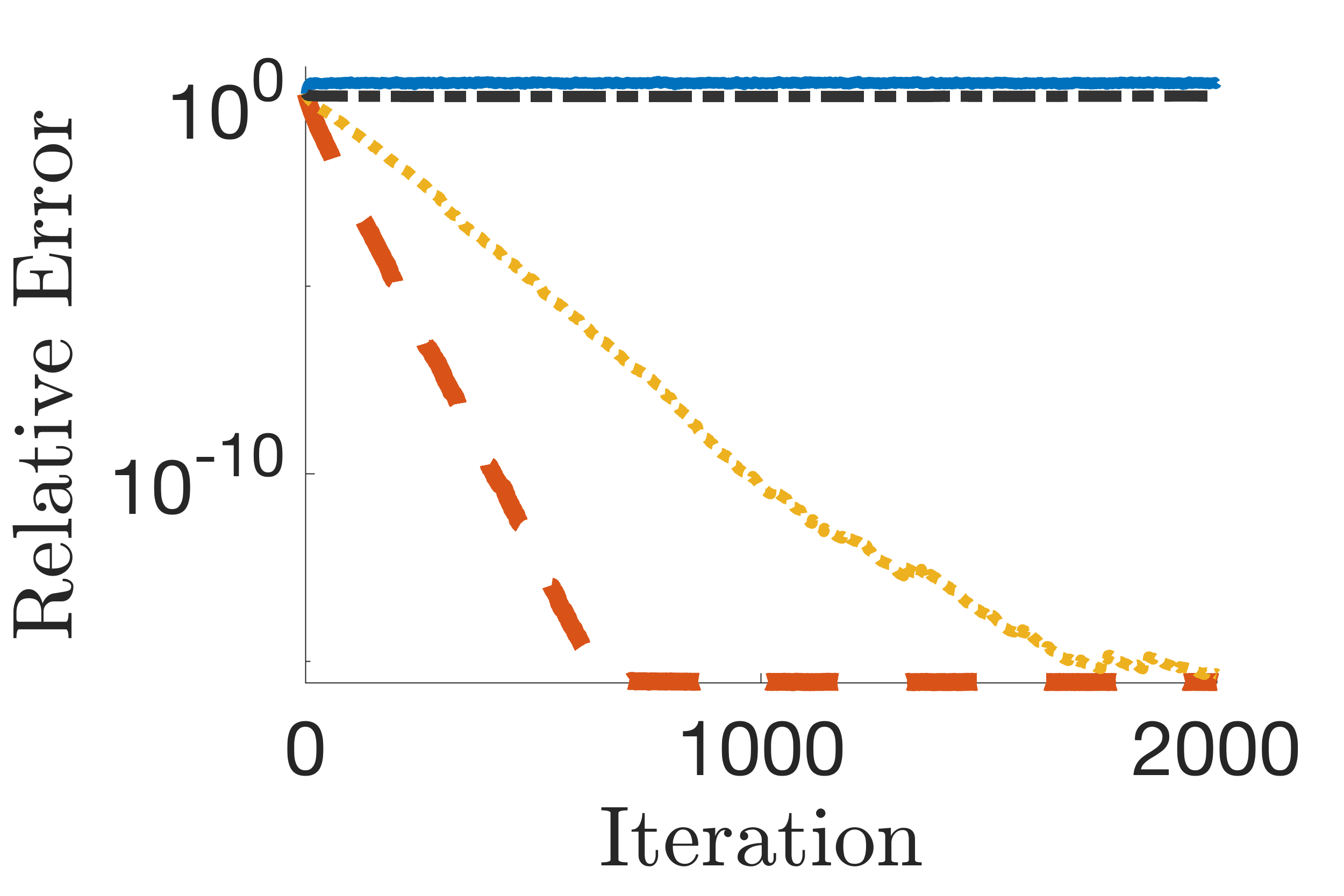}%
    \includegraphics[width=0.3\textwidth]{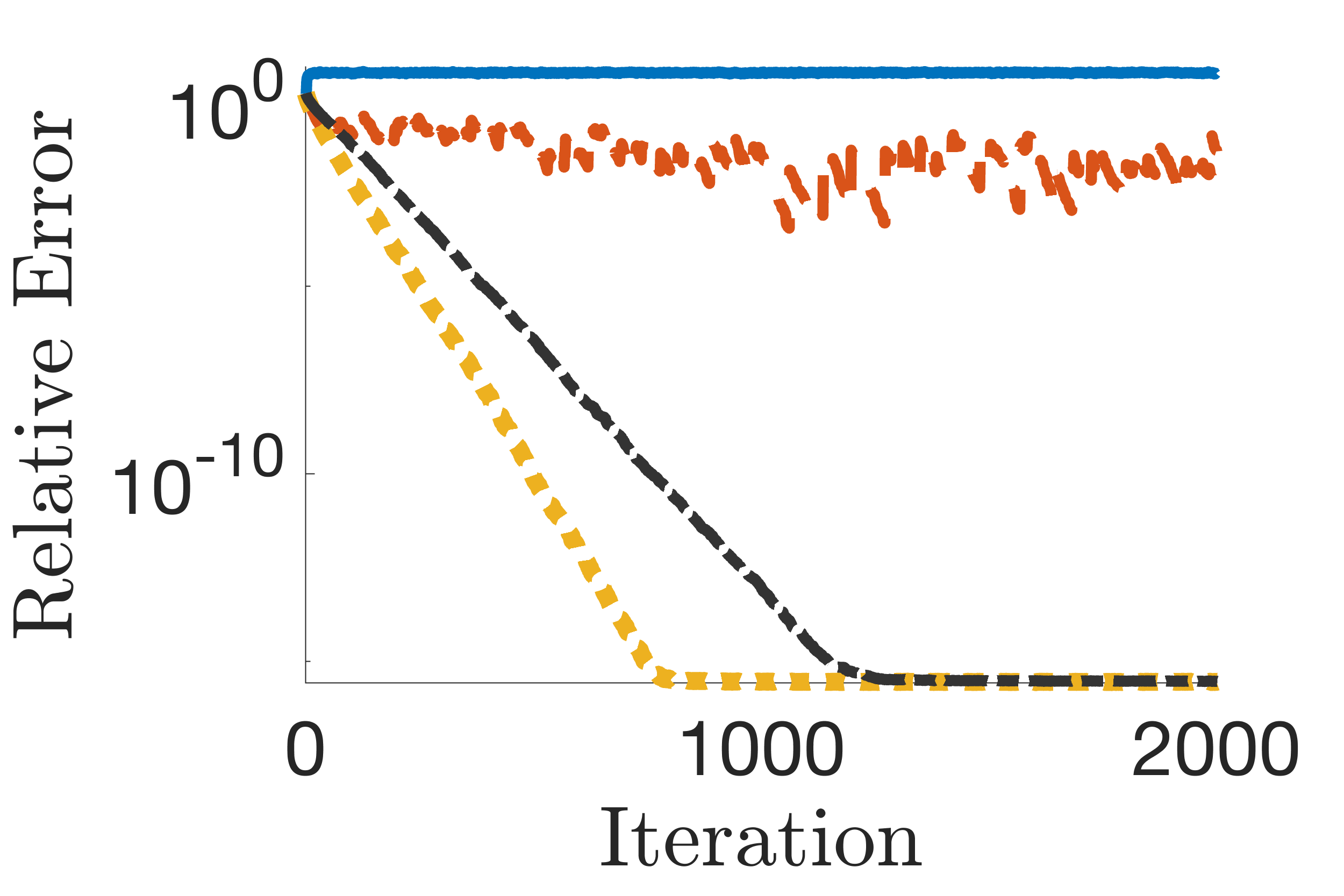}%
    \includegraphics[width=0.3\textwidth]{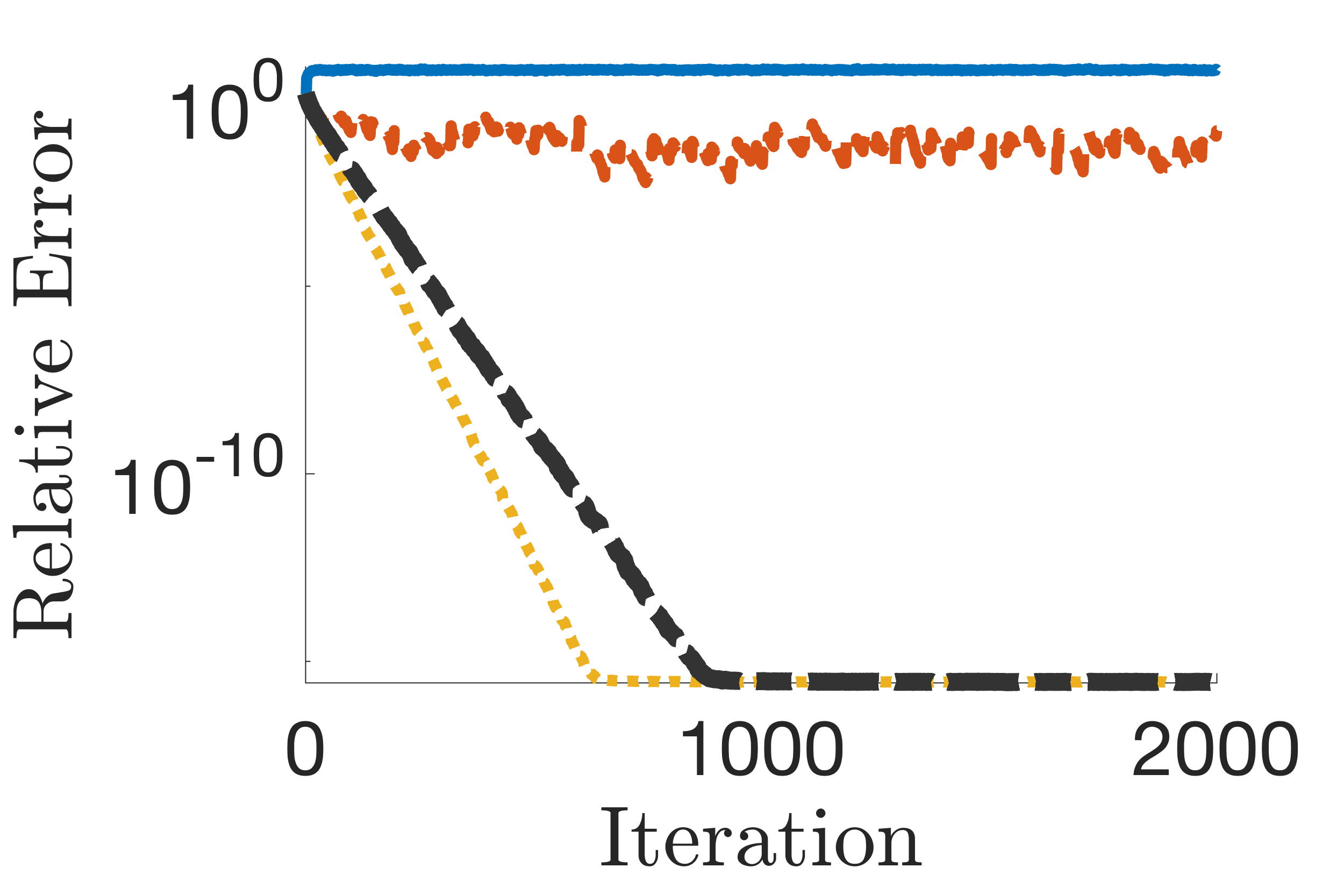}
    \\
    \includegraphics[width=0.3\textwidth]{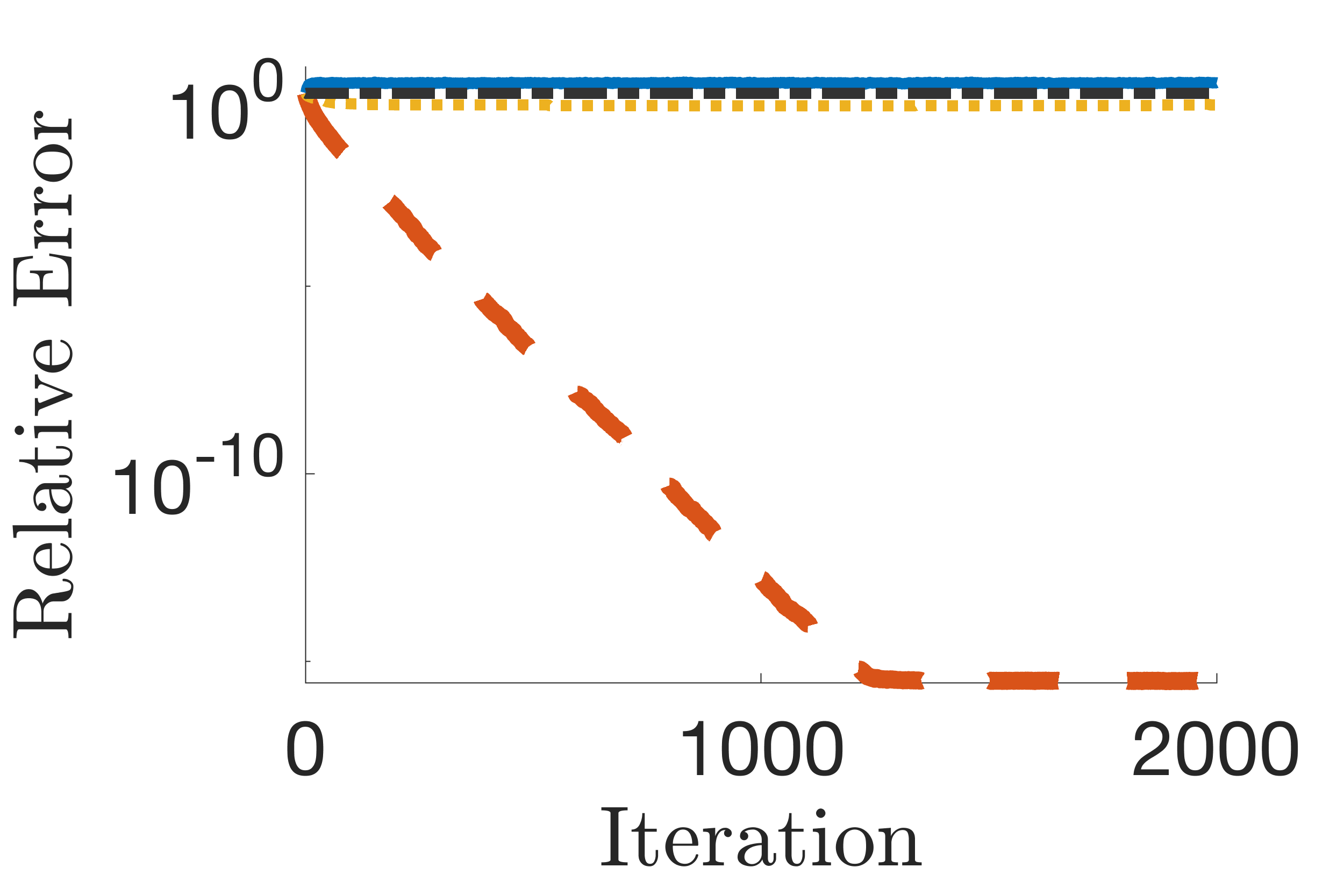}%
    \includegraphics[width=0.3\textwidth]{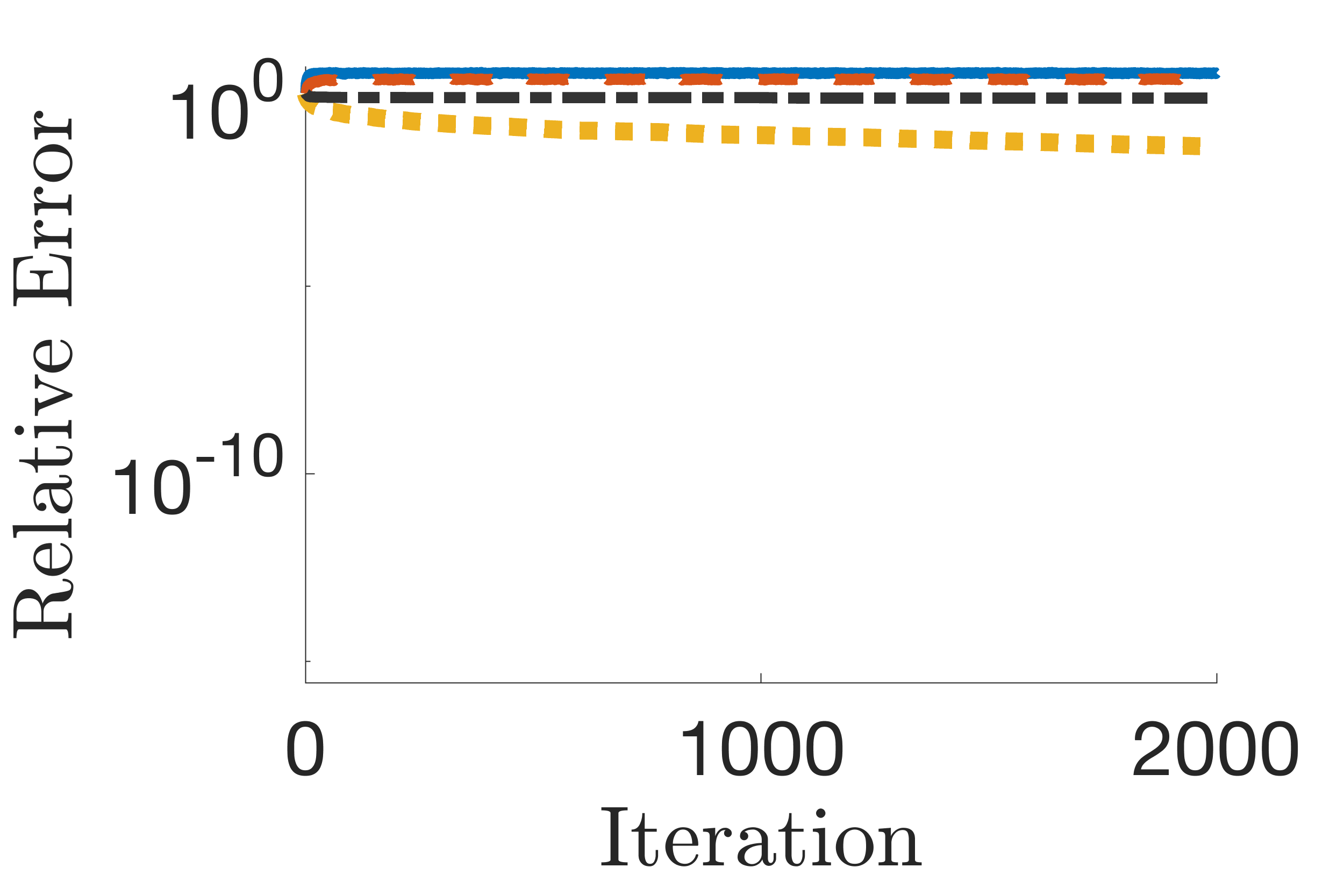}%
    \includegraphics[width=0.3\textwidth]{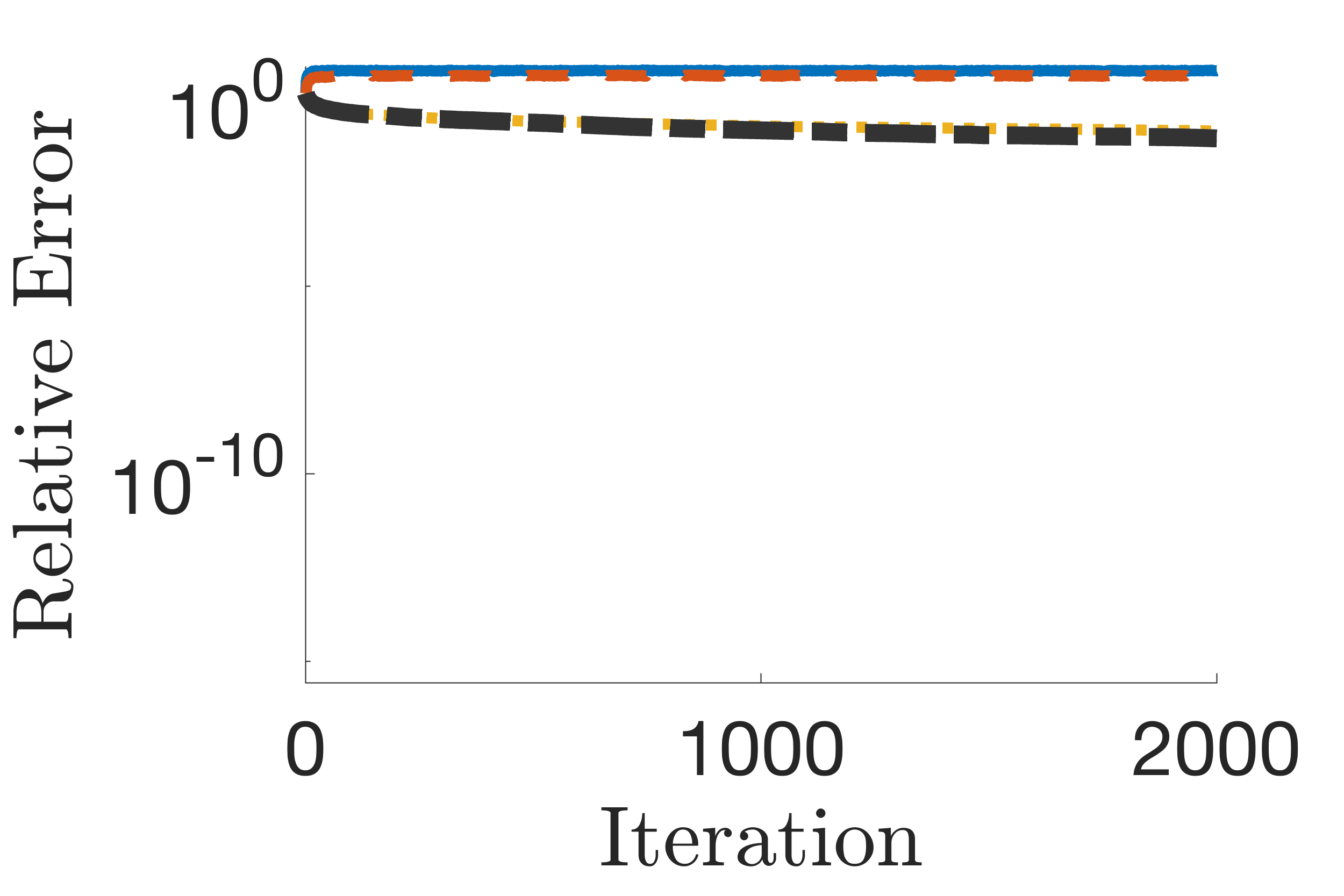}
    \caption{Median relative residual errors of QTRK applied on a system $\tA \tX = \tB$ where $\tA \in \R^{25 \times 5 \times 10}$, $\tB \in \R^{25 \times 4 \times 10}$, and the corruptions are generated from $\mathcal{N}(100,20)$. In the left column plots, $\tbeta = 0.025$, the middle column plots $\tbeta = 0.075$, and the right column plots $\tbeta = 0.1$. In the top row plots, $\tbrow = 0.2$, the middle row plots, $\tbrow = 0.4$, and the bottom row plots $\tbrow = 0.8$. }
     \label{fig:qtrk-large-corr}
\end{figure}

\subsection{QTRK on Synthetic Data}
\label{subsec:QTRK_experiments}

In the first set of experiments, presented in Figures~\ref{fig:qtrk-large-corr} and~\ref{fig:qtrk-small-corr}, we apply QTRK to solve corrupted tensor linear systems as defined in Section \ref{sec:exp design}. 
We consider values for $\tbrow \in \{0.2, 0.4, 0.8\}$ and $\tbeta \in \{0.025, 0.075, 0.01\}$. 
As an input for QTRK, we choose quantile values $q = 1 - \tbeta$ for each value of $\tbeta$ and consider the trivial case $q=1$ where QTRK assumes no corruptions and is thus equivalent to TRK. 
We remind the reader that, for quantile values $q = 1 - \tbeta$ the corresponding graphs are thicker.

In both Figures \ref{fig:qtrk-large-corr} and \ref{fig:qtrk-small-corr}, we observe the following:
\begin{enumerate}[label=(\alph*)]

\item When $\tbrow = 0.2$, QTRK converges, and does so faster, for the underestimates of $\tbeta$ (i.e., $q > 1 - \tbeta$). 
That is, when QTRK is not cautious. 
    
\item For the overestimates of $\tbeta$ (i.e., $q < 1 - \tbeta$), QTRK is more restrictive and thus could potentially be (frequently or completely) omitting row slices that contain important information for finding a solution. This is particularly apparent in the left-most plots (i.e, $\tbeta = 0.025$) with $q = 0.9$.

\item For $\tbrow = 0.8$, the convergence of QTRK is slower and for some values of $q$ it ceases to converge. Note that for $\tbrow = 0.8$ only 5 row slices are guaranteed to be uncorrupted.
This results in an uncorrupted subsystem with the same number of row and column slices (i.e., not highly overdetermined). \\

\end{enumerate}

In Figure \ref{fig:qtrk-large-corr}, we additionally observe the following.
For $\tbrow = 0.2$ and 0.4, as $\tbeta \ge 0.075$ increases, the set of $q$ values for which QTRK converges remains the same.
Further, the convergence of QTRK is not significantly affected by the density of corruptions in the restricted row slices.
Indeed, QTRK is designed to mask entire row slices where corruptions are estimated to be.
Interestingly, for the underestimates of $\tbeta$, it appears that QTRK converges faster when $\tbeta$ increases. 
Roughly speaking, we expect that if QTRK converges for some quantile values $q_1$ and $q_2$ where $q_1 > q_2 > 1 - \tbeta$, then it would converge faster for the larger quantile value $q_1$ (i.e., when QTRK is less cautious).  Finally, we point out that QTRK consistently outperforms TRK, illustrating the need for methods tailored to avoid projecting with corrupted data.

\begin{figure}[h!]
    \centering  
    \includegraphics[width=0.3\textwidth]{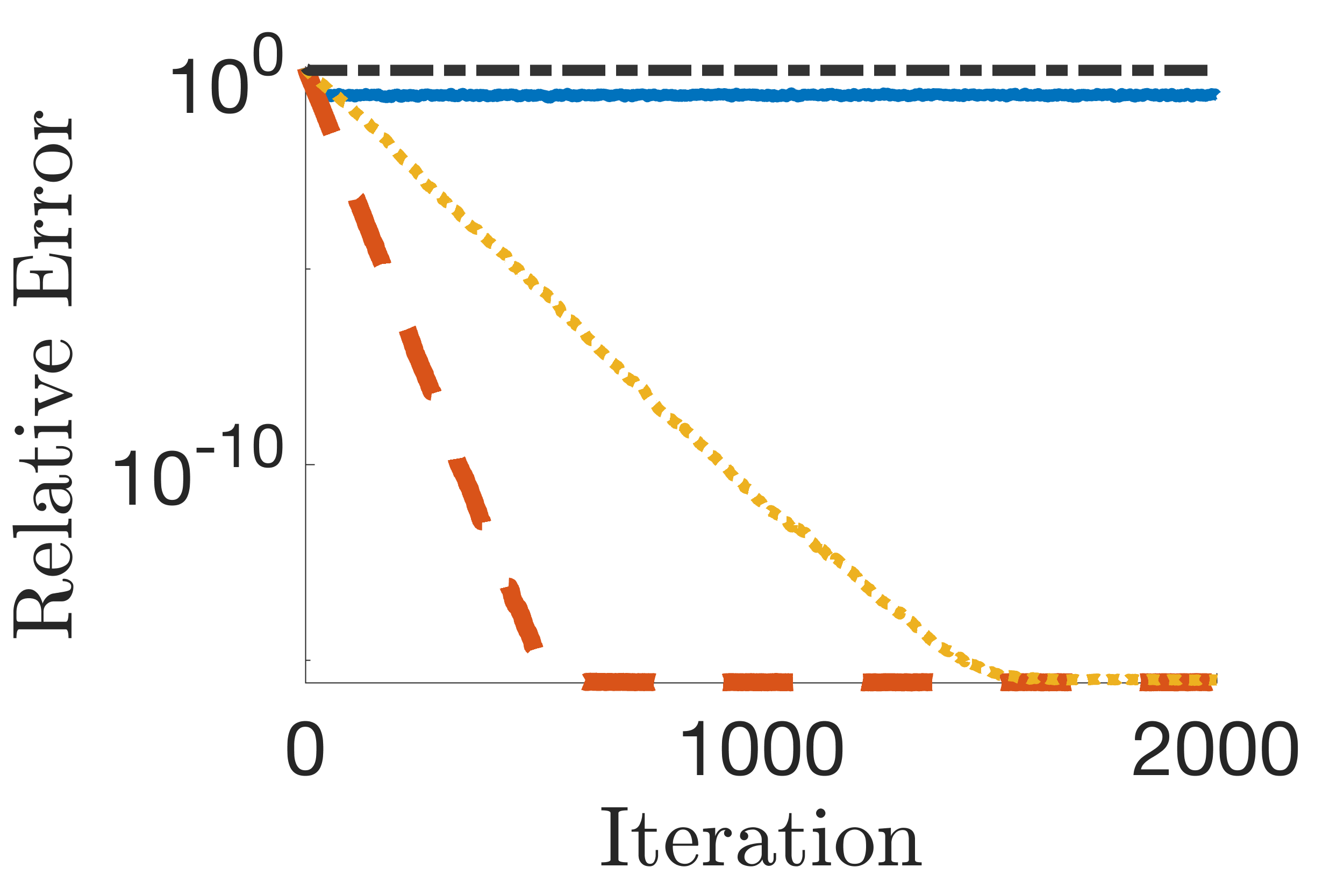}%
    \includegraphics[width=0.3\textwidth]{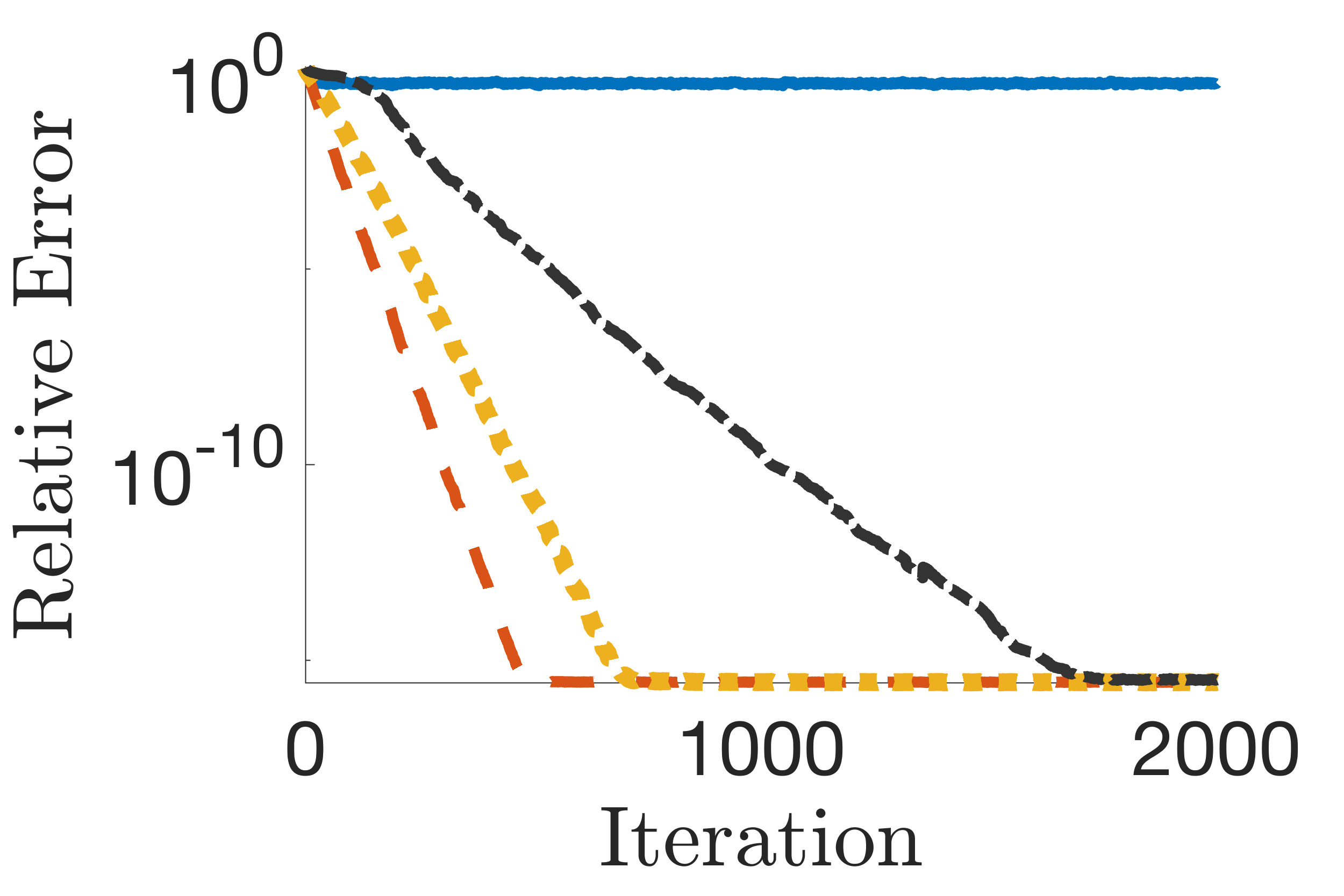}%
    \includegraphics[width=0.3\textwidth]{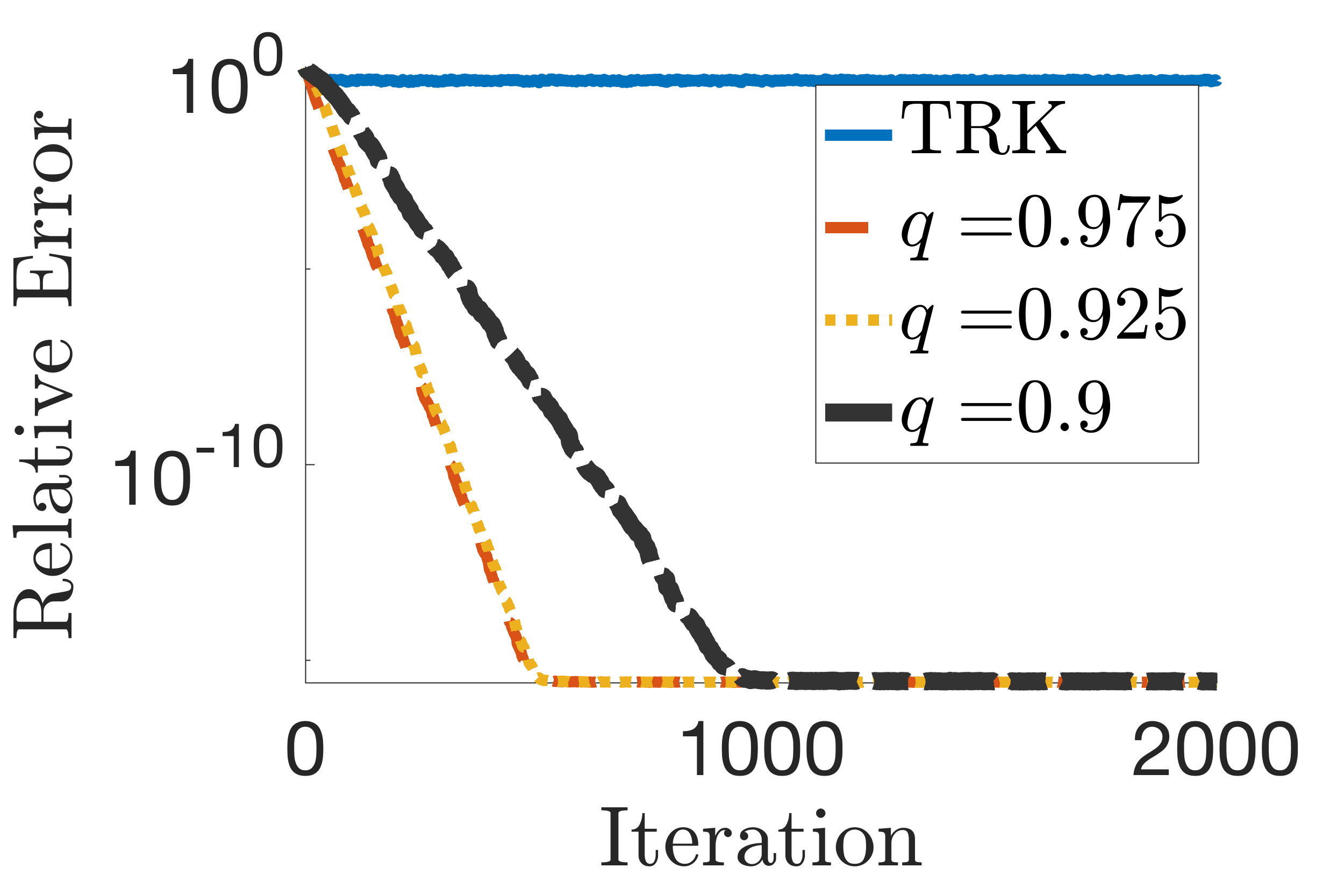}
    \\
    \includegraphics[width=0.3\textwidth]{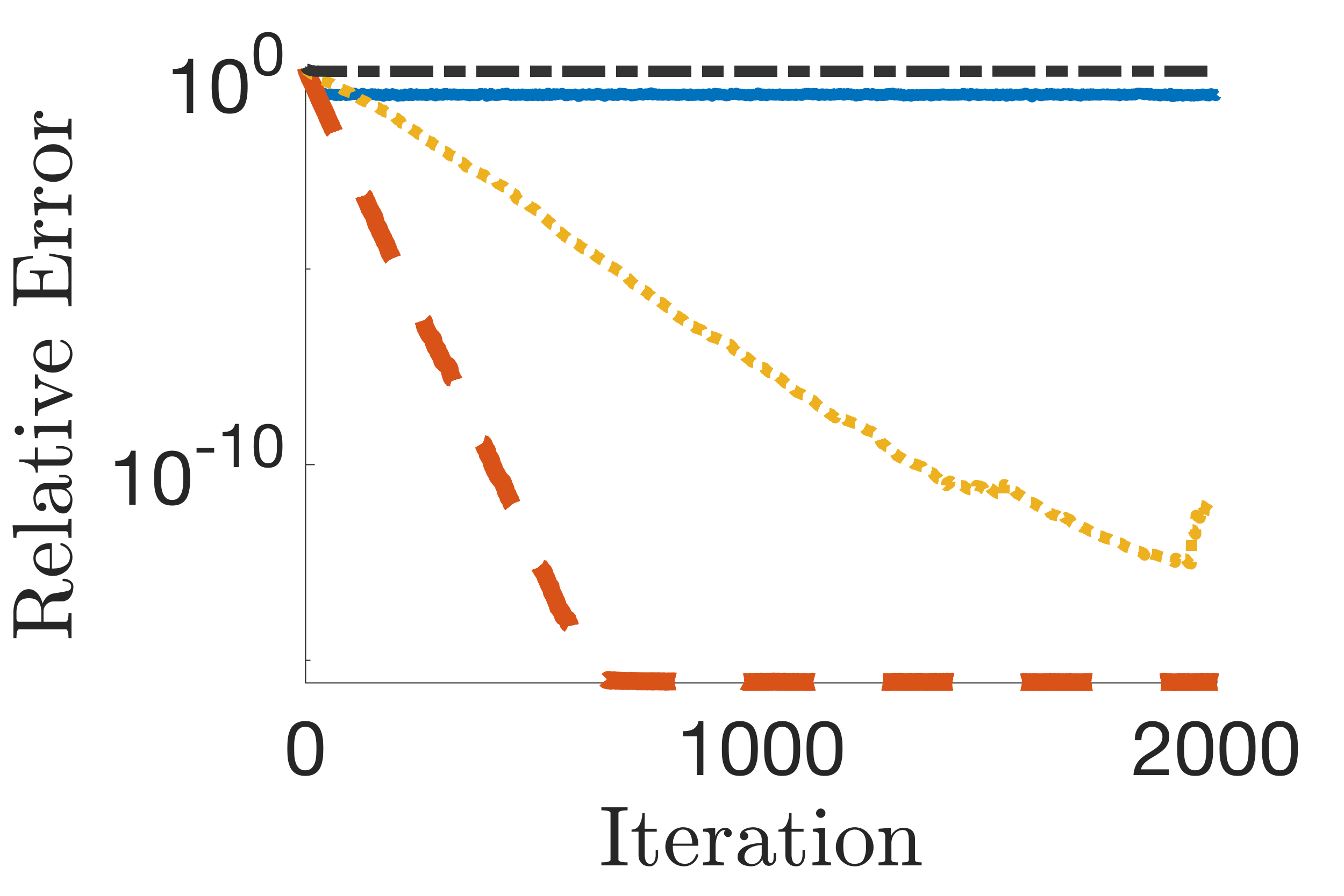}%
    \includegraphics[width=0.3\textwidth]{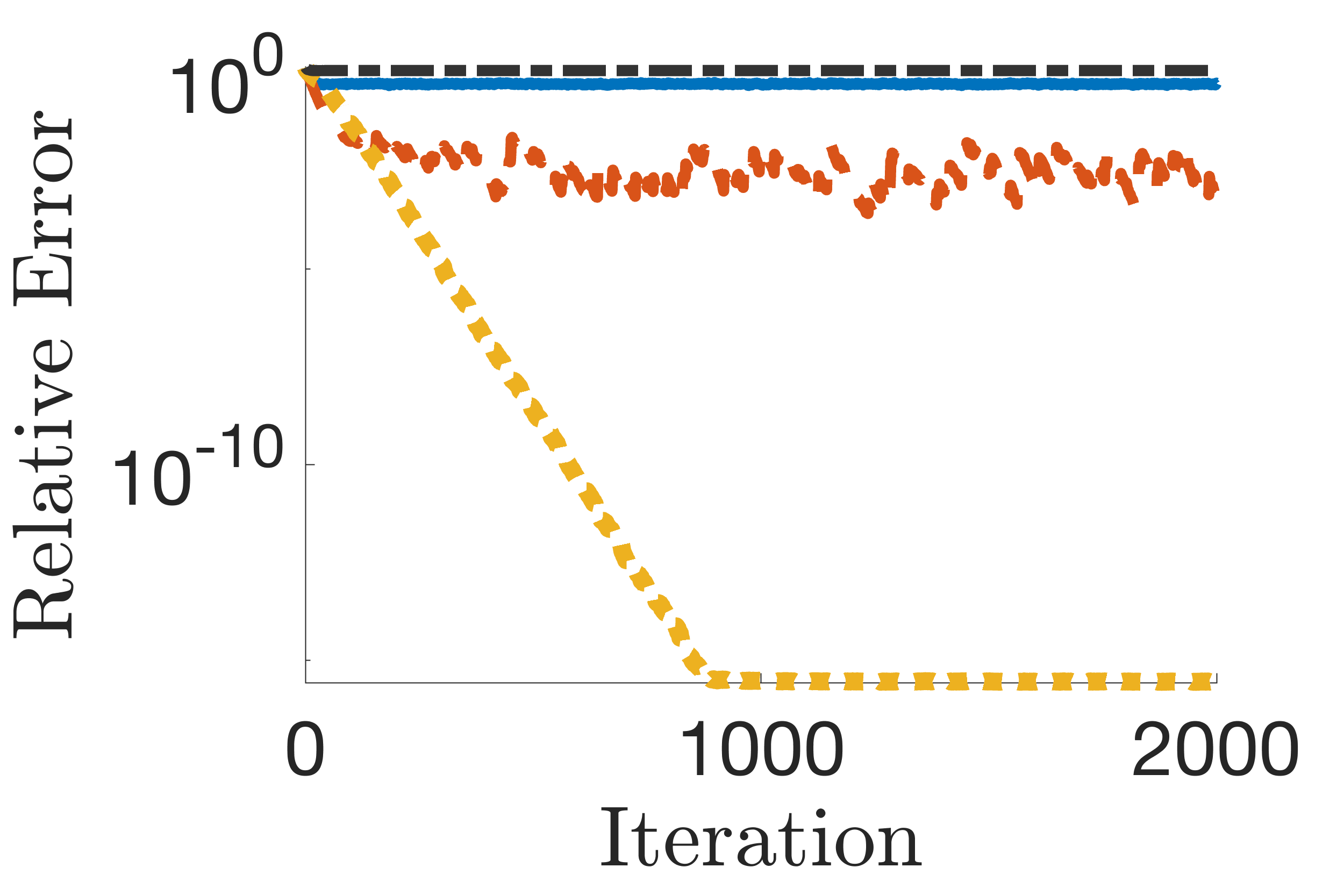}%
    \includegraphics[width=0.3\textwidth]{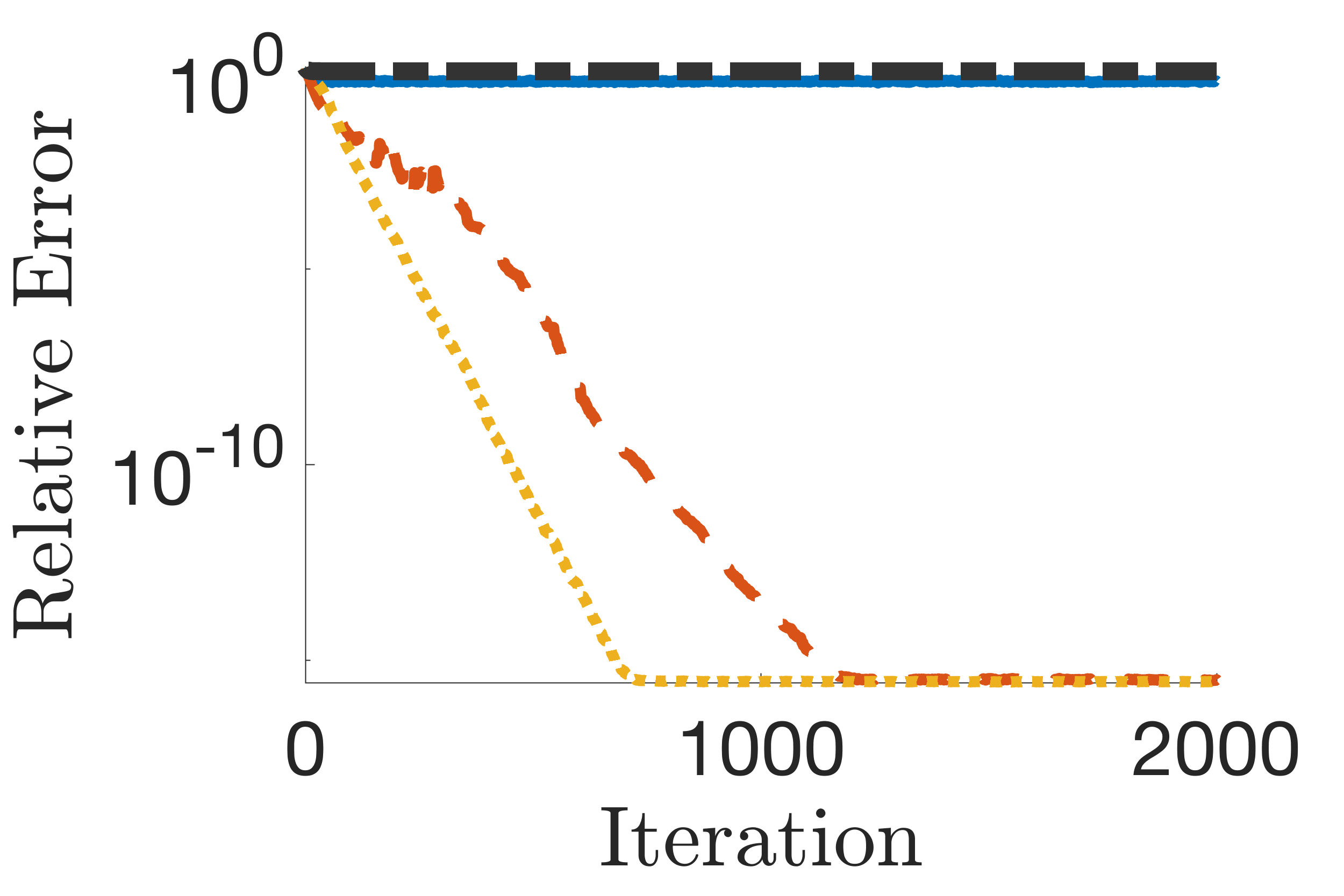}
    \\
    \includegraphics[width=0.3\textwidth]{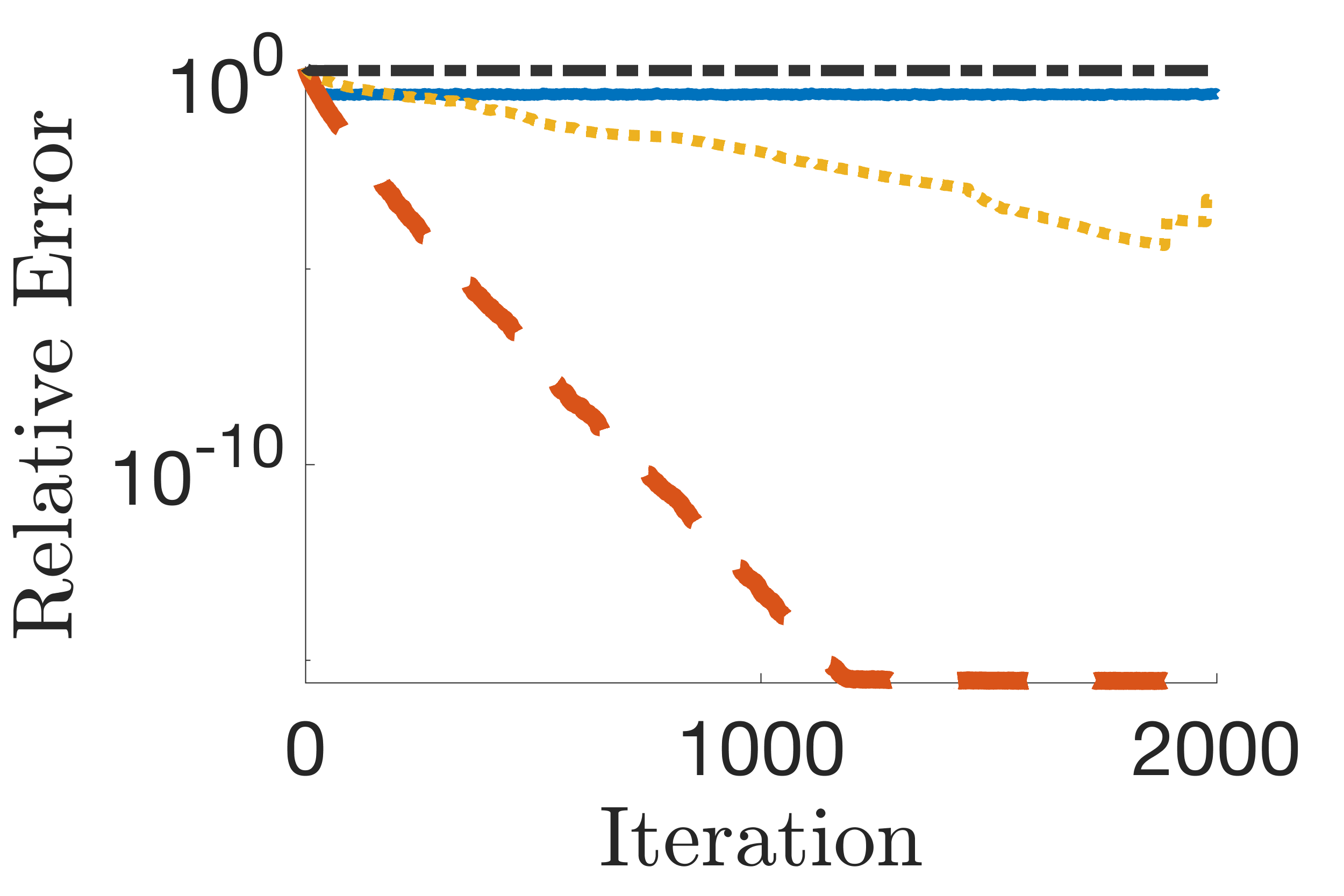}%
    \includegraphics[width=0.3\textwidth]{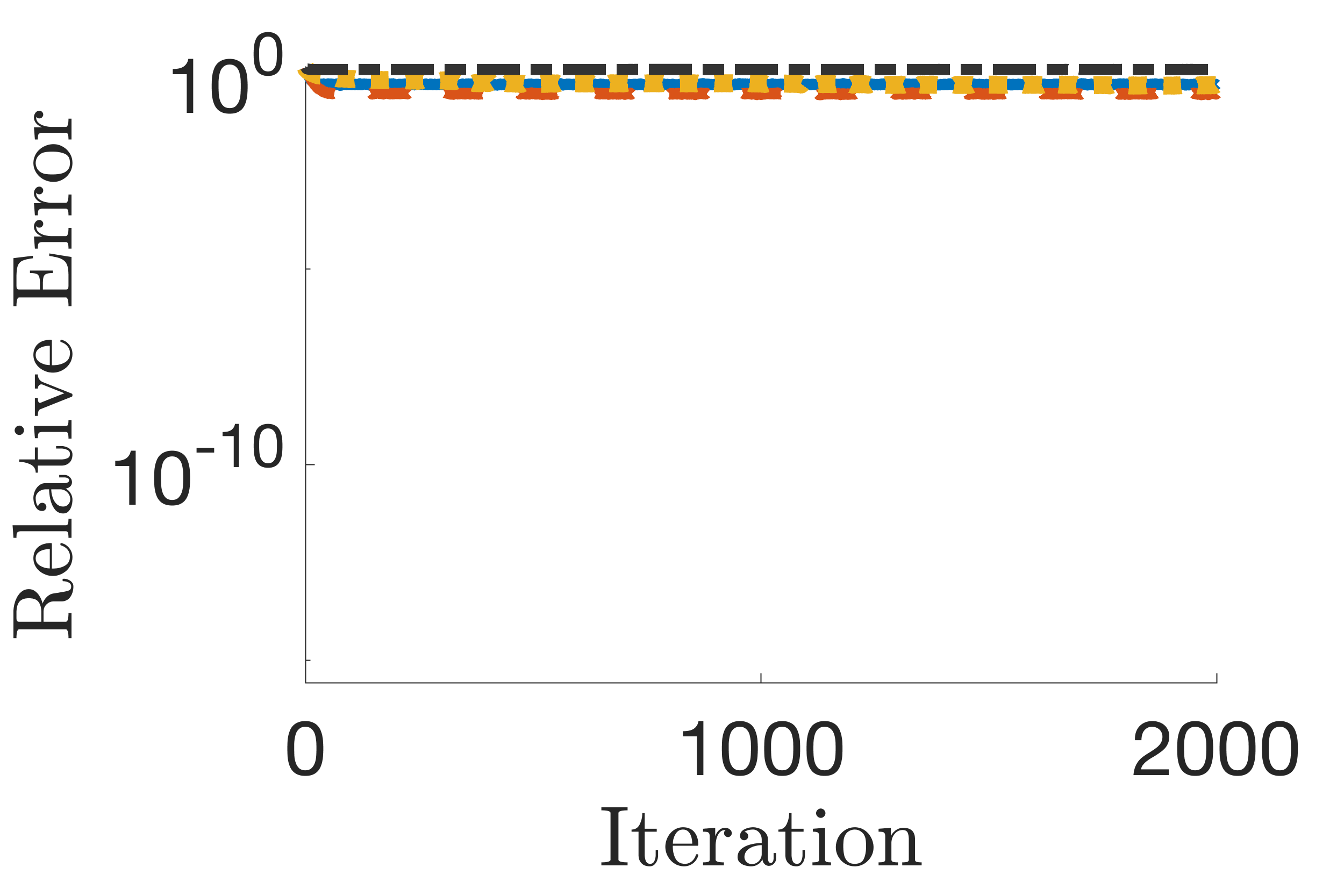}%
    \includegraphics[width=0.3\textwidth]{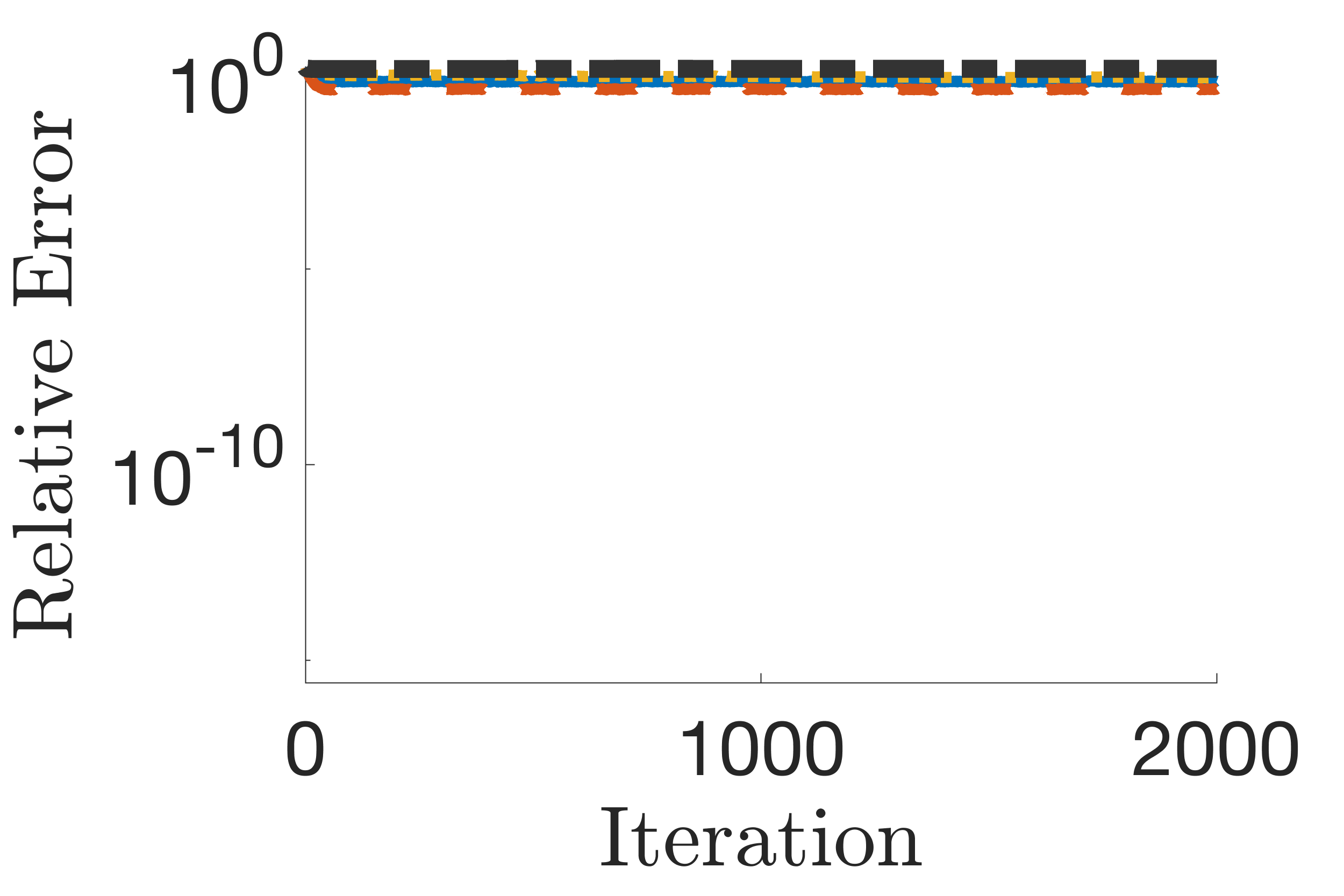}
    \caption{Median relative residual errors of QTRK applied on a system $\tA \tX = \tB$ where $\tA \in \R^{25 \times 5 \times 10}$, $\tB \in \R^{25 \times 4 \times 10}$, and the corruptions are generated from $\mathcal{N}(10,5)$. In the left column plots, $\tbeta = 0.025$, the middle column plots $\tbeta = 0.075$, and the right column plots $\tbeta = 0.1$. In the top row plots, $\tbrow = 0.2$, the middle row plots, $\tbrow = 0.4$, and the bottom row plots $\tbrow = 0.8$. }
     \label{fig:qtrk-small-corr}
\end{figure}

In addition to observation (b) above, in Figure \ref{fig:qtrk-small-corr}, we observe in all cases except for when $\tbeta = 0.1$ and $\tbrow = 0.2$, QTRK does not converge for $q=0.9$.
Interestingly, for a constant value of $q$, when $\tbeta$ is larger, QTRK has better odds of correctly identifying corruptions with larger magnitude and thus may converge when it did not for smaller values of $\tbeta$.
Further, TRK results in slightly lower median relative residual errors compared to when $q = 0.9$. 
We note that by design there is a tradeoff between correctly detecting a corruption and allowing a small but non-devastating corruption to be included in the update.  
Figure \ref{fig:qtrk-small-corr} shows that when the magnitudes of the corruptions create ambiguities in the algorithm, QTRK may ``detect" a corruption in each row slice of $\tB$ and thus ceases to update $\tX^{(k)}$. 
This could happen at the very first iteration.
Re-initialization or warm starts may be a possible way to reduce this risk, and is an interesting direction for future work.

\begin{figure}[h!]
    \centering  
    \includegraphics[width=0.3\textwidth]{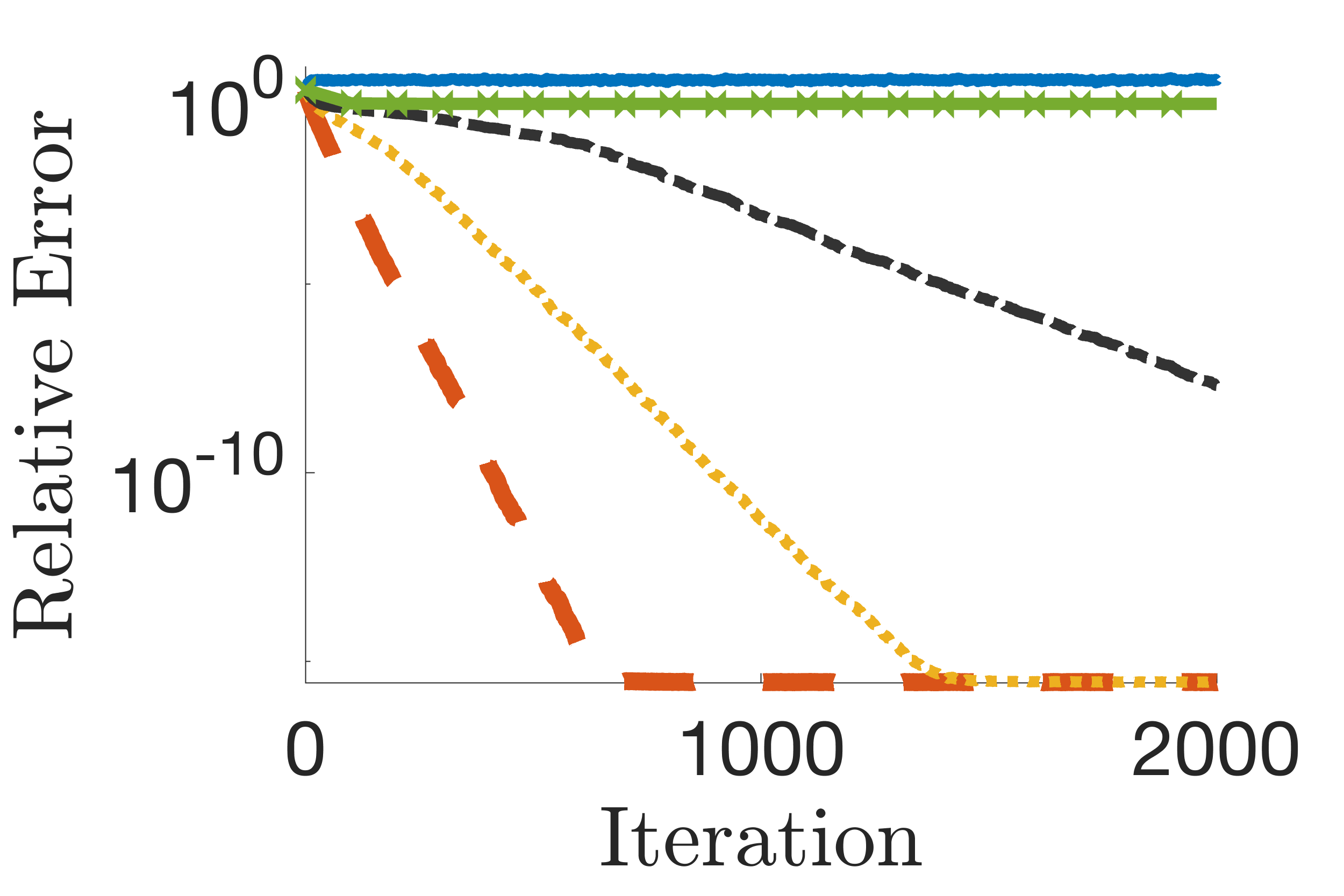}%
    \includegraphics[width=0.3\textwidth]{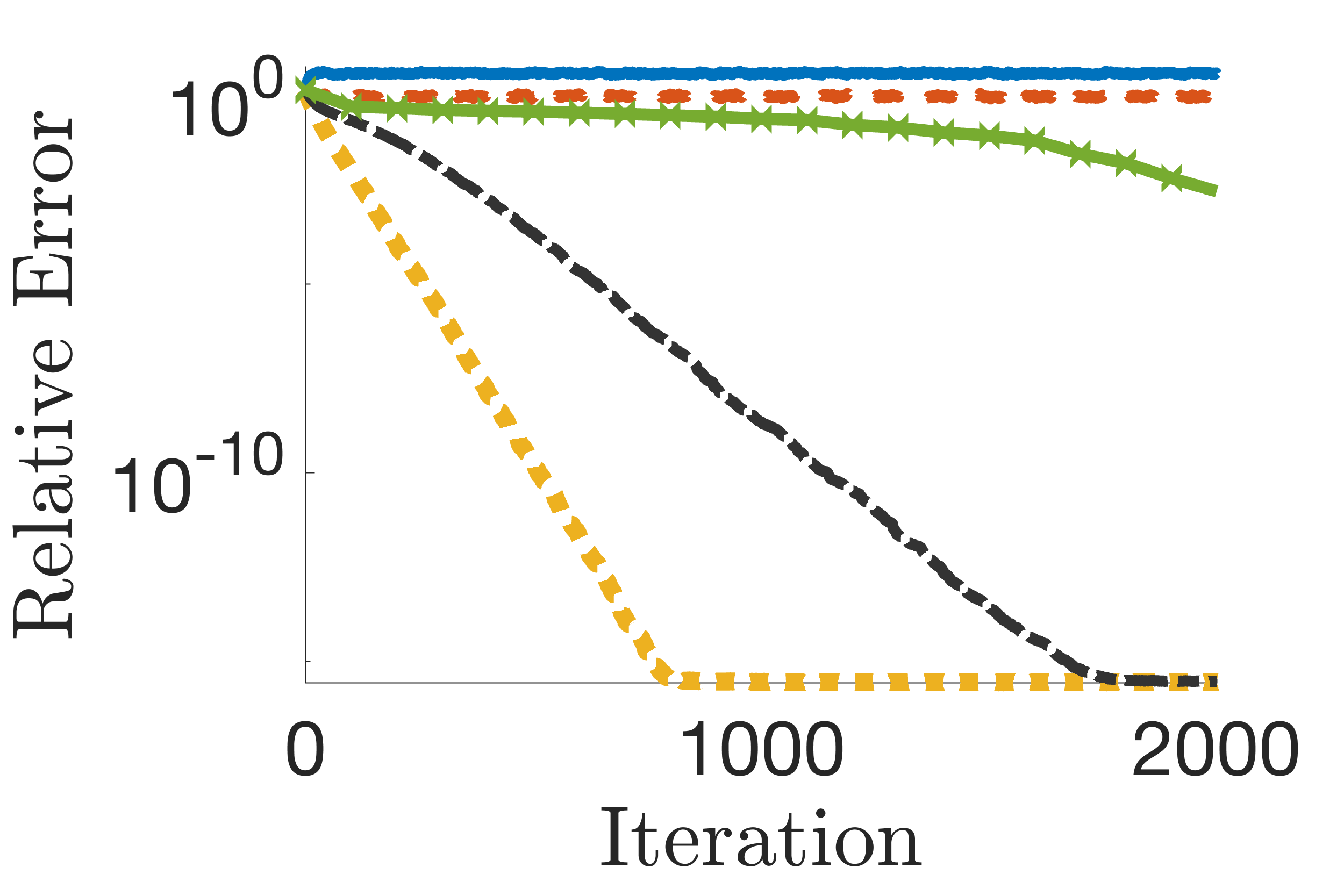}%
    \includegraphics[width=0.3\textwidth]{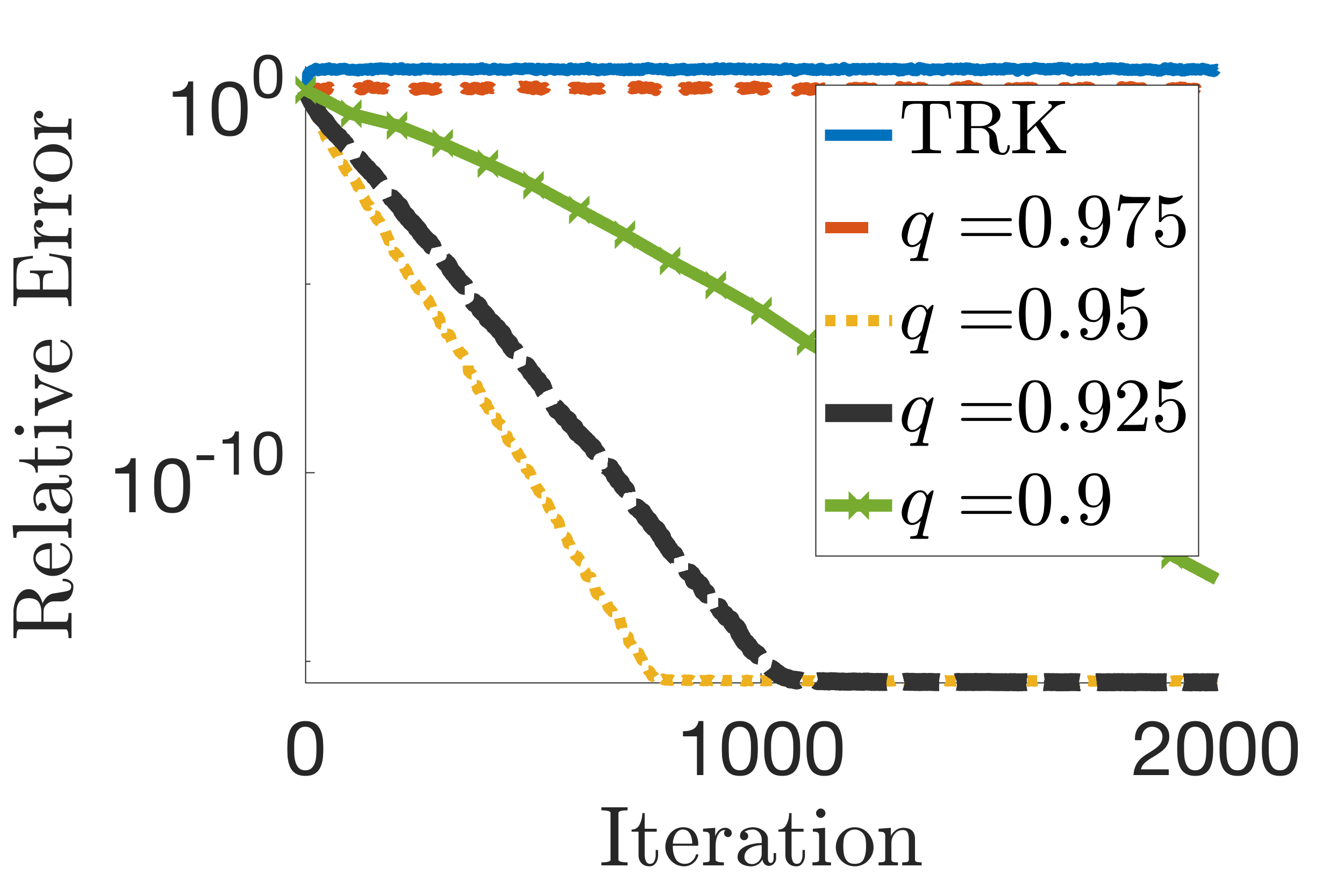}
    \\
    \includegraphics[width=0.3\textwidth]{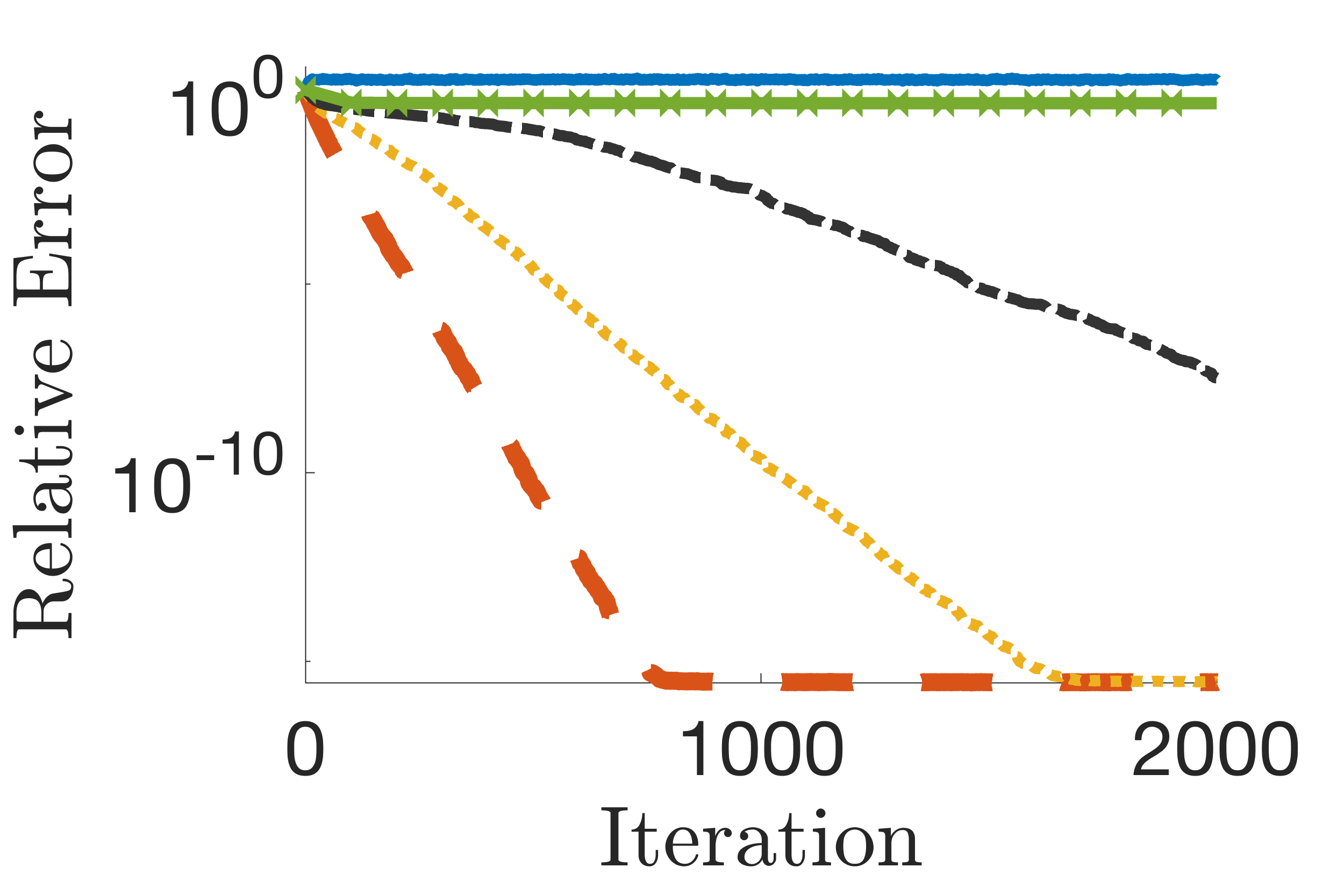}%
    \includegraphics[width=0.3\textwidth]{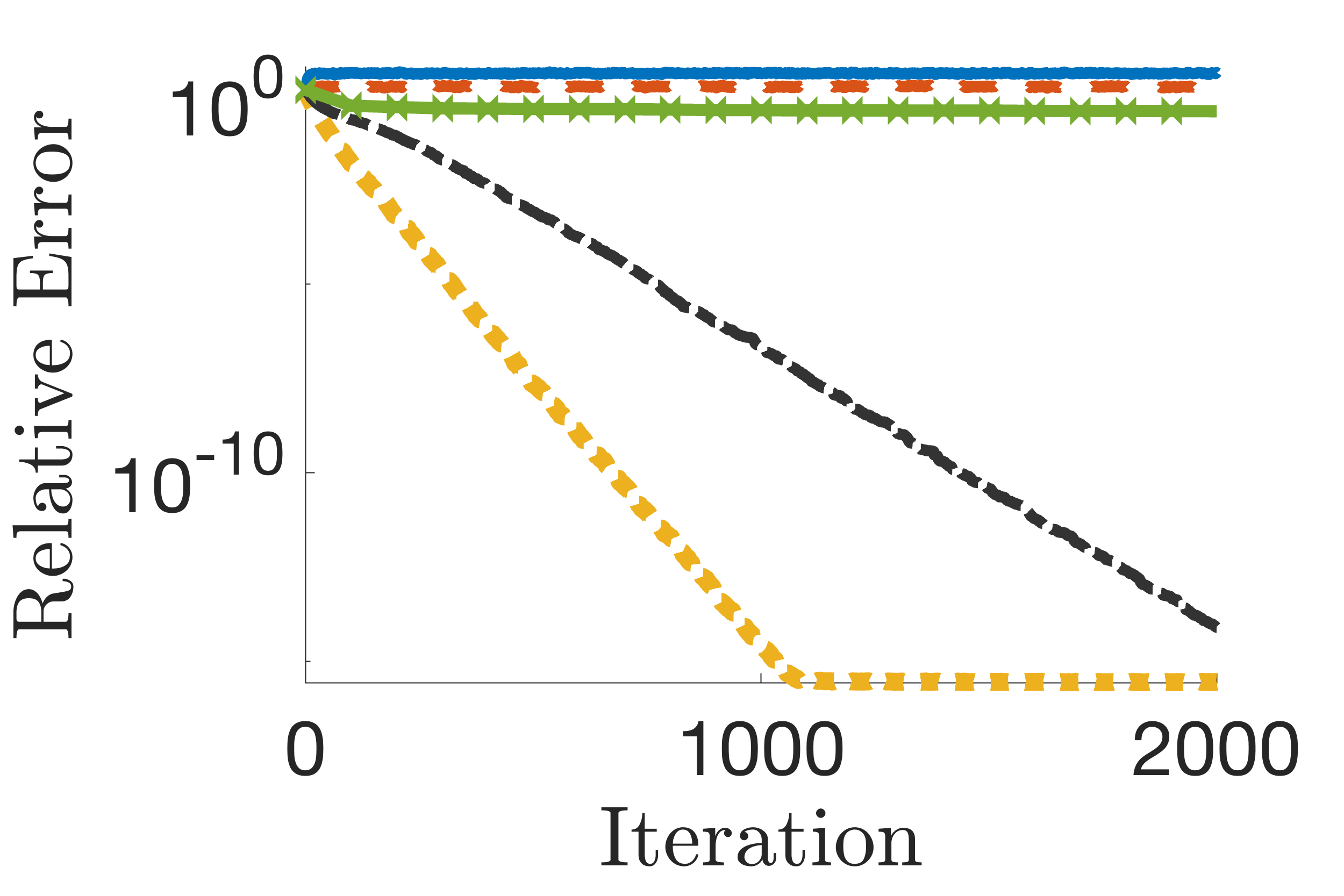}%
    \includegraphics[width=0.3\textwidth]{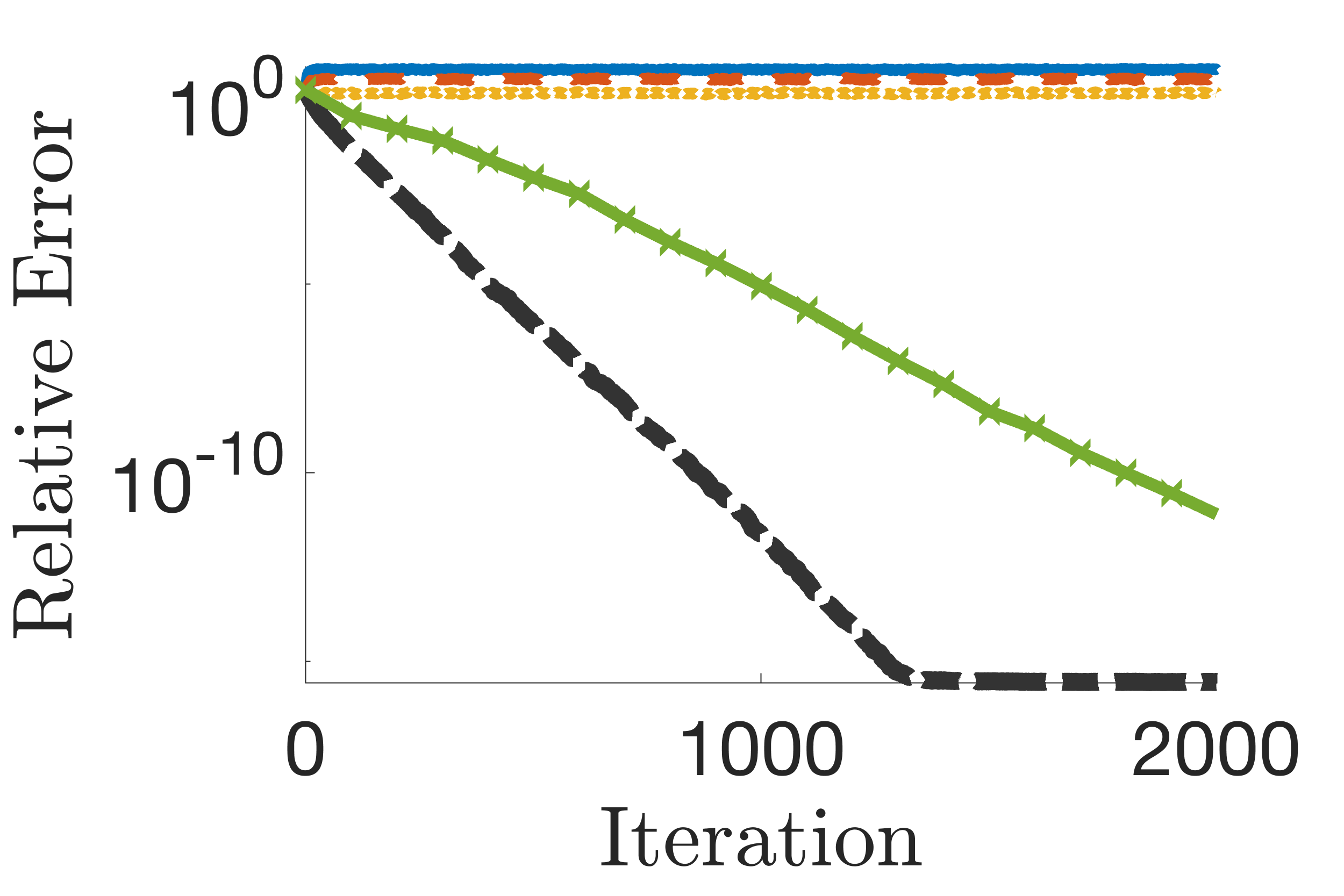}
    \\
    \includegraphics[width=0.3\textwidth]{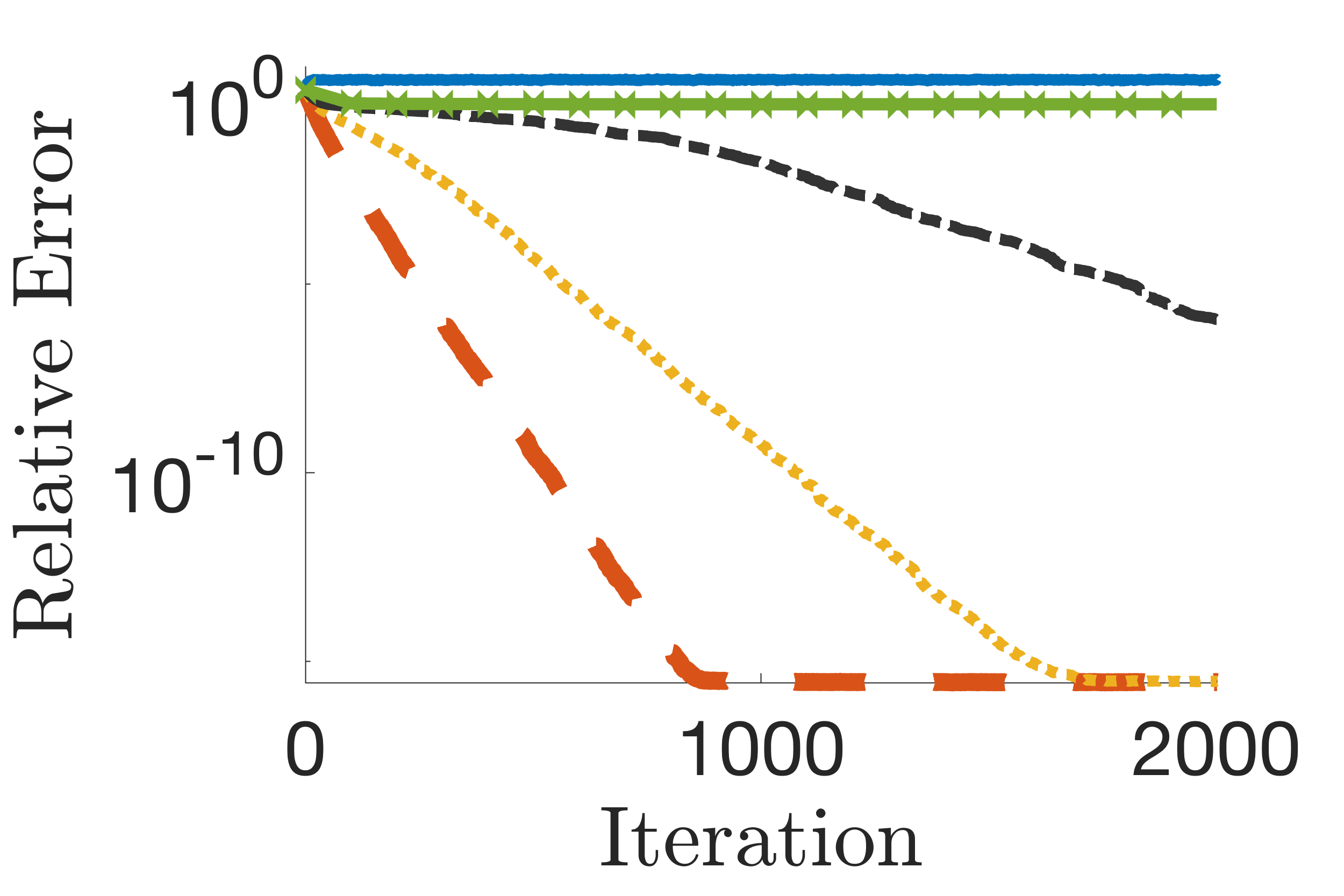}%
    \includegraphics[width=0.3\textwidth]{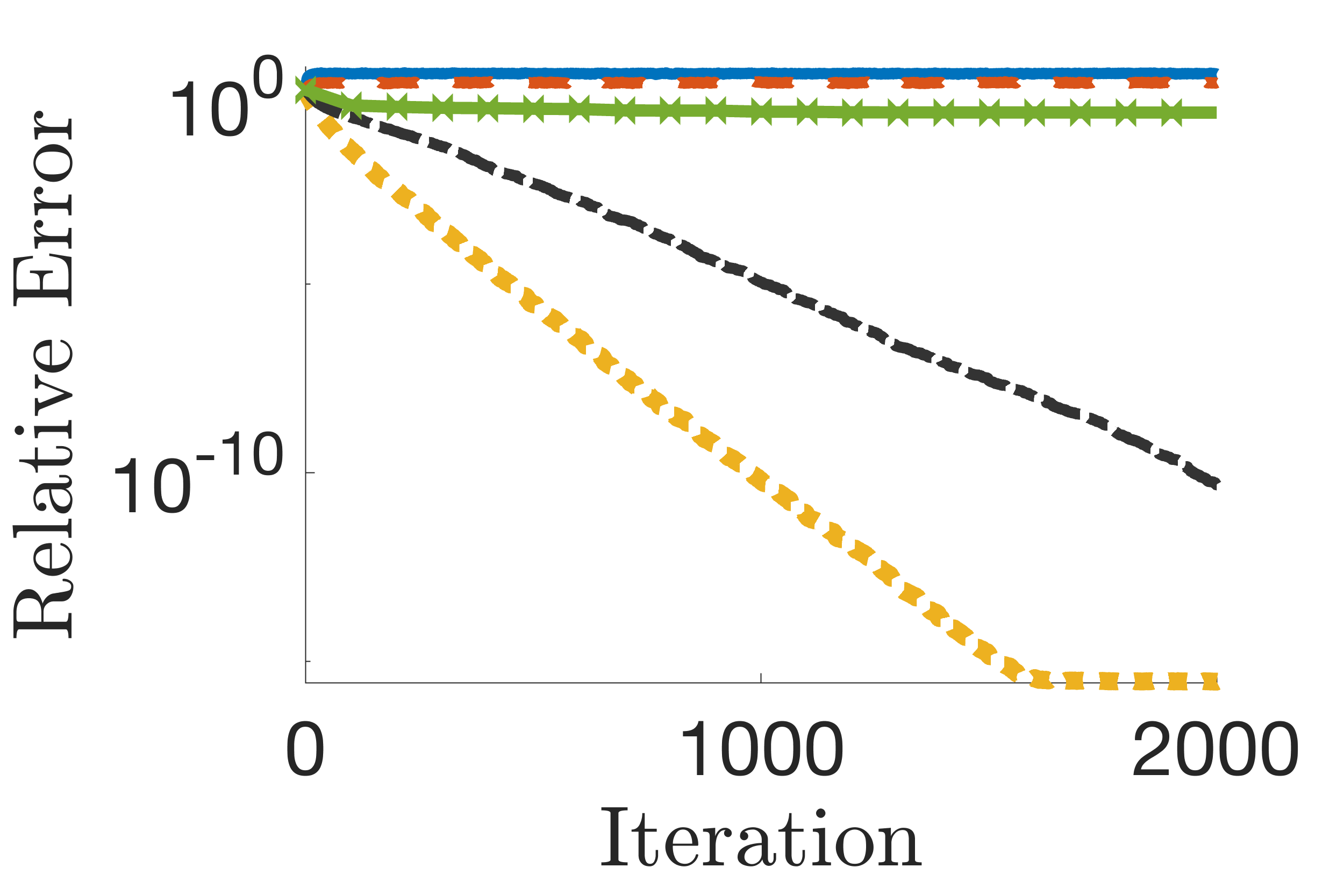}%
    \includegraphics[width=0.3\textwidth]{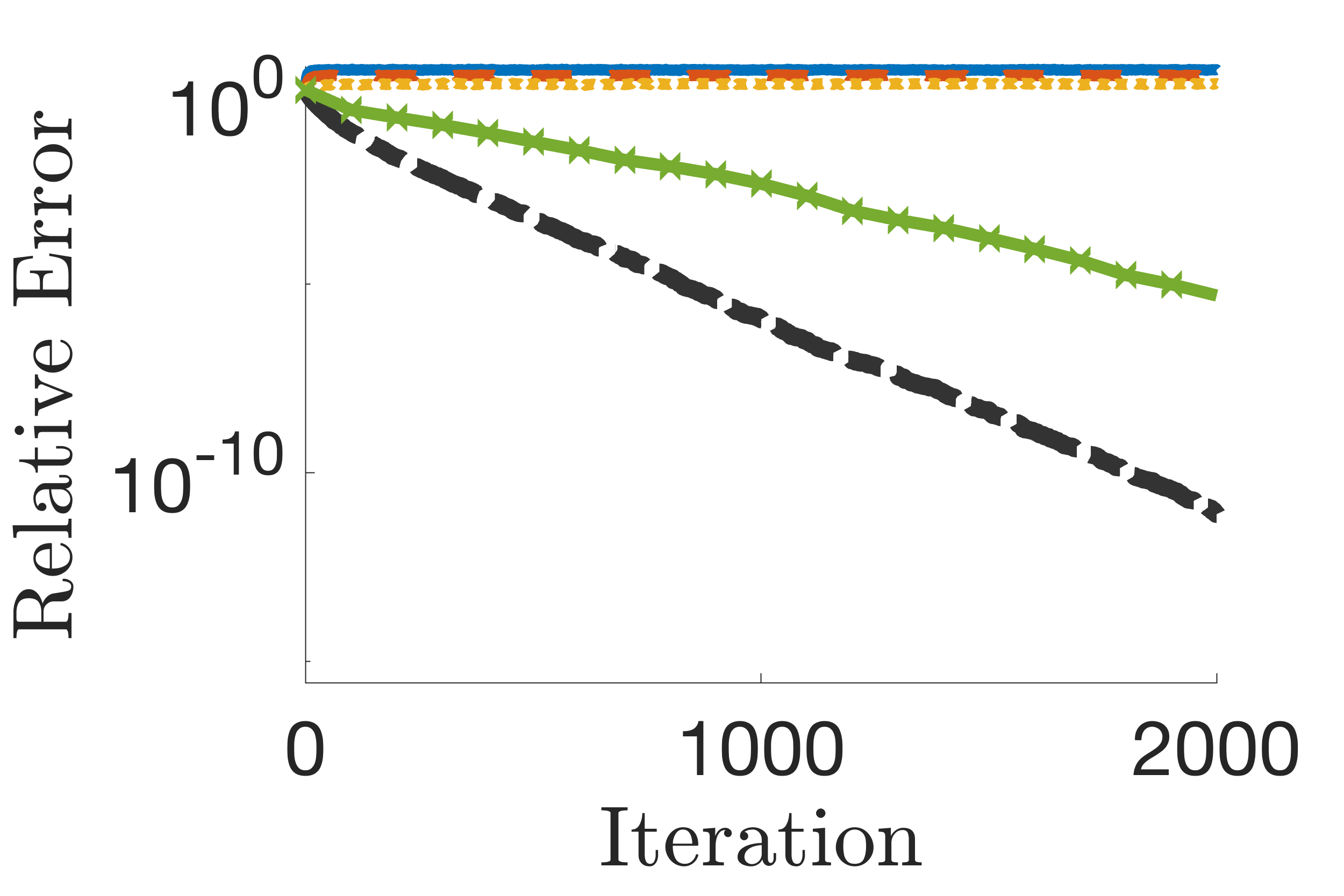}
    \\
    \includegraphics[width=0.3\textwidth]{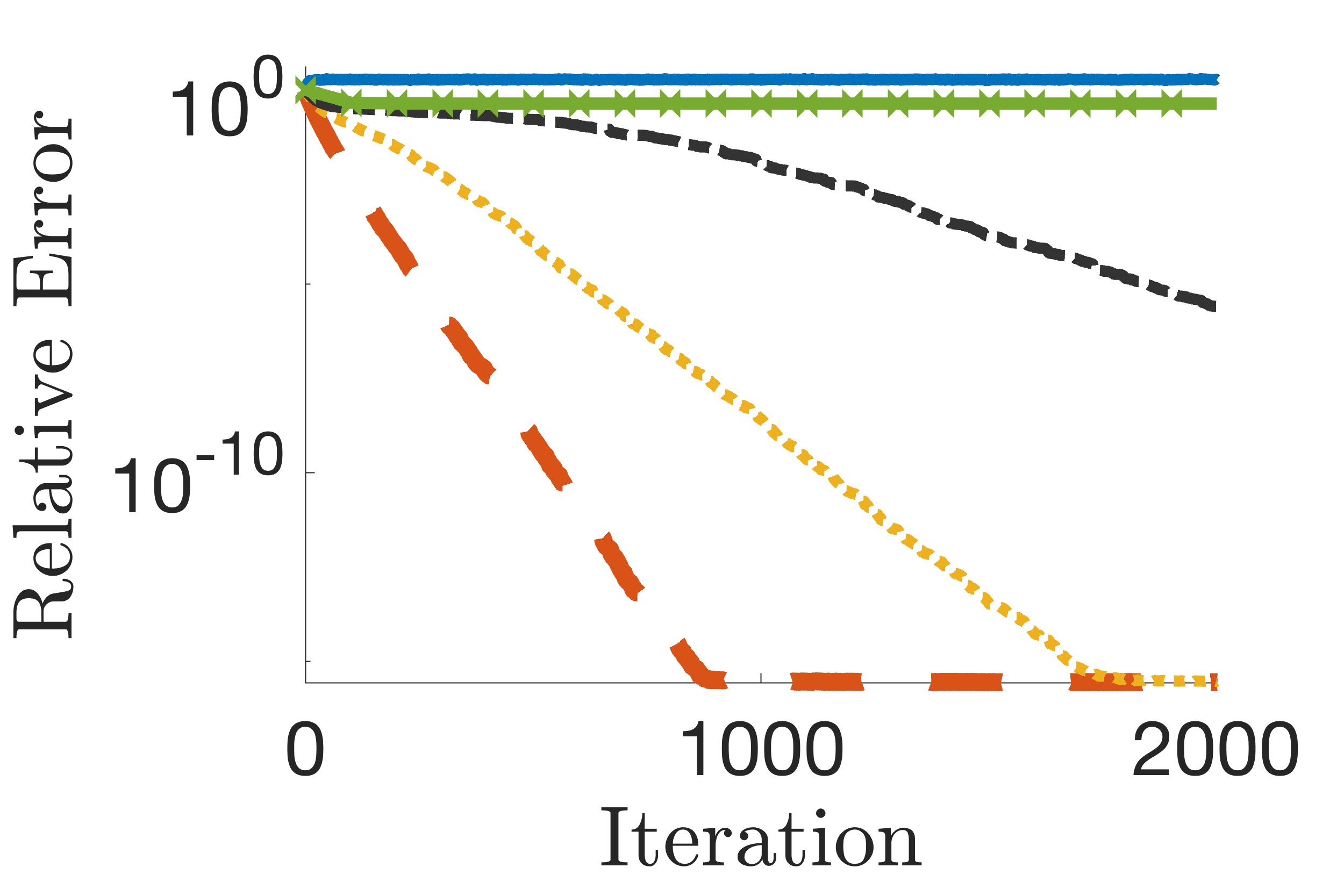}%
    \includegraphics[width=0.3\textwidth]{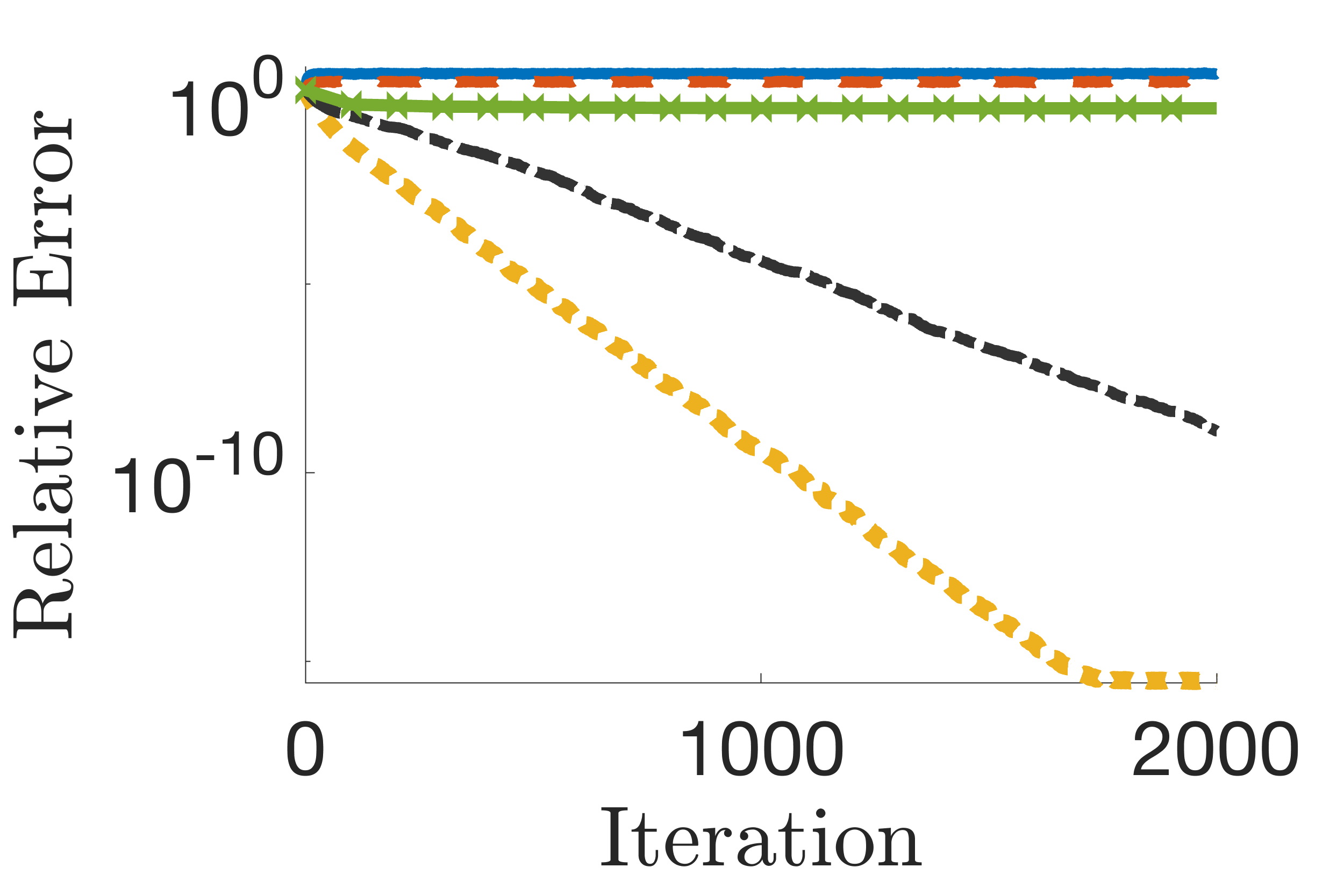}%
    \includegraphics[width=0.3\textwidth]{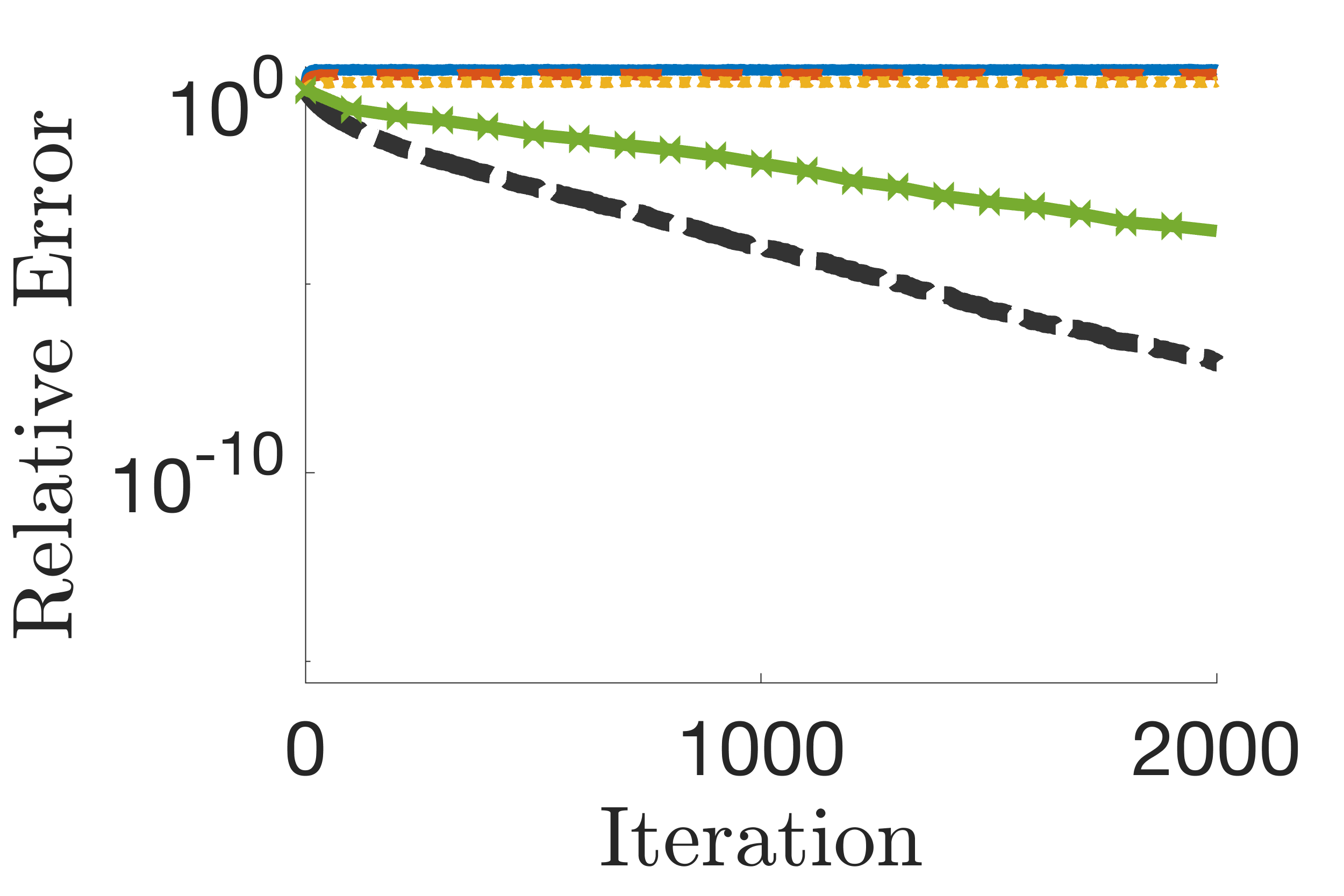}
    \caption{Median relative residual errors of mQTRK applied on a system $\tA \tX = \tB$ where $\tA \in \R^{25 \times 5 \times 10}$, $\tB \in \R^{25 \times 4 \times 10}$, and the corruptions are generated from $\mathcal{N}(100,20)$. In the left column plots, $\tbeta = 0.025$, the middle column plots $\tbeta = 0.05$, and the right column plots $\tbeta = 0.075$. In the top row plots, $\tbrow = 0.2$, the second row plots, $\tbrow = 0.4$, the third row plots $\tbrow = 0.8$, and the bottom row plots $\tbrow = 1$.}
     \label{fig:mqtrk-large-corr}
\end{figure}

\subsection{mQTRK on Synthetic Data}
\label{subsec:mQTRK_experiments}

In the second set of experiments, presented in Figures~\ref{fig:mqtrk-large-corr} and~\ref{fig:mqtrk-small-corr}, we apply mQTRK to solve corrupted tensor linear systems as defined in Section \ref{sec:exp design}.
We consider values for $\tbrow \in \{0.2, 0.4, 0.8, 1\}$ and $\tbeta \in \{ 0.025, 0.05, 0.075\}$. 
As an input for mQTRK, we choose quantile values $q = 1 - \tbeta$ for each value of $\tbeta$, $q=1$, and $q = 0.90$ (an underestimate of $1 - \tbeta$ in all cases). 
We remind the reader that, for quantile values $q = 1 - \tbeta$ the corresponding graphs are thicker.

An appropriate comparison with the performance of QTRK from the previous section would be examining the right-most plots of Figures~\ref{fig:mqtrk-large-corr} and~\ref{fig:mqtrk-small-corr} with the center plots of Figures~\ref{fig:qtrk-large-corr} and~\ref{fig:qtrk-small-corr}, respectively.
In all of these figures, the left-most plots have the same settings.
We will discuss this choice below and present additional comparisons experiments in Section~\ref{subsec:QTRKvsmQTRK}.

\begin{figure}[h!]
    \centering  
    \includegraphics[width=0.3\textwidth]{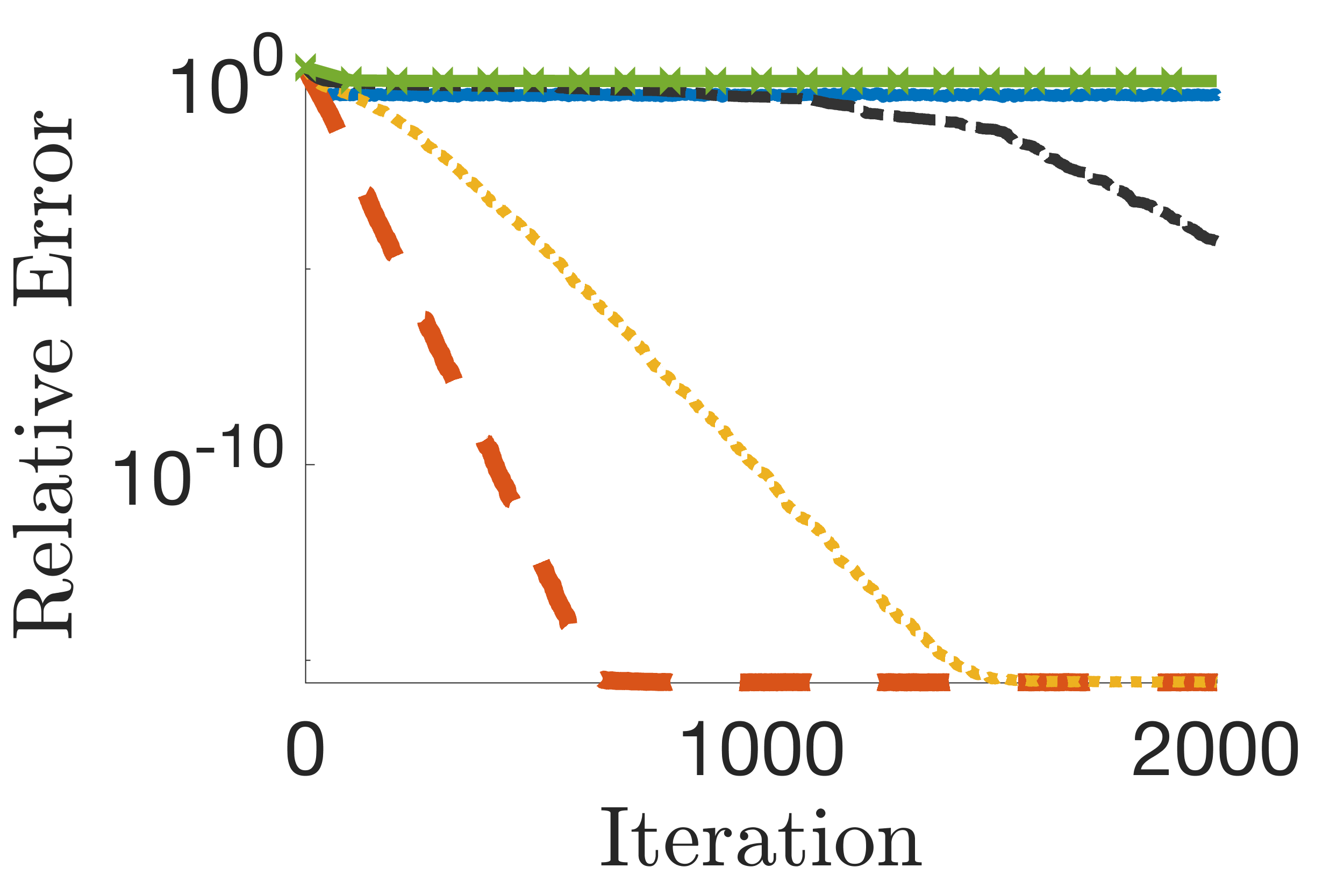}%
    \includegraphics[width=0.3\textwidth]{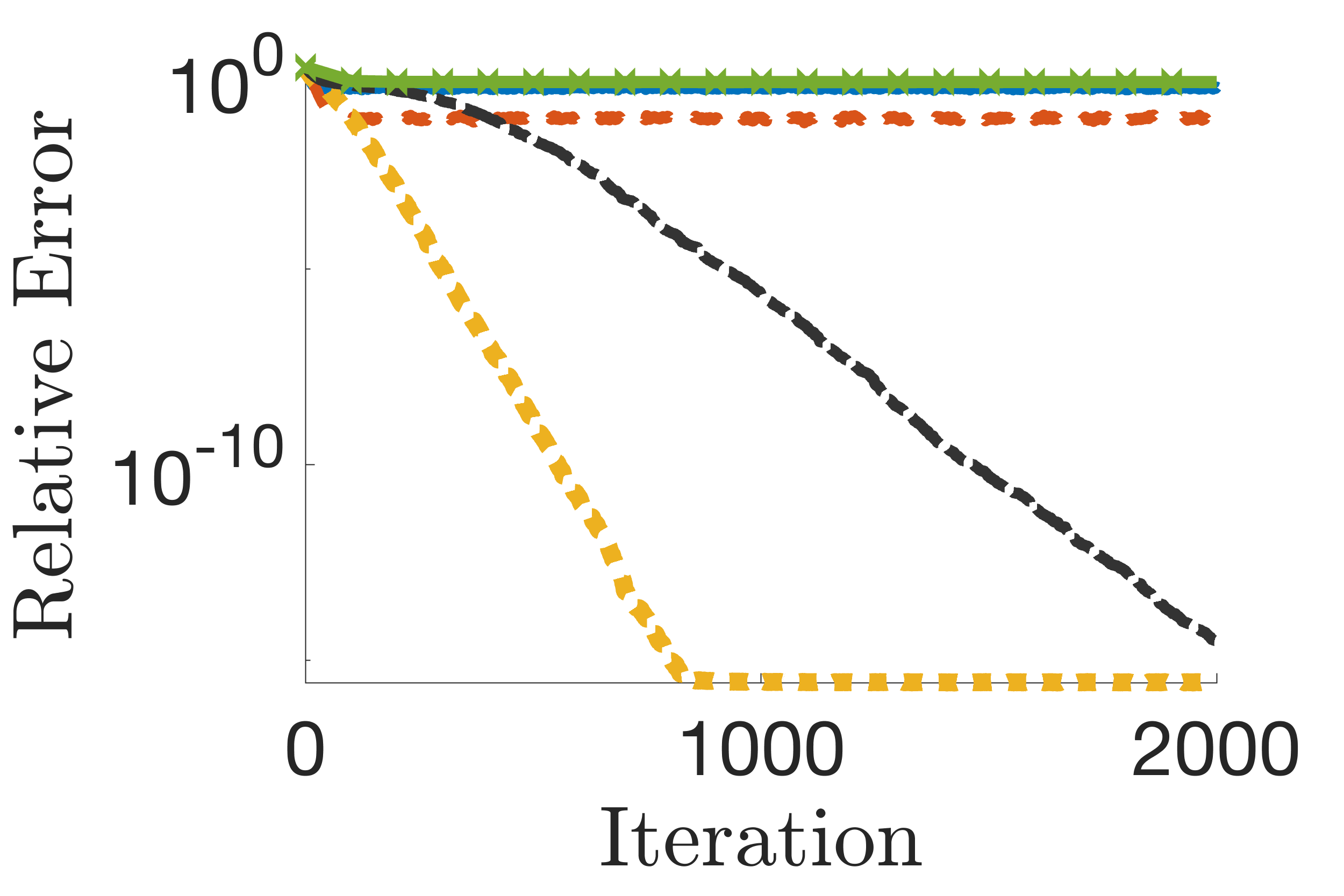}%
    \includegraphics[width=0.3\textwidth]{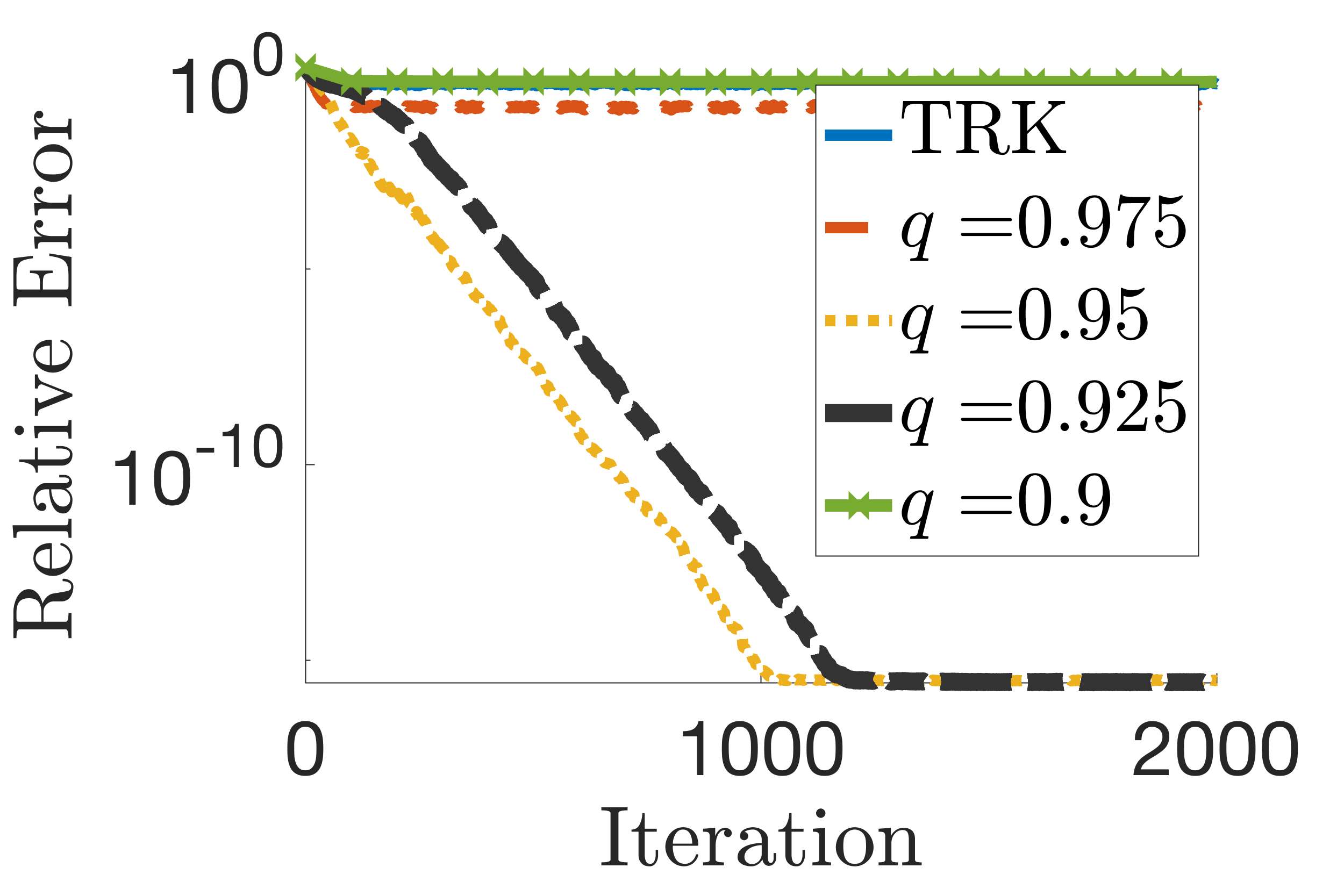}
    \\
    \includegraphics[width=0.3\textwidth]{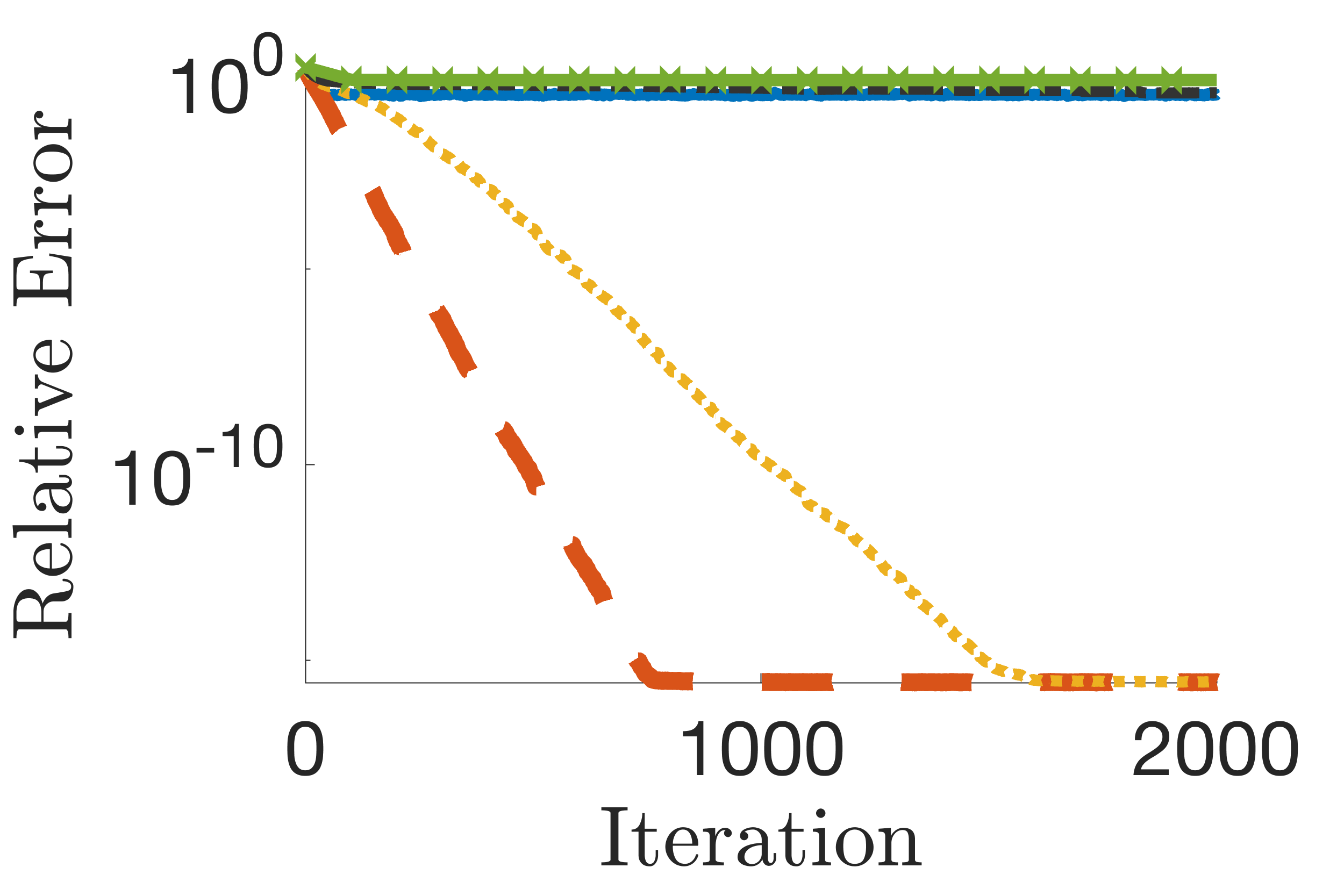}%
    \includegraphics[width=0.3\textwidth]{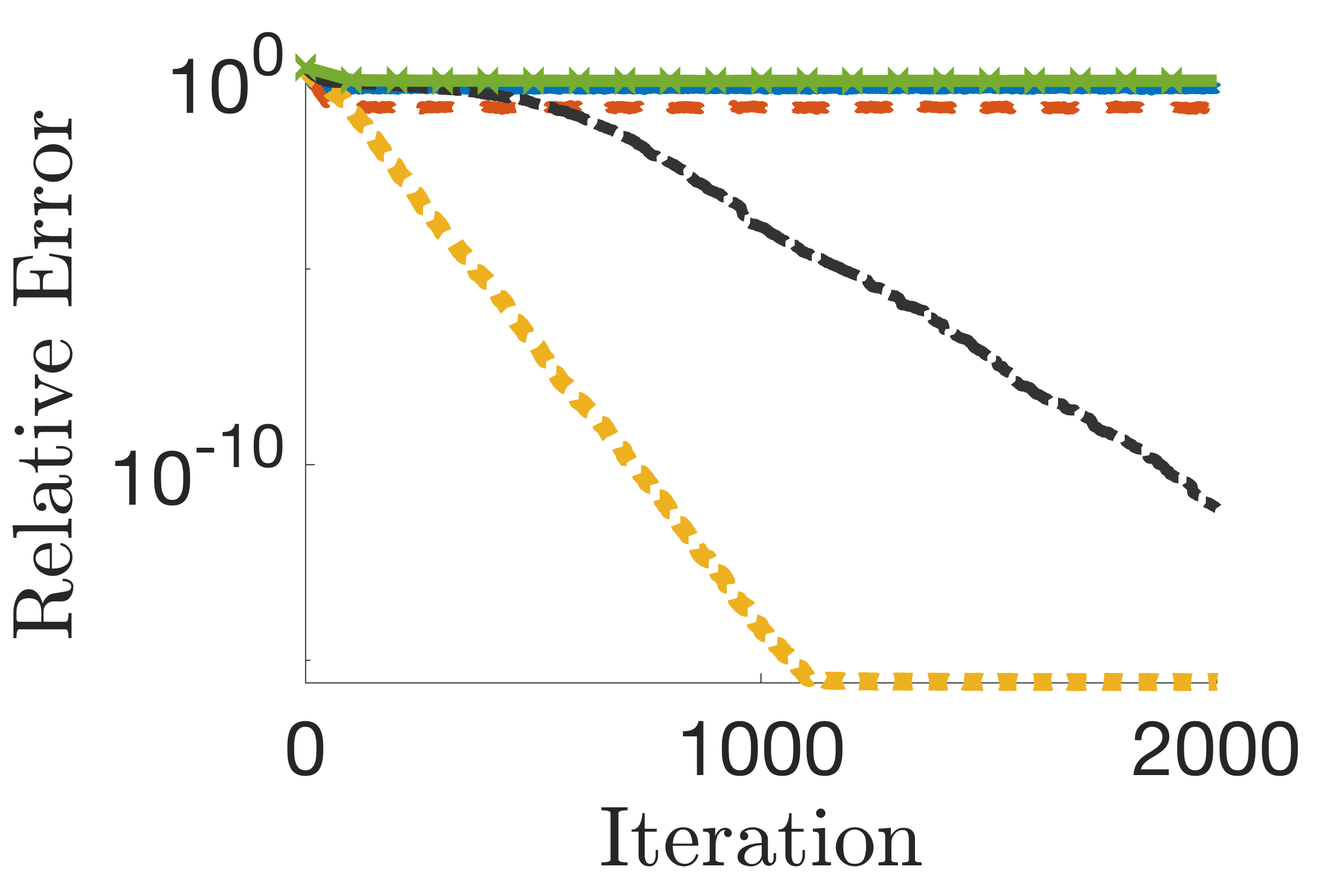}%
    \includegraphics[width=0.3\textwidth]{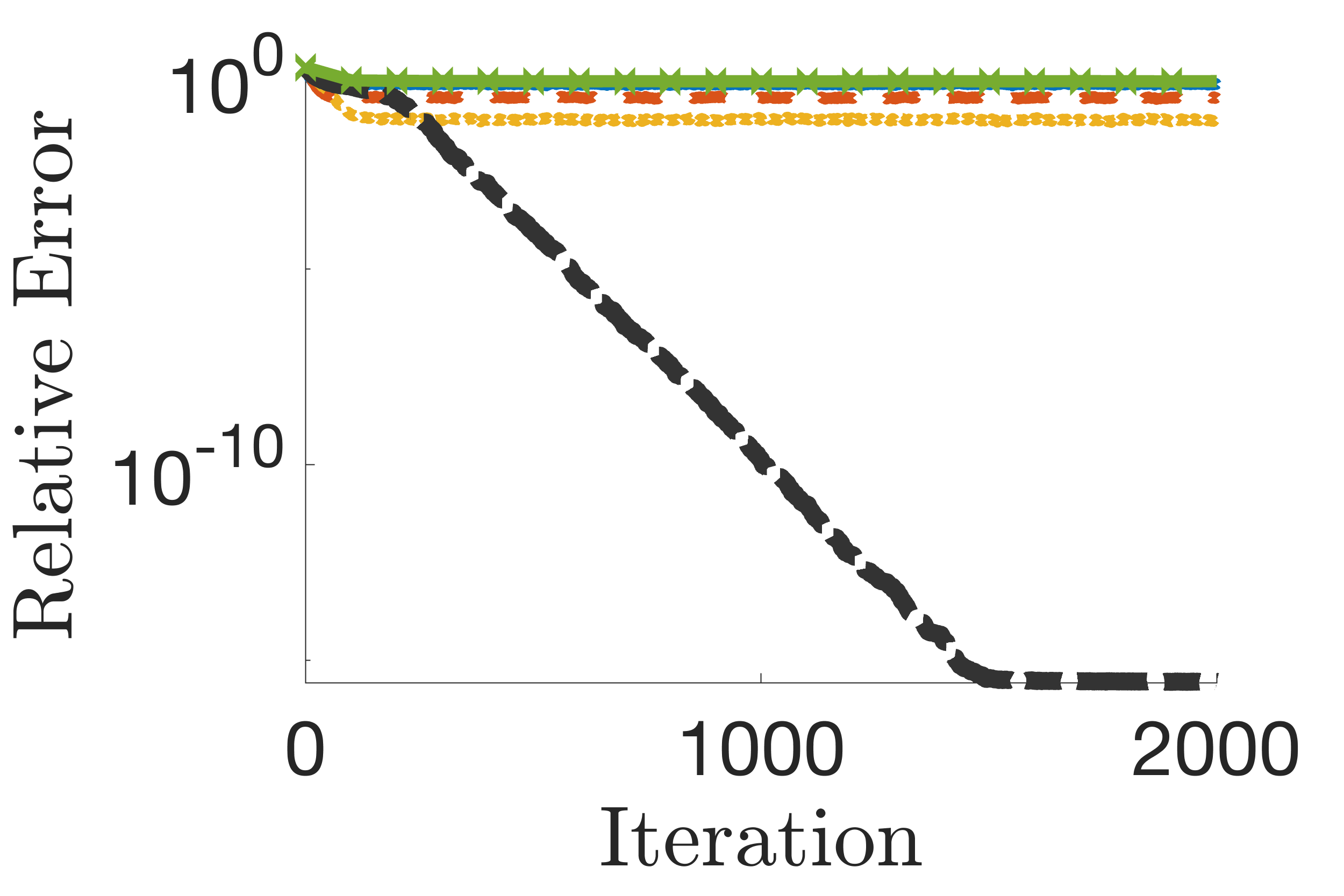}
    \\
    \includegraphics[width=0.3\textwidth]{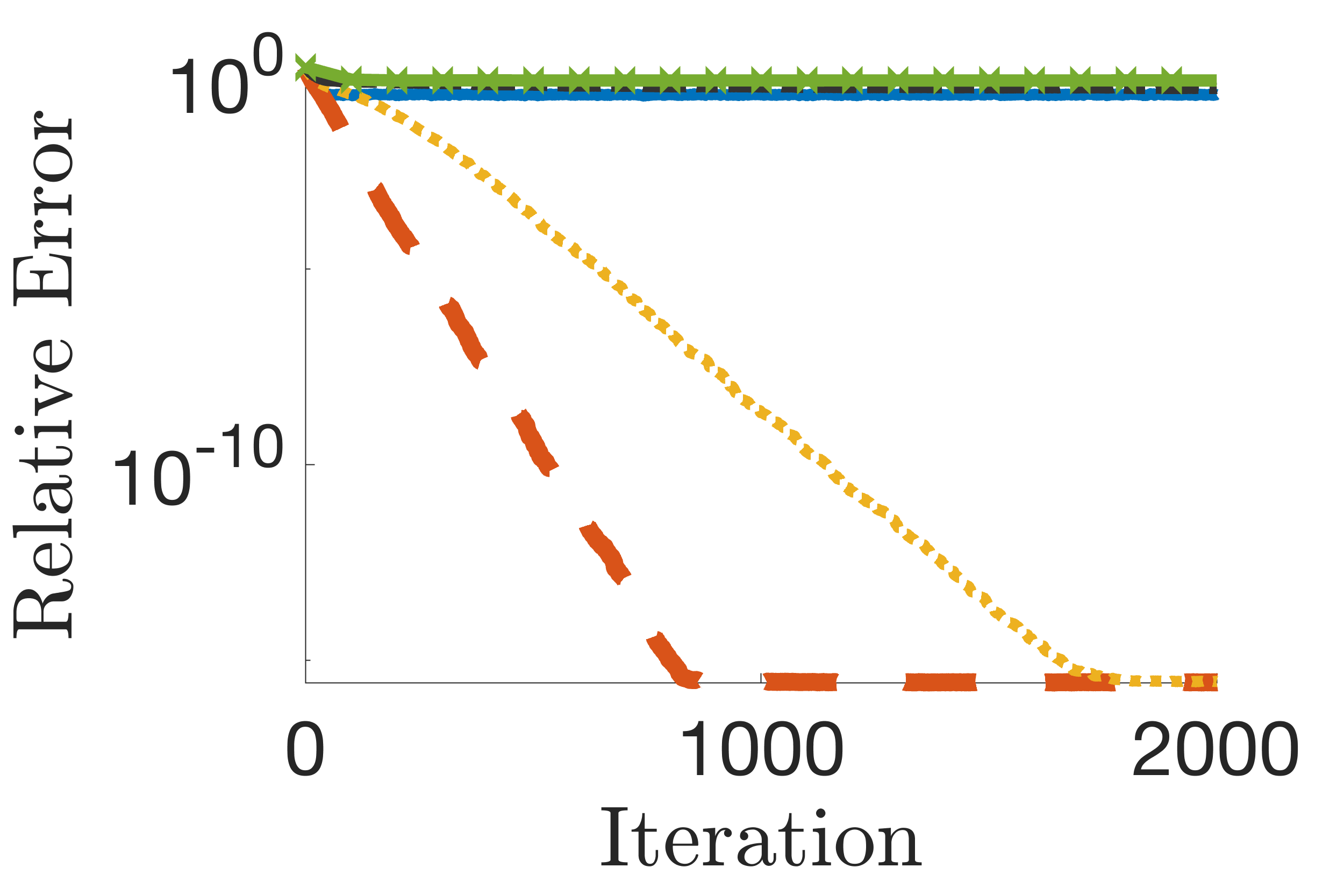}%
    \includegraphics[width=0.3\textwidth]{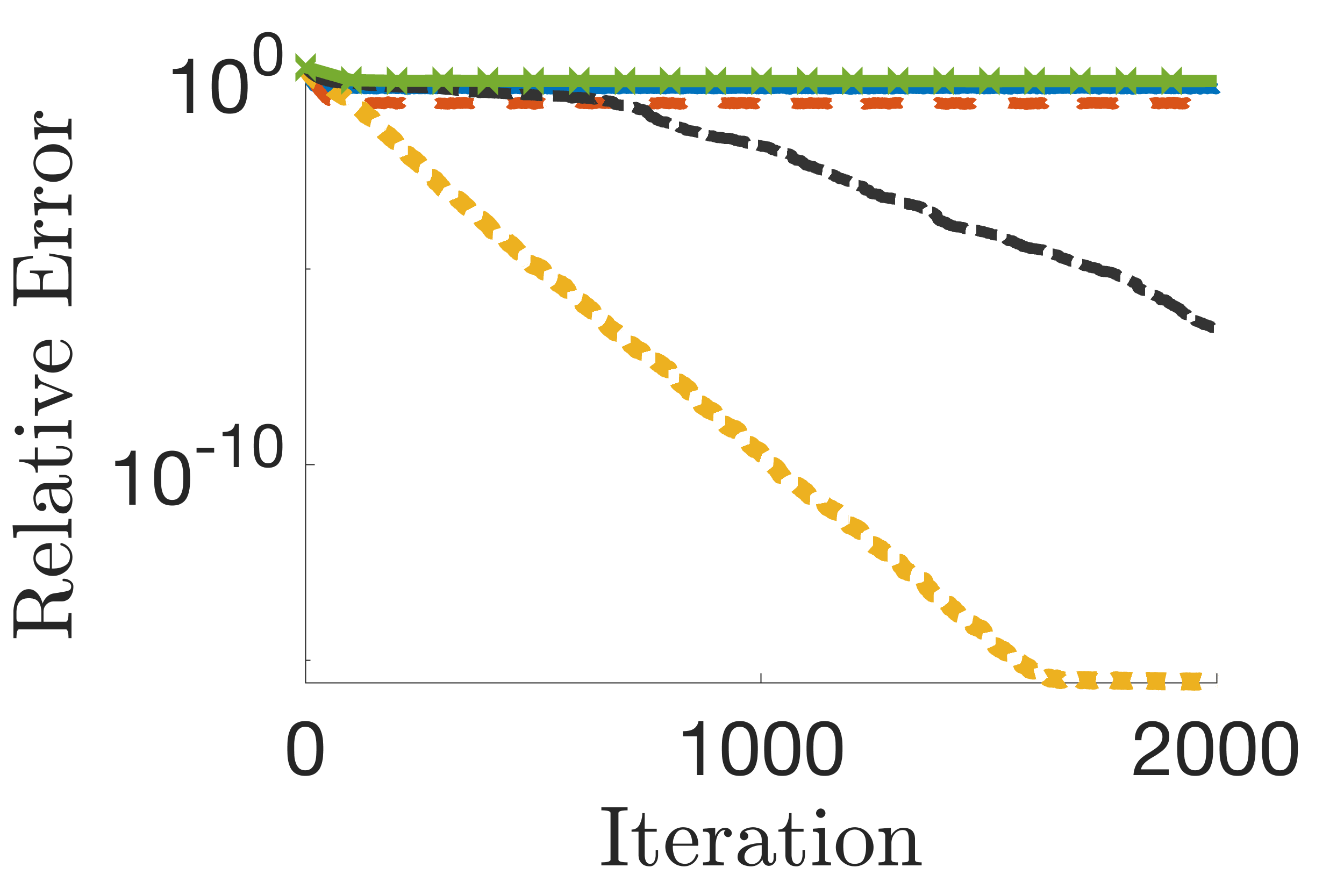}%
    \includegraphics[width=0.3\textwidth]{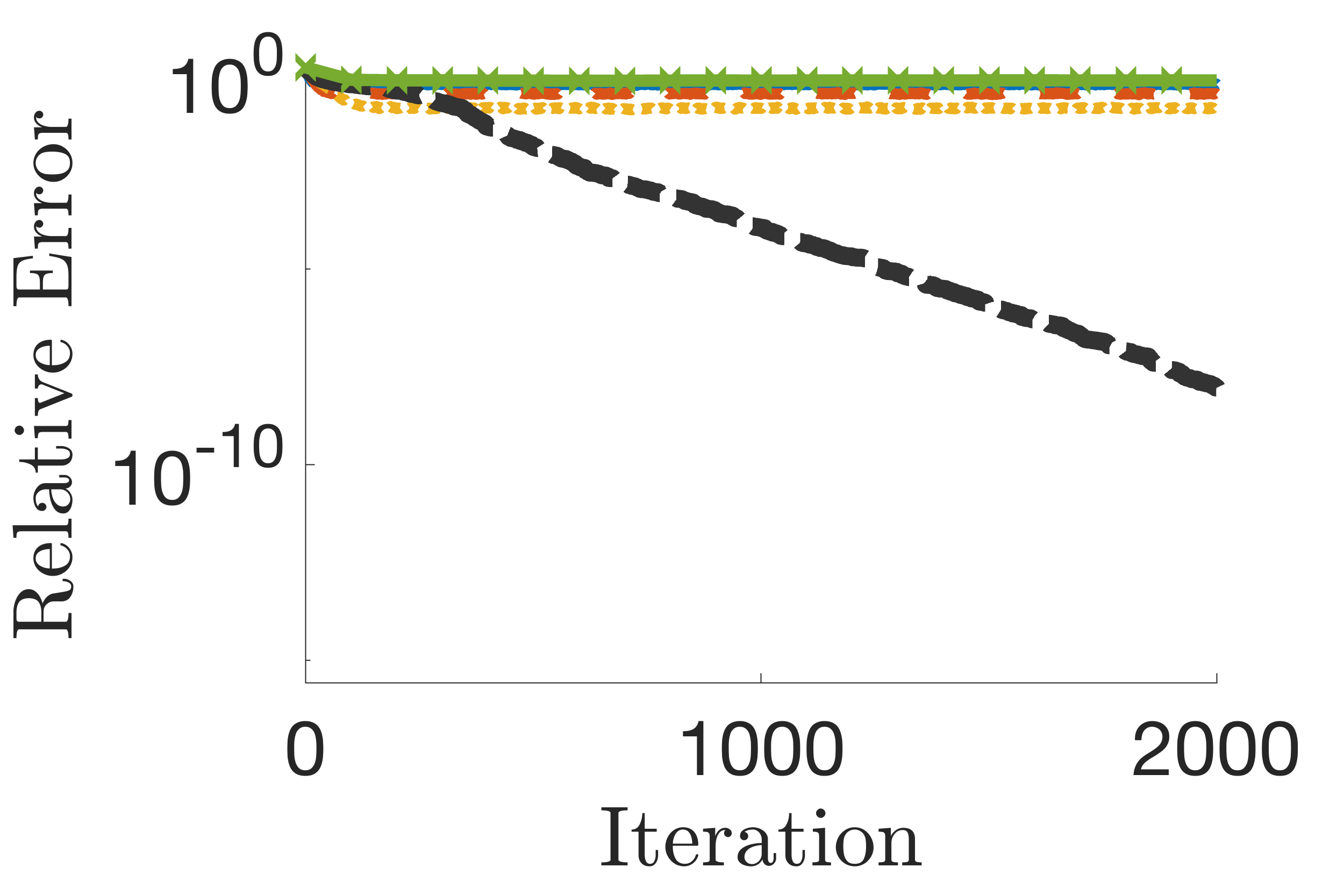}
    \\
    \includegraphics[width=0.3\textwidth]{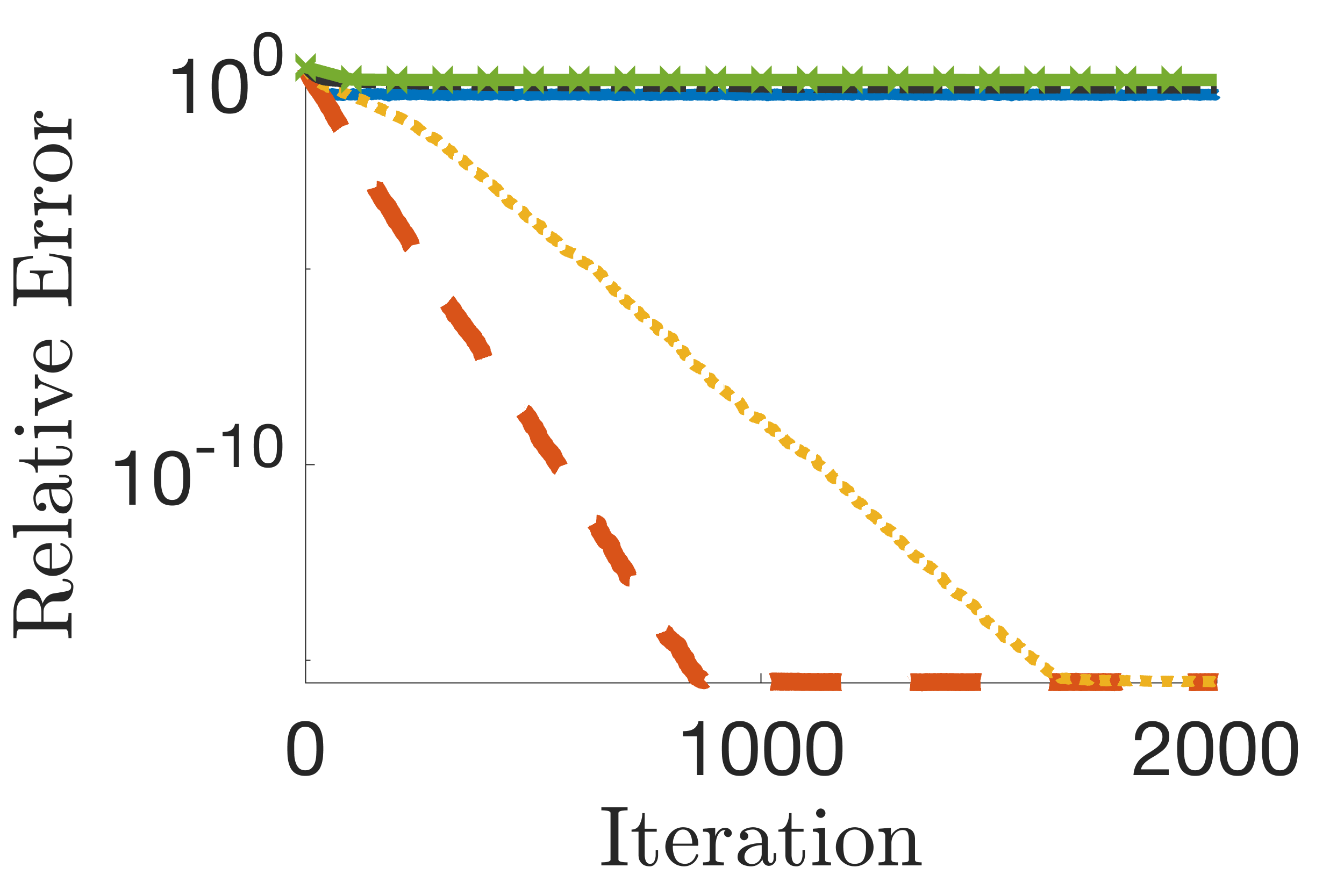}%
    \includegraphics[width=0.3\textwidth]{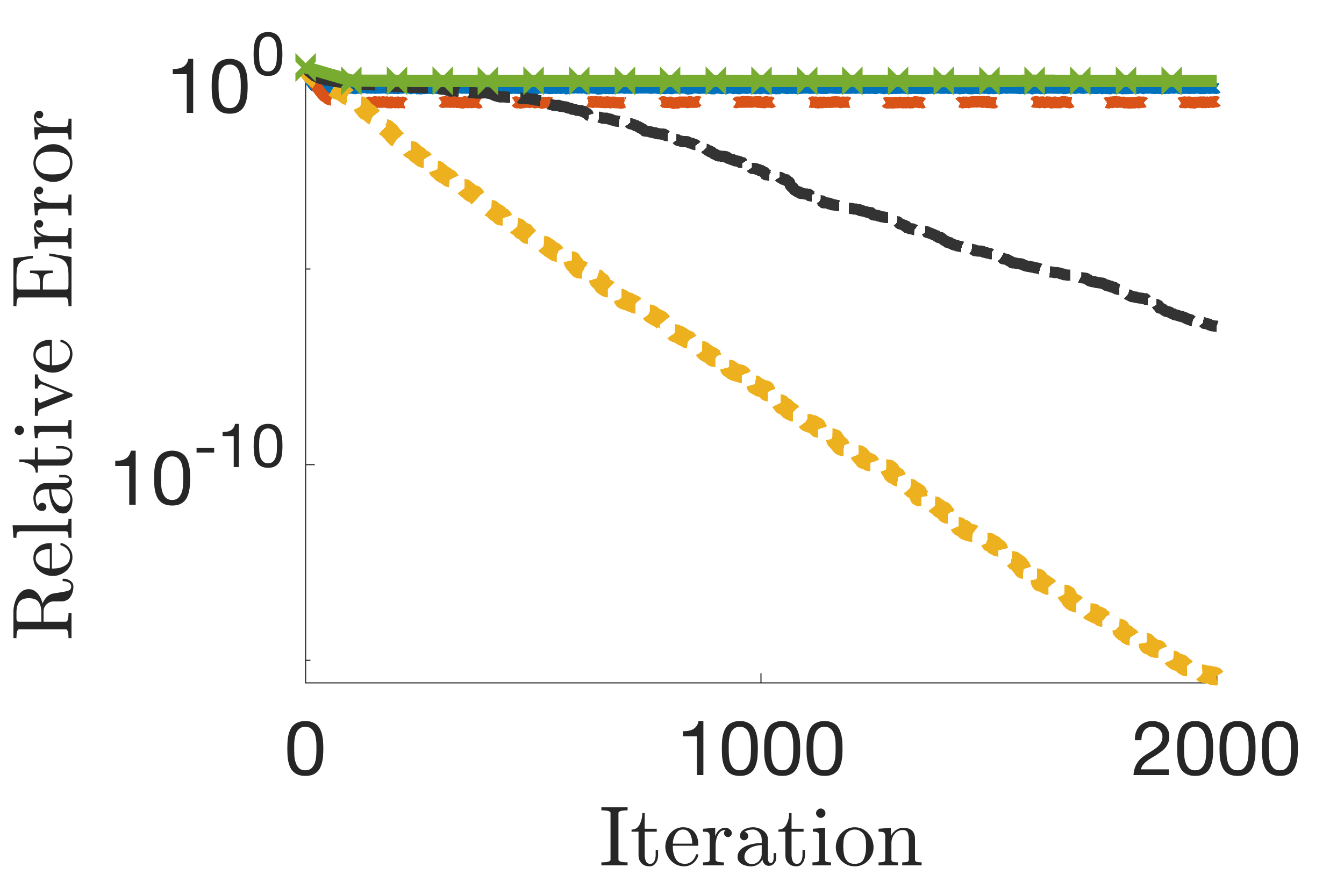}%
    \includegraphics[width=0.3\textwidth]{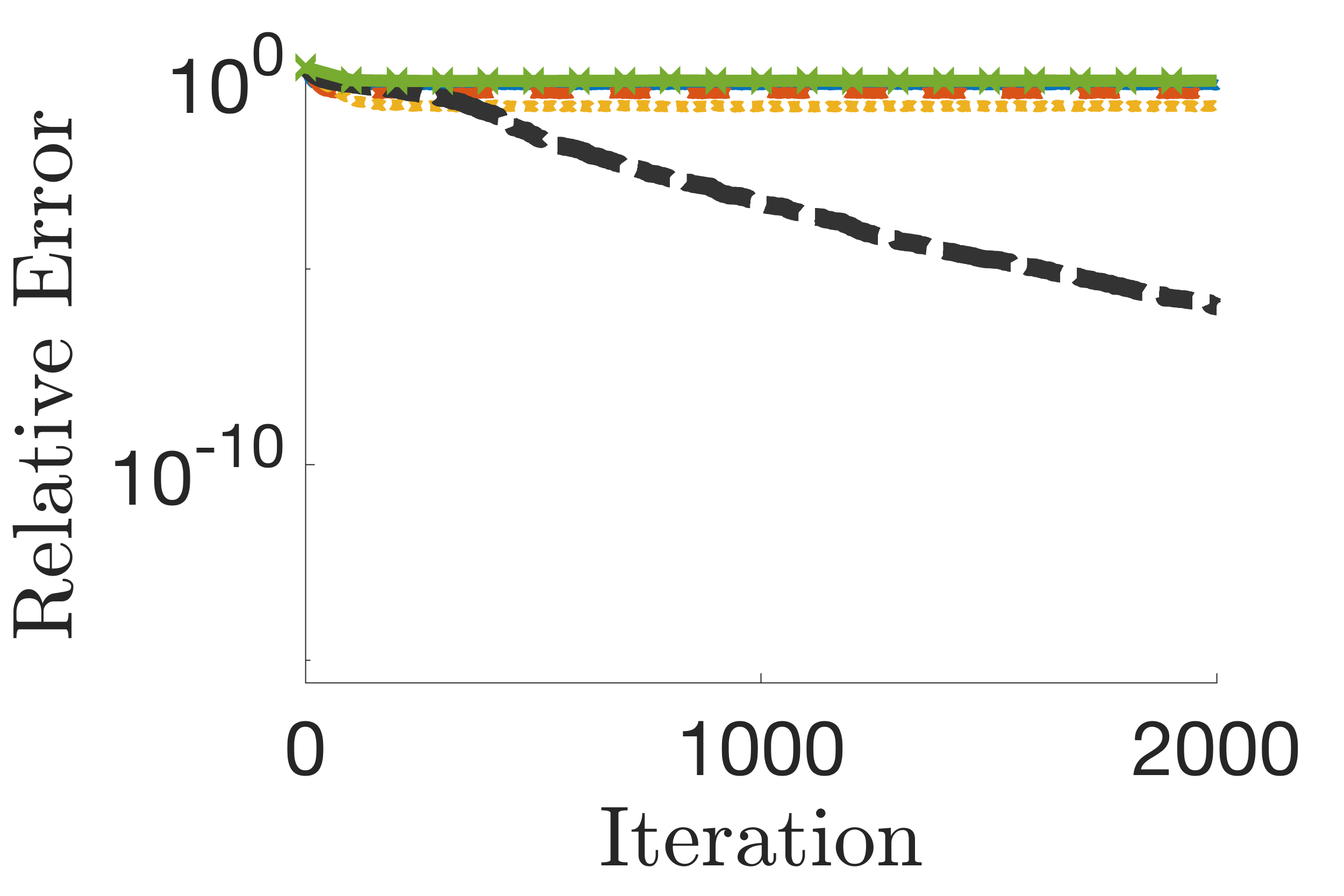}
    \caption{Median relative residual errors of mQTRK applied on a system $\tA \tX = \tB$ where $\tA \in \R^{25 \times 5 \times 10}$, $\tB \in \R^{25 \times 4 \times 10}$, and the corruptions are generated from $\mathcal{N}(10,5)$. In the left column plots, $\tbeta = 0.025$, the middle column plots $\tbeta = 0.05$, and the right column plots $\tbeta = 0.075$. In the top row plots, $\tbrow = 0.2$, the second row plots, $\tbrow = 0.4$, the third row plots $\tbrow = 0.8$, and the bottom row plots $\tbrow = 1$.}
     \label{fig:mqtrk-small-corr}
\end{figure}

In both Figures \ref{fig:mqtrk-large-corr} and \ref{fig:mqtrk-small-corr}, we observe the following:
\begin{enumerate}[label=(\alph*)]

\item For $\beta = $ 0.025 and 0.05, mQTRK converges for at least one of the $q$ values for all $\tbrow \in \{0.2, 0.4, 0.8, 1\}$.
Further, for $\beta =$ 0.025, as $\tbrow \ge 0.2$ increases, the set of $q$ values for which mQTRK converges remains the same. 
This is in contrast to QTRK (in Figures \ref{fig:qtrk-large-corr} and \ref{fig:qtrk-small-corr}) where we do not observe this behavior even without considering the case of  $\tbrow = 1$.
\item On the other hand, for $\tbrow = 0.2$ and 0.4, we observe that as $\tbeta \ge 0.025$ increases the convergence of mQTRK is slower and for some values of $q$ it ceases to converge even when $\tbeta \leq 0.075$. 
We note that this suggests that mQTRK may be most appropriately used when the overall corruption rate is small, although mQTRK is quite robust to scenarios where the corruptions are well distributed across the row slices of the measurement tensor.
\item We note that, although Remark~\ref{rem:mQTRK-bad-behavior} illustrates that scenarios exist where mQTRK can fail to make any progress even for reasonable choices of quantile $q$, empirically we do not see these scenarios arise.  We that see mQTRK makes progress with at least one value of $q$ in all experiments in Figures~\ref{fig:mqtrk-large-corr} and~\ref{fig:mqtrk-small-corr}.

\end{enumerate}

In Figure \ref{fig:mqtrk-large-corr}, we additionally observe the following.
For $\tbrow \geq 0.4$, mQTRK ceases to converge for the underestimates $\tbeta$ (i.e., when mQTRK is not cautious). 
This is in contrast to QTRK as observed in Figures \ref{fig:qtrk-large-corr} and \ref{fig:qtrk-small-corr}. 
On the other hand, it appears that the overestimates of $\tbeta$ are more favorable in the convergence of mQTRK. 
This is also observed in the right-most plots with $q = 0.90$ as an overestimate of $\tbeta = 0.075$.  In summary, mQTRK performs better when $q$ is chosen conservatively; this is in contrast to QTRK which performs better when $q$ is chosen boldly. This perhaps suggests choosing for $q$ slight overestimates of our belief on the value of $\tbeta$ when using mQTRK and and slight underestimate when using QTRK. The figures demonstrate some of these sensitivities.

In Figure \ref{fig:mqtrk-small-corr}, we observe in all the subplots that mQTRK does not converge for $q=0.9$ (corresponding to an overestimate of $\tbeta$).
Similar to what we noted in Section~\ref{subsec:QTRK_experiments} for QTRK, we observe slower convergence for mQTRK when the magnitudes of the corruptions create ambiguities.
While overestimates of $\tbeta$ are more favorable in the convergence of mQTRK when the magnitudes of the corruptions are overall larger (setting of Figure \ref{fig:mqtrk-small-corr}), we do not observe that in Figure \ref{fig:mqtrk-small-corr}.
mQTRK might be too restrictive and ``detect" at least one corruption in the same column slice over (many or all) iterations.
Thus, (repeatedly or completely) not updating the corresponding column in $\tX^{(k)}$.  Mitigating this behavior is an interesting direction for future research;  in this case, restarting may be an advantageous strategy.  Finally, we point out that mQTRK consistently outperforms TRK, again illustrating the promise of quantile-based methods for avoiding the effects of corruption.

\begin{figure}[h!]
    \centering  
    \includegraphics[width=0.3\textwidth]{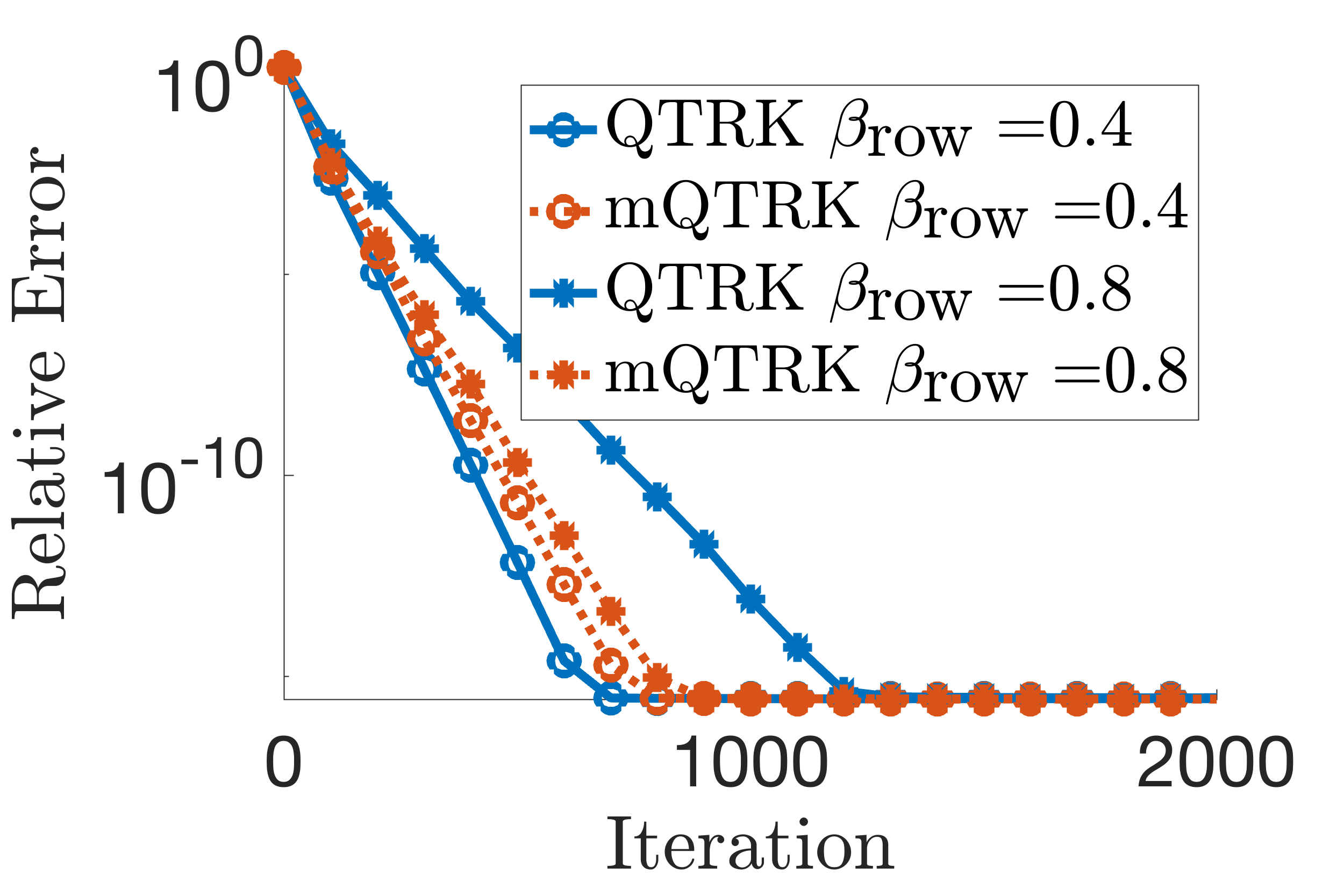}
    \includegraphics[width=0.3\textwidth]{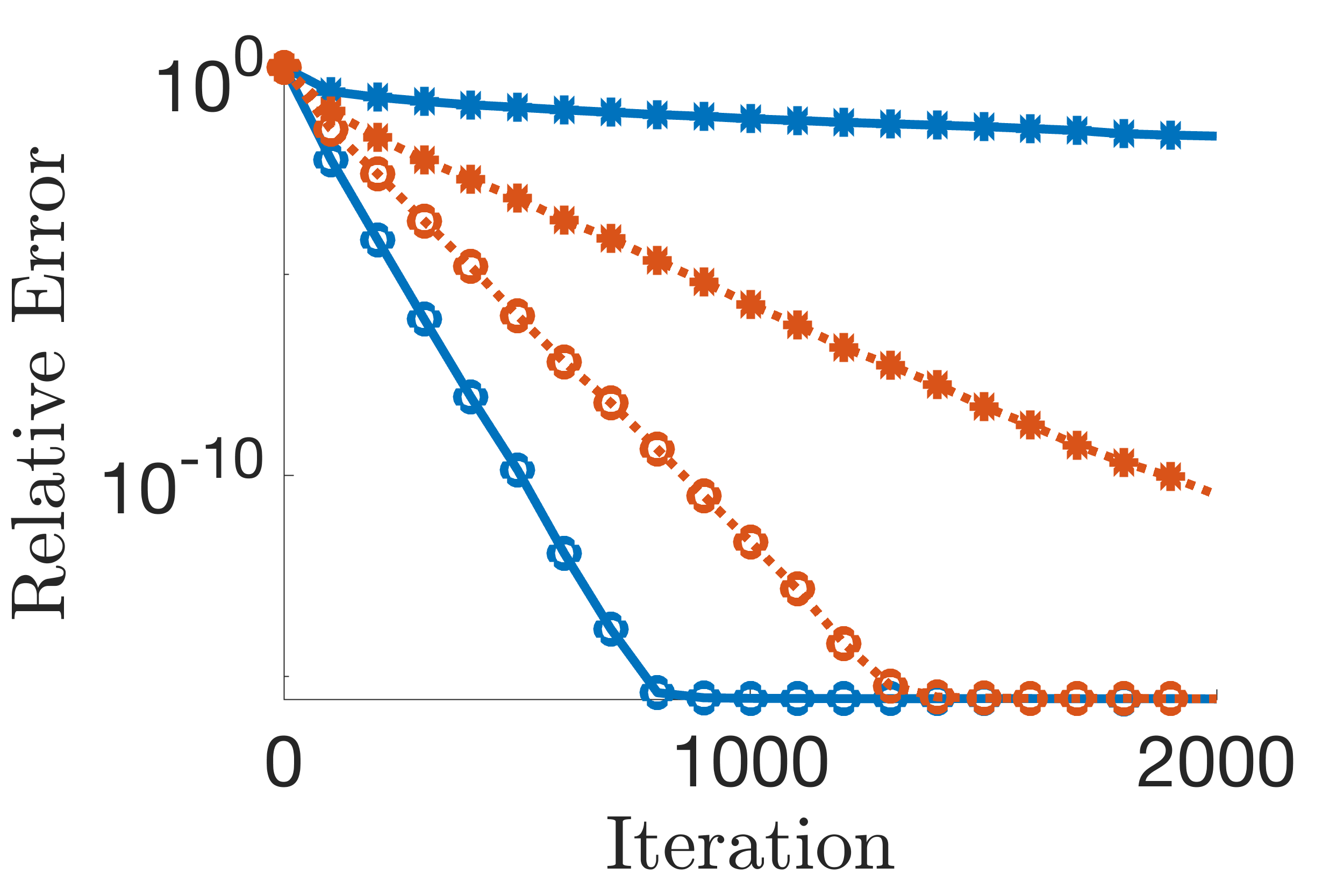}
    \includegraphics[width=0.3\textwidth]{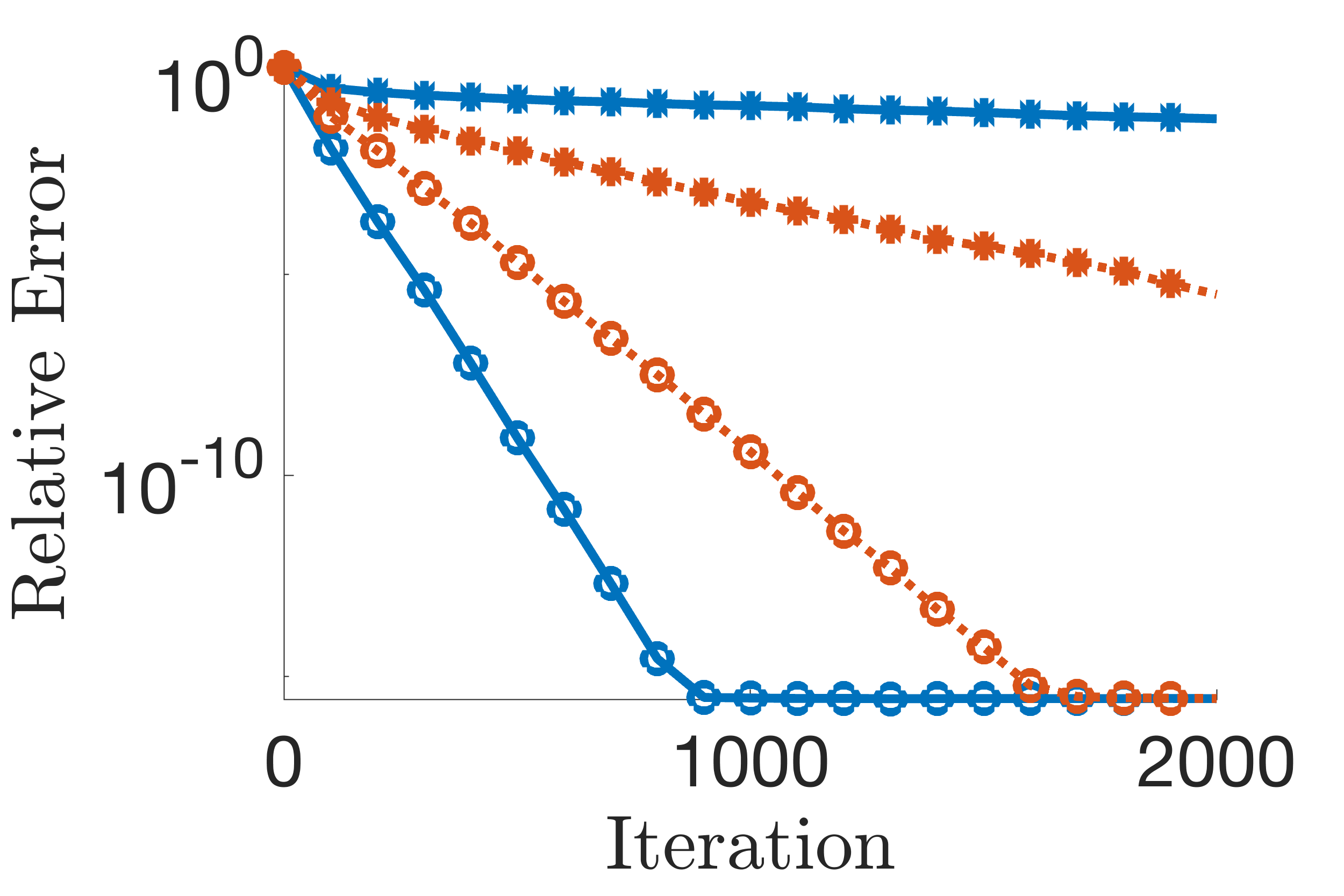}\\  
    \includegraphics[width=0.3\textwidth]{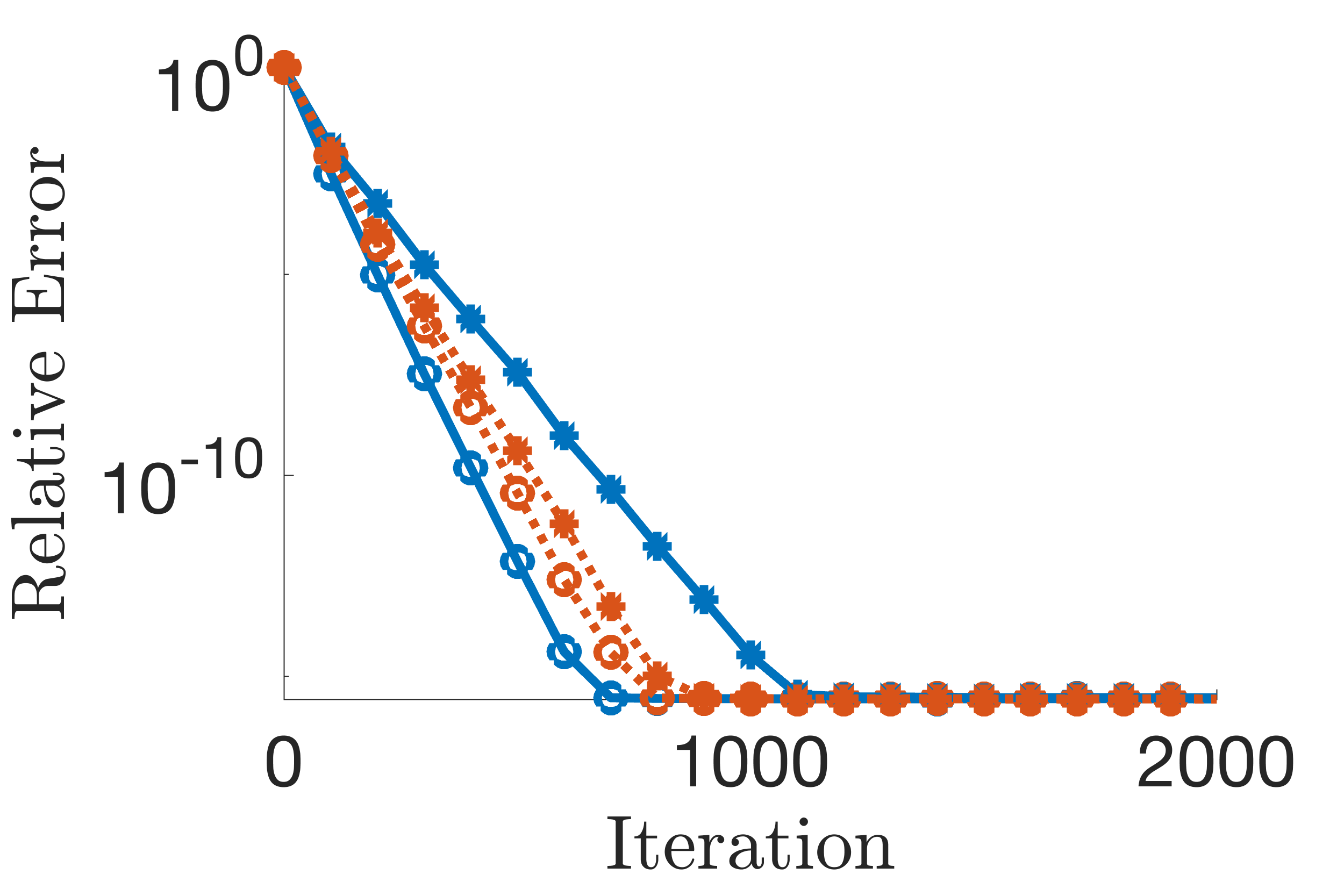}
    \includegraphics[width=0.3\textwidth]{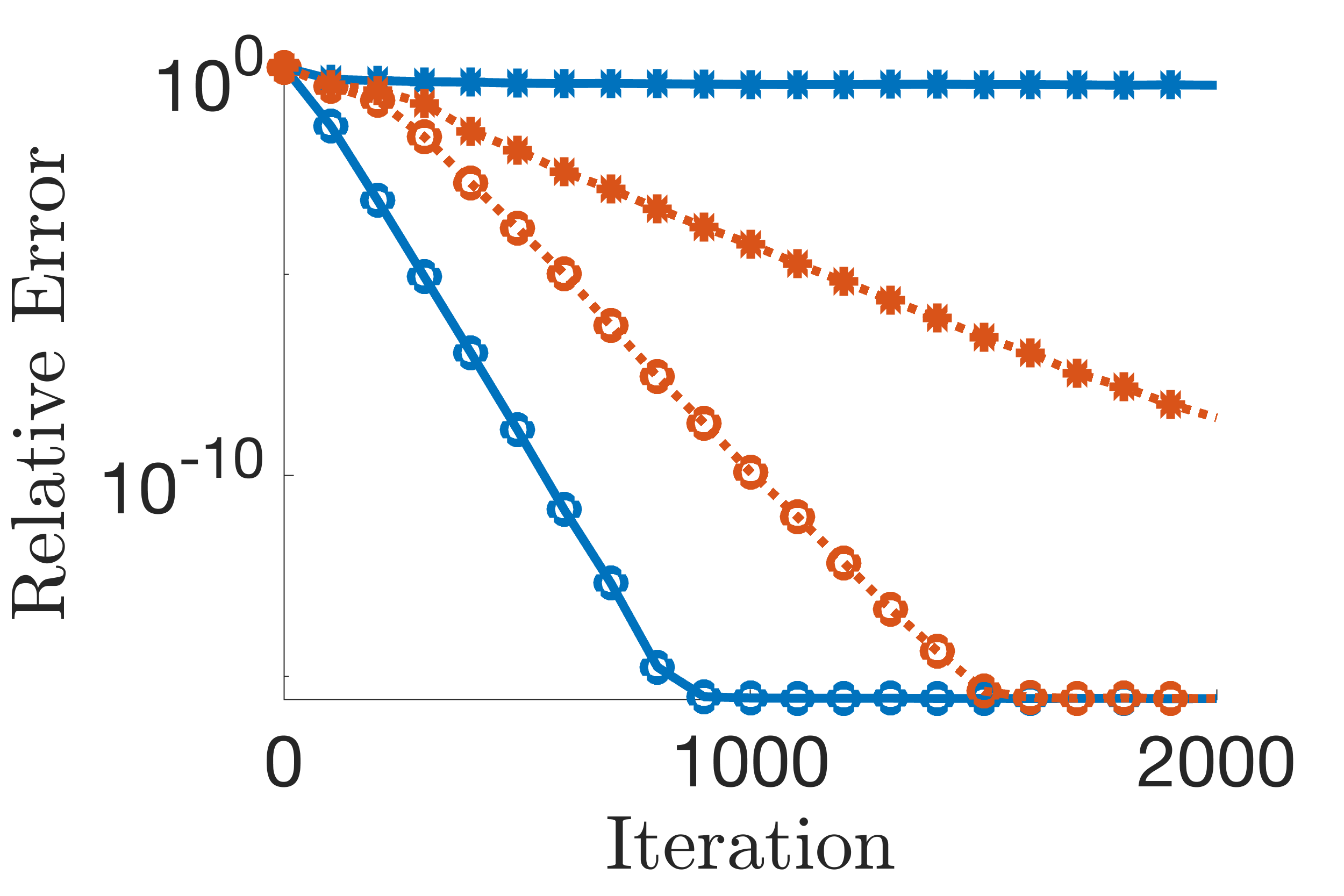}    
    \includegraphics[width=0.3\textwidth]{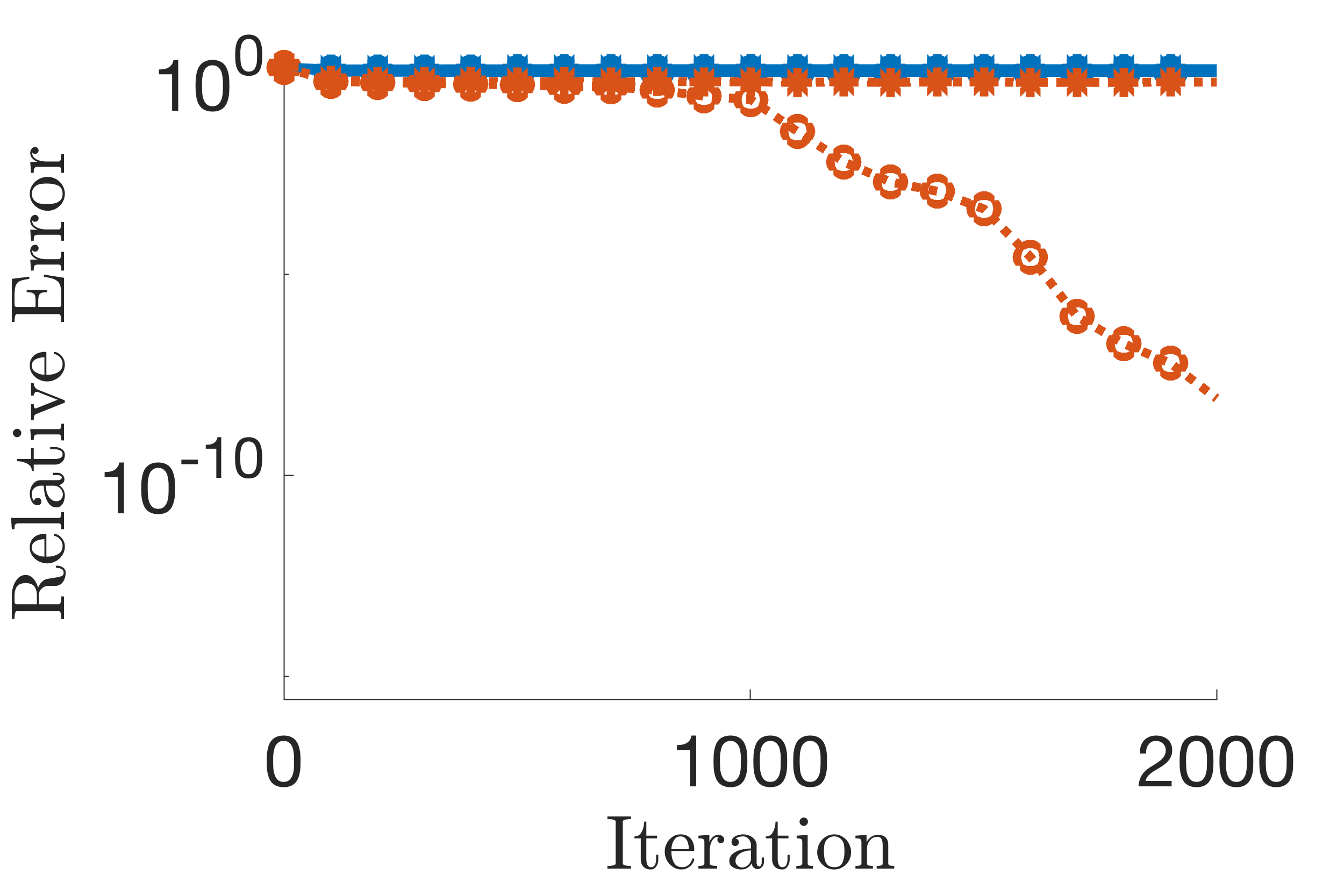}
    \caption{Median relative residual errors of QTRK and mQTRK applied on the same system $\tA \tX = \tB$ where $\tA \in \R^{25 \times 5 \times 10}$ and $\tB \in \R^{25 \times 4 \times 10}$. In the top row plots, the corruptions are generated from $\mathcal{N}(100,20)$, and in the bottom row plots, the corruptions are generated from  $\mathcal{N}(10,5)$.
    In the left column plots, $\tbeta = 0.025$, the middle column plots $\tbeta = 0.075$, and the right column plots $\tbeta = 0.1$. In all experiments, $q = 1 - \tbeta$. }
\label{fig:comp_plots}
\end{figure}

\subsection{Comparison of QTRK and mQTRK}\label{subsec:QTRKvsmQTRK}

In the set of experiments presented in Figure~\ref{fig:comp_plots} we apply QTRK and mQTRK to solve the same tensor linear systems $\tA \tX = \tB$ given the same initialization $\tX^{(0)}$.
We consider values for $\tbrow \in \{0.4, 0.8\}$ and $\tbeta \in \{ 0.025. 0.075, 0.01\}$. 
As an input for QTRK and mQTRK, we choose quantile values $q = 1 - \tbeta$ for each value of $\tbeta$. 
In the top row plots of Figure~\ref{fig:comp_plots}, the magnitudes of the corruptions are sampled from $\mathcal{N}(100,20)$ and in the bottom row plots the magnitudes are sampled from $\mathcal{N}(10,5)$.

In all the plots of Figure \ref{fig:comp_plots}, except for the bottom right plot, QTRK converges faster than mQTRK when $\tbrow = 0.4$.
When $\tbeta = 0.025$ the difference is minor, but that grows when $\tbeta = 0.075$.
In all the plots, except for the bottom right plot, mQTRK converges faster than QTRK when $\tbrow = 0.8$.
These observations are consistent with the results presented in Figures~\ref{fig:qtrk-large-corr} through \ref{fig:comp_plots}. 


\section{Application to Video Deblurring}
\label{subsec:video_deblurring}

In this section, we consider the problem of recovering a true video tensor $\tX \in \R^{l\times p\times n}$ from a blurry video tensor (with known blurring operator) that contains additive corruptions.
We showcase the performance of QTRK and mQTRK applied to the \texttt{mri} image dataset obtained from the Matlab Example Data Sets\footnote{\url{https://www.mathworks.com/help/matlab/import_export/matlab-example-data-sets.html}} after the images are blurred and corrupted.
We first describe how this problem can be formulated as a tensor linear system.

Suppose the video is blurred frame by frame using a circular convolution kernel $\mat{H} \in \R^{l_1 \times p_1}$ where $l_1<l$ and  $p_1<p$. 
By suitably padding the kernel with zeros, we can assume, without loss of generality, that $\mat{H} \in \R^{l \times p}$.
Let $\tH \in \R^{p \times p \times l}$ denote the blurring tensor 
with frontal slices given by $\tH_{::i} = \mathrm{circ}(\mat{H}_{i:}) \in \R^{p\times p}$, a circulant matrix defined on the $i^{th}$ row of $\mat{H}$.

One can then identify this with a tensor linear system
$$ \tH  \tilde{\tX} = \tilde{\tY},$$ 
where $\tilde{\tX}$ and $\tilde{\tY} \in \R^{p \times n \times l}$. Here, $\tilde{\tX}$ and $\tilde{\tY}$ are obtained by reordering the original video and blurry video tensor so that the horizontal slices of $\tX$ and $\tY \in \R^{l\times p \times n}$ now represent the frontal slices of $\tilde{\tX}$ and $\tilde{\tY}$ respectively.

\begin{figure}[h!]
    \includegraphics[width=0.65\textwidth]{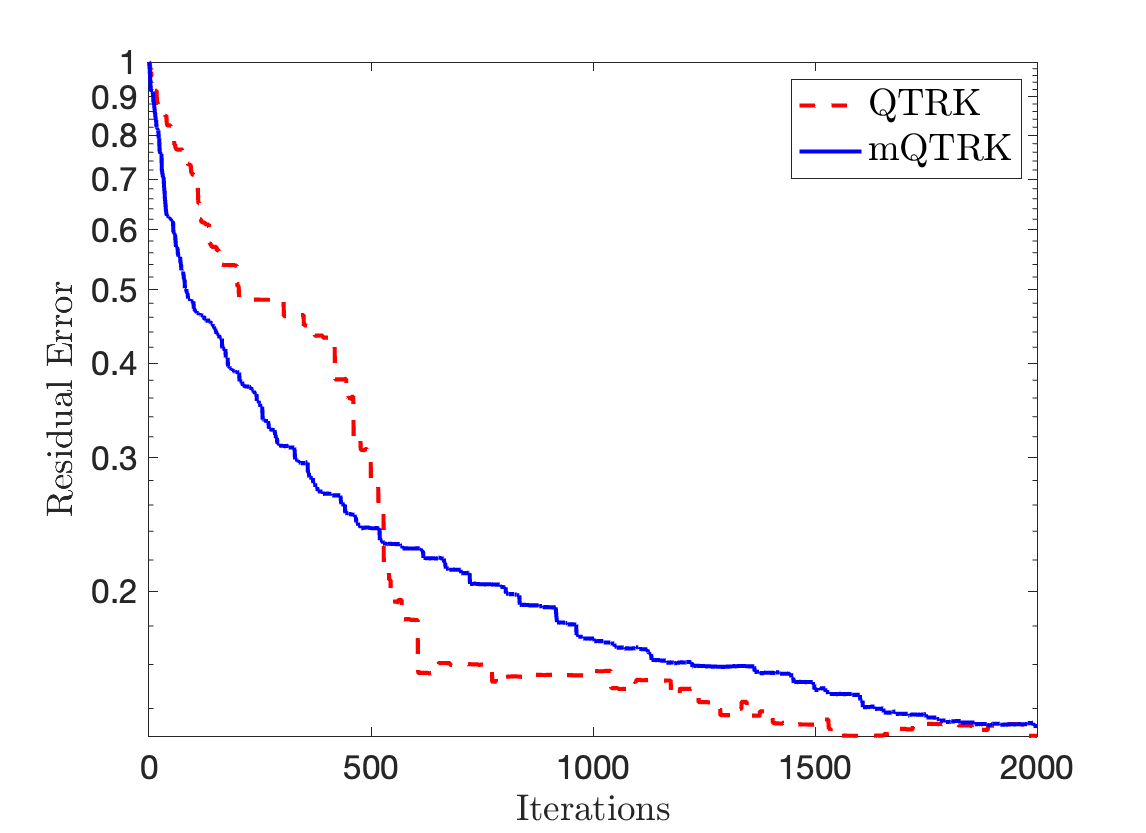}%
    \includegraphics[width=0.31\textwidth]{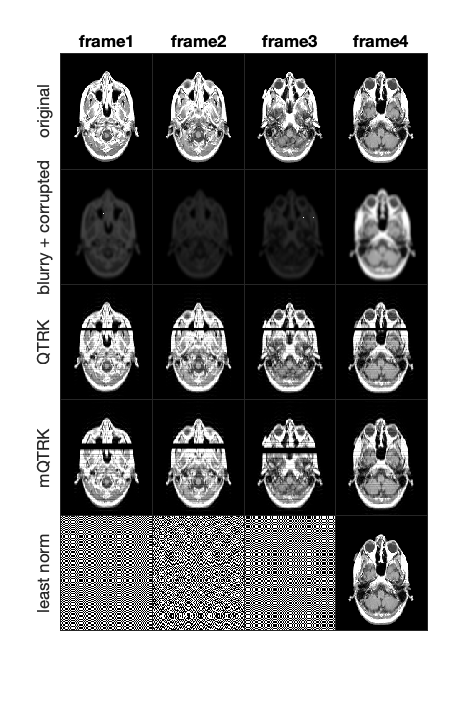}
    \caption{QTRK and mQTRK applied for the recovery of blurred and corrupted MRI video data. In the left plot, the relative residual error of QTRK and mQTRK is reported over 2000 iterations using log scale. In the right plot, the frames from the original video are shown in the top row, the blurred and corrupted frames in the second row, the frames recovered using QTRK and mQTRK in the third and fourth rows resp., and the least norm solution in the bottom row.}\label{fig:deblurring-exp}
\end{figure}

In this work, we are interested in the case where the corrupted blurred video $\widetilde{\tY}_c \in \R^{p \times n \times l}$ is observed instead.  
In the experiment presented in Figure~\ref{fig:deblurring-exp}, we consider the recovery of 12 frames of blurred and corrupted MRI images of size $128 \times 128$.
The images were first circularly blurred by using a $5\times5$ Gaussian filter to obtain $\widetilde{\tY}$ and then additively corrupted to obtain $\widetilde{\tY}_c\in\R^{128\times12\times128}$.
The corruption values are sampled from $|\mathcal{N}(3,2)|$ with 15 corruptions in total restricted to 6 row slices.
We use QTRK and mQTRK to solve the corrupted system given by $ \tH  \tilde{\tX} = \tilde{\tY}_c$ where 
$\tH \in \R^{128 \times 128 \times 128}$ is the Gaussian blurring operator.

In the left plot of Figure~\ref{fig:deblurring-exp}, we plot the relative residual errors $\|\tilde\tY - \tH \tX^{(k)} \|_F/\|\tilde\tY\|_F$ where $\tX^{(k)}$ is the output at the $k$-th iteration of QTRK and mQTRK. 
Both algorithms are ran for 2000 iterations with quantile value $q=0.99$.

As seen on the right of Figure~\ref{fig:deblurring-exp}, these recovered images visibly match the original video data outside of the row slices affected by corruption,   
In this figure, the first row of images gives four frames of the original data, the second row the same four frames after blurring and additive corruption, the third row the same four frames after recovery with QTRK, the fourth row the same four frames after recovery with mQTRK, and the last row the same four frames after recovery using a simple least-norm least-squares solve (which fails dramatically due to the additive corruption). 
We note that this system is not highly overdetermined so it is unlikely that our theoretical results hold in this setting.  However, we see good recovery of uncorrupted slices regardless.

\section{Conclusion}
We present two methods for solving the tensor regression problem in a large-scale setting with arbitrarily large corruptions. These approaches utilize the iterative framework of the Kaczmarz method along with a quantile that aims to detect and ignore corruptions. Interestingly, both the analysis and the behavior of this approach, while motivated from the matrix setting, are quite drastically different from that simpler setting. Indeed, because a single corruption in the observations affects multiple possible projections within the algorithm, greater care is needed when designing, analyzing, and implementing these approaches. To that end, we propose a straightforward quantile tensor randomized Kaczmarz method that, when the proportion of corruptions is small enough, will guarantee convergence. We also propose a more complex ``masked" variant that aims to allow for a larger proportion of corruptions but suffers from the impossibility of guaranteeing theoretical convergence. Nonetheless, we show that both methods perform well in practice, and include a discussion about the likelihood the masked version would encounter the scenario that prevents its theoretical convergence. Interesting future work could include a modification of the masked version that removes this impossibility, or an analysis that tracks the probability of trajectories of the iterates and bounds the catastrophic events.

\section{Acknowledgements}
The initial research for this effort was conducted at the Research Collaboration Workshop for Women in Data Science and Mathematics (WiSDM), August 2023 held at the Institute for Pure and Applied Mathematics (IPAM). Funding for the workshop was provided by IPAM, AWM and DIMACS (NSF grant CCF1144502).

This material is based upon work supported by the National Science Foundation under Grant No. DMS-1928930, while several of the authors were in residence at the Mathematical Sciences Research Institute in Berkeley, California, during the summer of 2024.  

Several of the authors also appreciate support provided to them at a SQuaRE at the American Institute of Mathematics. The authors thank AIM for providing a supportive and mathematically rich environment.

This material is based upon work supported by the National Science Foundation under Grant No. DMS-1929284 while the author was in residence at the Institute for Computational and Experimental Research in Mathematics in Providence, RI, during the ``Randomized Algorithms for Tensor Problems with Factorized Operations or Data" Collaborate@ICERM.

JH was partially supported by NSF DMS \#2211318 and NSF CAREER \#2440040.  DN was partially supported by NSF DMS \#2408912. KYD was partially supported by NSF LEAPS-MPS \#2232344.

\bibliographystyle{ieeetr}
\bibliography{main2}

\newpage

\appendix

\end{document}